\documentclass[review,times]{elsarticle}
\usepackage[T1]{fontenc}

\usepackage[hang]{subfigure}

\usepackage{xtab,booktabs}

\usepackage{tikz-cd}
\usepackage{lineno,hyperref}
\modulolinenumbers[5]

\journal{}

\usepackage{changepage}

\usepackage{amsthm} %this add end-of-proof symbol

\newtheorem{proposition}{Proposition}

\usepackage{caption}

\usepackage{paralist}
\usepackage{amsmath}
\usepackage{amssymb}
\usepackage{amsfonts}
\usepackage{mathrsfs}
\usepackage{subfigure}
\usepackage{epsfig}
\usepackage{graphicx}
\usepackage{color}

\usepackage{float}

\DeclareCaptionFormat{cont}{#1 (cont.)#2#3\par}

\usepackage[linesnumbered,ruled,]{algorithm2e}

\graphicspath{{../RPM_model_occlude/}}

\begin{document}
	
	\begin{frontmatter}
		
%		\title{A Hybrid  Bilinear and Quadratic  Programming Approach for Aligning Partially Overlapping Point Sets}

\title{Decomposed Global Optimization for Robust Point Matching with Low-Dimensional Branching}

		\author[mymainaddress]{Wei Lian} %\corref{mycorrespondingauthor}}
	\ead{lianwei3@gmail.com}

\author[mymainaddress]{Zhesen Cui}
\ead{cuizhesen@gmail.com}

\author[mymainaddress]{Fei Ma}
\ead{mafei@czc.edu.cn}

\author[mymainaddress]{Hang Pan}
\ead{panhang@czc.edu.cn}

\author[mysecondaryaddress]{Wangmeng Zuo}
\ead{cswmzuo@gmail.com}

%	
%		\author[mymainaddress]{Zihan Guo}
%guo_zihan11@163.com
%	
%	
	\author[mymainaddress]{Jianmei Zhang \corref{cor1}}
\ead{czxyzjm@163.com}
%	\cortext[mymainaddress]{Corresponding author3}
\cortext[cor1]{Corresponding author}

	\address[mymainaddress]{Department of Computer Science, Changzhi University, Changzhi, Shanxi, China, 046011}
\address[mysecondaryaddress]{School of Computer Science and Technology, Harbin Institute of Technology, Harbin 150001, China}

	\begin{abstract}	
Numerous applications require algorithms that can align partially overlapping point sets while maintaining invariance to geometric transformations (e.g., similarity, affine, rigid). This paper introduces a novel global optimization method for this task by minimizing the objective function of the Robust Point Matching (RPM) algorithm. We first reveal that the original RPM objective is a cubic polynomial. Through a concise variable substitution, we transform this objective into a quadratic function. By leveraging the convex envelope of bilinear monomials, we derive a tight lower bound for this quadratic function. This lower bound problem conveniently and efficiently decomposes into two parts: a standard linear assignment problem (solvable in polynomial time) and a low-dimensional convex quadratic program. Furthermore, we devise a specialized Branch-and-Bound (BnB) algorithm that branches exclusively on the transformation parameters, which significantly accelerates convergence by confining the search space. Experiments on 2D and 3D synthetic and real-world data demonstrate that our method, compared to state-of-the-art approaches, exhibits superior robustness to non-rigid deformations, positional noise, and outliers, particularly in scenarios where outliers are distinct from inliers.

	\end{abstract}
	
	
\end{frontmatter}

%\end{document}

%\linenumbers

\section{Introduction}

% 开篇点题，直接指出问题和现有方法的局限性，引出本文的研究动机。
%Point cloud registration, the task of aligning two point sets into a coherent model, is a fundamental problem in computer vision, robotics, and 3D sensing.   
%While traditional methods like the Iterative Closest Point (ICP) algorithm \cite{ICP,ICP2} are widely used, they are  sensitive to the initial alignment and often converge to poor local minima in challenging  situations characterized by significant outliers and partial overlap.
%This limitation motivates the development of globally optimal algorithms that can guarantee convergence to the correct solution regardless of initialization.

Point cloud registration, the process of aligning two point sets into a unified coordinate system, is a cornerstone of 3D computer vision, robotics, and augmented reality. While traditional methods like the Iterative Closest Point (ICP) algorithm are widely used, they're notoriously sensitive to initial alignment and frequently get trapped in local minima, especially when dealing with significant outliers or partial data overlap. This fundamental limitation highlights the need for globally optimal algorithms that can reliably converge to the correct solution regardless of   initialization. %the starting point.

The Branch-and-Bound (BnB) framework has emerged as a powerful tool for achieving such global guarantees.
Many work applies BnB to solve the coupled problem of finding the optimal correspondence $\mathbf{P}$ and transformation $\mathbf{T}(\cdot|\boldsymbol{\theta})$, typically formulated as minimizing the following  objective function:
%The Branch-and-Bound (BnB) framework has emerged as a powerful tool for achieving such global guarantees. Many BnB-based registration methods aim to solve the coupled problem of finding both the optimal correspondence P and the transformation T, typically by minimizing an objective function like this:
\begin{equation}
	E(\mathbf{P}, \boldsymbol{\theta}) = \sum_{i,j} p_{ij} \|\mathbf{y}_j - \mathbf{T}(\mathbf{x}_i | \boldsymbol{\theta})\|^2,
	\label{en_BLALS}
\end{equation}
where ${\mathbf{x}_i}$ and ${\mathbf{y}_j}$ are points from the two sets, $p_{ij}\in \{0,1\}$ are binary variables indicating correspondence, and $\boldsymbol{\theta}$ parameterizes the spatial transformation. 
An Early attempt to solve this problem by branching over either $\mathbf{P}$ or $\boldsymbol{\theta}$ proved inefficient due to the immense search space and an inefficient bounding scheme.
 Subsequent research has largely split into two camps:

\begin{enumerate}

	\item 	
	\textbf{Correspondence-only Methods}: These approaches eliminate the transformation variables $\boldsymbol{\theta}$, simplifying Eq. \eqref{en_BLALS} into a problem that only involves correspondences. While this can reveal   low-rank concave structures amenable to BnB optimization \cite{RPM_concave_PAMI, RPM_model_occlude},
it comes at a cost. Initial methods were either restricted to subset alignment scenarios \cite{RPM_concave_PAMI} or lacked transformation invariance \cite{RPM_model_occlude}. While follow-up works \cite{RPM_model_occlude_PR, lian2021polyhedral} successfully restored invariance, they could not overcome the core limitation of this approach: branching in the high-dimensional correspondence space. This large search space inherently leads to slow convergence, making these methods impractical for larger problems.

\item
\textbf{Transformation-space Methods}: This more promising strategy retains the transformation parameters $\boldsymbol{\theta}$ and performs branching within the much smaller, low-dimensional transformation space. A representative example is the work by \cite{LIAN2023126482}, which  treats Eq.~\eqref{en_BLALS} as a cubic polynomial and uses trilinear relaxation to derive a lower bound for the BnB algorithm. This approach effectively handles partial overlap and maintains transformation invariance, benefiting from a significantly reduced search space.

\end{enumerate}

\subsection{Our Contributions}
In this paper, we follow the second strategy of branching on the low-dimensional transformation space. However, we introduce a novel method to derive a simpler, yet more effective, lower bound that avoids a key complexity of prior work. The algorithm in \cite{LIAN2023126482} requires a complex reformulation of the objective function to ensure one of the variables in the trilinear monomial  includes the origin within its range.

In contrast, as illustrated in Fig. \ref{RPM_lb_idea}, 
%In contrast,
 our method exclusively uses \textbf{bilinear relaxation}, which eliminates this requirement and results in a more straightforward and streamlined algorithm.

% 总结贡献，条理清晰，与前文论述相呼应。
Our main contributions are summarized as follows:
\begin{itemize}
	\item 
	We present a novel BnB algorithm for point cloud registration that utilizes a simpler, more robust \textbf{lower-bounding scheme based on bilinear relaxation}. This eliminates the need for complex objective function reformulations required by previous method \cite{LIAN2023126482}, simplifying the overall algorithm.

	\item 	
	Our lower bound optimization is  efficient. It \textbf{decomposes into two well-behaved subproblems:} a classic linear assignment problem for correspondences and a low-dimensional convex quadratic program for the transformation. Both can be solved to global optimality very quickly.

	\item
	
	Our resulting BnB algorithm inherits the key advantages of its predecessors, including \textbf{robustness to partial overlap} and \textbf{invariance to transformations}. Crucially, it leverages a low-dimensional branching space, enabling it to scale  effectively to large problems while maintaining global optimality guarantees.
	
\end{itemize}

\begin{figure*}
	\begin{adjustwidth}{-1.8cm}{0cm}
	\centering
	\begin{tikzcd}[row sep=scriptsize, column sep=-0.9em]
		E(\mathbf P,\boldsymbol{\theta})
		%		{\triangleq}{\sum_{i,j}}p_{ij}\|\mathbf y_j {-} \mathbf T(\mathbf x_i| \boldsymbol{\theta})\|^2 
		{=}&
		\textcolor{red}{\boxed{\text{trilinear term of }\boldsymbol\theta \text{ and } \mathbf P}}
		\arrow[d,"\text{define}" swap, "\boldsymbol\Theta \triangleq\boldsymbol\theta\boldsymbol\theta^\top" ]
		& +
		&\text{quadratic term of }\boldsymbol\theta
		&  
		+&\text{bilinear term of }\boldsymbol\theta \text{ and } \mathbf P&
		+\text{linear term of } \mathbf P
		\\
		E(\mathbf P,\boldsymbol\Theta,\boldsymbol\theta)=&\boxed{\textcolor{red}{\text{bilinear term of }\boldsymbol\Theta \text{ and } \mathbf P}}
		\ar[d,"\text{bilinear}" swap,  "\text{relaxation}" ]
		&+
		&\text{quadratic  term of }\boldsymbol\theta &
		+&\boxed{ \textcolor{blue}{\text{bilinear term of }\boldsymbol\theta \text{ and } \mathbf P}} \ar[d,"\text{bilinear}" swap, "\text{relaxation}" ]
		&
		+\text{linear term of } \mathbf P \\
		E_l(\mathbf P,\boldsymbol\Theta,\boldsymbol\theta) = &\boxed{\textcolor{red}{\text{linear term of }\boldsymbol\Theta \text{ and } \mathbf P}} 
		& +  
		& \text{quadratic  term of }\boldsymbol\theta 
		\ar[d,"\text{substitute}" swap,  "\text{back }\boldsymbol\Theta=\boldsymbol\theta\boldsymbol\theta^\top" ]
		&
		+		
		&\boxed{\textcolor{blue}{\text{linear term of }\boldsymbol\theta \text{ and } \mathbf P}} &
		+\text{linear term of } \mathbf P \\
		E_l(\mathbf P,\boldsymbol\theta)= & \text{quadratic term of } \boldsymbol\theta &+& \text{linear term of } \boldsymbol\theta &+& \text{linear term of }\mathbf P
		%	& X \times_Z Y \arrow[r, "p"] \arrow[d, "q"]
		%	& X \arrow[d, "f"] \\
		%	& Y \arrow[r, "g"]
		%	& Z
	\end{tikzcd}
	\end{adjustwidth}
	
	\caption{
		Derivation of  a  lower bound function of the RPM objective function
		%				$E(\mathbf P,\boldsymbol{\theta})$ 
		via variable substitution and  bilinear relaxation.
		%		Here 
		\label{RPM_lb_idea}		%
	}
\end{figure*}

The rest of the paper is organized as follows. Section~\ref{sec:relate_work} provides a review of related work. Section~\ref{sec:bil_lb} introduces the theoretical framework for our lower bound. Section~\ref{sec:problem} details our objective function, and Section~\ref{sec:optimize} describes the optimization strategy. We apply the algorithm to 2D and 3D registration in Sections~\ref{sec:app_one} and~\ref{sec:app_two}, respectively. Finally, we present experimental results in Section~\ref{sec:exp} and conclude in Section~\ref{sec:conclude}.

\section{Related Work \label{sec:relate_work}}

A comprehensive survey of point set registration is available in~\cite{10.1007/s11263-020-01359-2}. Here, we focus on works most closely related to our own, distinguishing between local and global optimization strategies.

\subsection{Local Methods}
These methods are efficient but do not guarantee global optimality, making them susceptible to local minima.

\textbf{Simultaneous Pose and Correspondence.} The venerable Iterative Closest Point (ICP) algorithm~\cite{ICP,ICP2} alternates between finding point correspondences and updating the spatial transformation. While efficient, its reliance on discrete assignments often leads it to get trapped in local minima. To enhance robustness, variations like that by Zhang \textit{et al.}~\cite{ICP_2D_rot} constrain the rotation angle, while others incorporate additional information, such as RGB data \cite{ICP_rgb_correntropy}.

The Robust Point Matching (RPM) method~\cite{RPM_TPS} addresses this by relaxing correspondences to have fuzzy values, using deterministic annealing (DA) to gradually refine them. Sofka \textit{et al.}~\cite{CDC} used the covariance of transformation parameters to control this fuzziness, achieving better robustness to missing structures. For non-rigid registration, Ma \textit{et al.}~\cite{ieee7001713} employed L2E for robust transformation estimation.

Our method, like these approaches, jointly models spatial transformation and correspondence. However, we achieve greater robustness to various disturbances by guaranteeing global optimality, which is not possible with these local methods.

\textbf{Correspondence-free Methods.} These methods, such as Coherent Point Drift (CPD)~\cite{CPD_match}, avoid explicit correspondence matching. CPD frames registration as fitting a Gaussian Mixture Model (GMM) from one point set to another. GMMREG~\cite{kernel_Gaussian_journal} represents both point sets with GMMs and minimizes the L$^2$ distance between them. Similarly, SVGMs~\cite{mixture_SVC} use sparse Gaussian components. GMM-based methods have been made more efficient with filtering techniques~\cite{GMM_filtering}, while issues with point set density have been addressed by modeling scene structure~\cite{mixture_average} and using hierarchical representations~\cite{tree_Gaussian_mixture}. These methods are computationally intensive and can still suffer from local minima, whereas our approach offers a global optimality guarantee.

\subsection{Global Methods}
These methods offer guarantees of finding the globally optimal solution, albeit at a higher computational cost.

\textbf{BnB-based Methods.} The Branch-and-Bound (BnB) algorithm is a powerful global optimization technique. An early application based on Lipschitz theory~\cite{Lipschitz_3D_align} was sensitive to outliers. More recent methods leverage the specific geometry of transformations. Go-ICP~\cite{Go-ICP_pami} globally optimizes the ICP objective by exploring the structure of 3D rigid motions. GOGMA~\cite{BnB_mixture_Gaussian} registers point sets by aligning their corresponding GMMs. Other works, such as the Fast Rotation Search (FRS)~\cite{BnB_consensus_project}, efficiently find 3D rotations, while a Bayesian approach~\cite{BnB_Bayesian_mixture} uses a novel tessellation of the rotation space.

While these methods offer global guarantees, they often focus on specific problems (e.g., rigid motion) and do not handle correspondence explicitly, which can limit their applicability to non-rigid deformations. Our method differs by \textbf{jointly and globally optimizing both transformation and correspondence}, making it more versatile.

\textbf{Mismatch Removal Methods.} An alternative research direction focuses on finding the correct spatial transformation from an initial set of putative correspondences that may contain many errors. The Fast Global Registration (FGR) method~\cite{fast_global_regist} optimizes a robust objective based on line processes. GORE~\cite{GORE_outlier_removal} prunes outliers using geometric constraints before applying RANSAC. TEASER++~\cite{TEASER} uses a graph-theoretic approach to decouple the estimation of rotation, scale, and translation, formulating each as a truncated least squares problem. TR-DE \cite{9878458} also decomposes the 6-DOF transformation to reduce computational complexity within a two-stage BnB search. TEAR \cite{10656079} uses a truncated entry-wise absolute residual function for outlier robustness.

These methods excel at handling high outlier rates but require an initial set of correspondences, which can be difficult to obtain robustly in highly cluttered or unstructured environments. Our method avoids this by simultaneously solving for both correspondence and transformation from scratch.

\section{Bilinear Convex Envelope Linearization \label{sec:bil_lb}}

Our approach to constructing a lower bound for the BnB algorithm relies on the \textbf{convex envelope} of a bilinear monomial $xy$. For variables $x$ and $y$ within a rectangular domain defined by $\underline{x}\le x\le \overline{x}$ and $\underline{y}\le y\le \overline{y}$, the convex envelope is the maximum of two linear functions:
\begin{gather}
	(xy)_{l} = \max\left\{ \underline{x}y+x\underline{y}-\underline{x}\underline{y}, \quad \overline{x}y+x\overline{y}-\overline{x}\overline{y} \right\} \le xy.
	\label{bil_conv_env}
\end{gather}
While this provides the tightest possible linear lower bound, using a `max` function within our objective would result in a non-linear problem, complicating the subproblem optimization.

To maintain linearity and ensure computational efficiency, we 
%follow a common practice in global optimization and 
use a simpler, \textbf{linearized lower bound}. We use the average of the two linear components from the convex envelope:
\begin{equation}
	(xy)_{l} \triangleq
	\frac{1}{2} \left( \left(\underline{x}y+x\underline{y}-\underline{x}\underline{y}\right) + \left(\overline{x}y+x\overline{y}-\overline{x}\overline{y}\right) \right) \le xy. \label{bil_lin_lb}
\end{equation}
This choice results in a strictly linear function, which is crucial for decomposing our lower bound subproblem into solvable components.

A key property of this average-based lower bound is its relationship to the branching process. As shown by \cite{LIAN2023126482},
% in their trilinear relaxation context,
  for the lower bound to converge to the true value of the monomial, it is sufficient to branch on only one of the variables—either $x$ or $y$. This property is fundamental to the efficiency of our algorithm, as it allows us to effectively reduce the search space by branching only on the transformation parameters $\boldsymbol{\theta}$ while leaving the correspondence variables $p_{ij}$ unbranched. This single-variable branching strategy is the cornerstone of our algorithm's convergence and its ability to scale to larger problems.

\section{Problem formulation \label{sec:problem}}

We aim to find a globally optimal alignment between two point sets, $\mathscr{X}=\{ \mathbf x_i\}_{i=1}^{n_x}$
and
$\mathscr{Y}=\{ \mathbf y_j\}_{j=1}^{n_y}$ 
in  $\mathbb R^{n_d}$. Following the mixed linear assignment-least squares framework of \cite{RPM_model_occlude}, the problem is formulated as minimizing the following objective function:
\begin{subequations}	
		\begin{gather}
			\min E(\mathbf P,\boldsymbol{\theta})=\sum_{i,j}p_{ij}\|\mathbf y_j-\mathbf J(\mathbf x_i) \boldsymbol{\theta}\|^2  \label{RPM_obj} \\
			=%\notag \\
			{\boldsymbol\theta}^\top \mathbf J^\top  (\text{diag}(\mathbf P  \mathbf 1_{n_y}) \otimes \mathbf I_d) \mathbf J  {\boldsymbol\theta}   
			-2{\boldsymbol\theta}^\top \mathbf J^\top (\mathbf P\otimes \mathbf I_{d}) \mathbf y 
+			 \mathbf 1_{n_x}^\top \mathbf P \widetilde{\mathbf y} 
\label{energy_case_one_generic}			\\
			s.t.\ \mathbf P  \mathbf 1_{n_y}\le  \mathbf 1_{n_x},\
			\mathbf 1_{n_x}^\top \mathbf P\le  \mathbf 1_{n_y}^\top,\
			\mathbf 1_{n_x}^\top \mathbf P  \mathbf 1_{n_y}=n_p,\  %\notag\\
			\mathbf P\ge 0, \
			\underline{\boldsymbol{\theta}}\le \boldsymbol{\theta}\le \overline{\boldsymbol{\theta}}
			% \}
			\label{k_card_P_const}
		\end{gather}
\end{subequations}		
where $\mathbf P=\{p_{ij}\}$ is the binary assignment matrix, with $p_{ij}=1$ if point $\mathbf x_i$ matches $\mathbf y_j$, and 0 otherwise. The spatial transformation is parameterized by $\boldsymbol{\theta}$ such that $\mathbf T(\mathbf x_i |\boldsymbol{\theta})=\mathbf J(\mathbf x_i)\boldsymbol\theta$, where $\mathbf J(\mathbf x_i)$ is the Jacobian matrix at point $\mathbf x_i$. The constraint $\mathbf 1_{n_x}^\top \mathbf P  \mathbf 1_{n_y}=n_p$ specifies that the number of matches is a known constant $n_p$, which models partial overlap.
	$\mathbf I_{d}$ is  the $d\times d$ identity matrix. 
	$\mathbf 1_{n_x}$ is the $n_x$-dimensional  vector of  all ones. 
	$\otimes$ is the Kronecker product. 	$\underline{\boldsymbol{\theta}}$ (resp. 
	$\overline{\boldsymbol{\theta}}$) is 
	the  lower (resp.  upper) bound of $\boldsymbol{\theta}$.
Vectors	$\mathbf y\triangleq\begin{bmatrix}
	\mathbf y_1^\top, \dots, \mathbf y_{n_y}^\top
\end{bmatrix}^\top$,  %\hline
$\widetilde{\mathbf {y}}\triangleq \begin{bmatrix}
	\|\mathbf y_1\|_2^2, \dots, \|\mathbf y_{n_y}\|_2^2
\end{bmatrix}^\top$
and matrix $\mathbf J\triangleq\begin{bmatrix}
	\mathbf J^\top(\mathbf x_1), \dots, \mathbf J^\top(\mathbf x_{n_x})
		\end{bmatrix}^\top$.
		
For the purpose of optimization, we reformulate the objective by vectorizing the matrices. By converting the assignment matrix $\mathbf P$ to a vector $\mathbf p=\text{vec}(\mathbf P)$, the objective function can be written in a more compact form:
\begin{gather}
	E(\mathbf p,\boldsymbol \theta)=\boldsymbol\theta^\top\text{mat} (\mathbf B\mathbf p) \boldsymbol{\theta}-
	2\boldsymbol\theta^\top \mathbf A\mathbf p  + 
	\boldsymbol \rho^\top \mathbf p  
\end{gather}
Here, $\mathbf A\triangleq(\mathbf{J}^\top \otimes\mathbf y^\top )\mathbf W^{n_x,n_y}_{n_d}$, 
$\mathbf B\triangleq(\mathbf{J}_2^\top \otimes \mathbf I_{n_\theta}) \mathbf W^{n_x,1}_{n_\theta} (\mathbf I_{n_x} \otimes\mathbf 1_{n_y}^\top )$,
 and $\boldsymbol\rho\triangleq\mathbf 1_{n_x}  \otimes\widetilde{\mathbf y}$ are constant matrices and vectors derived from the point coordinates and Jacobian matrices,
Matrices $
\mathbf{J}_2\triangleq\begin{bmatrix}
	\mathbf{J}(\mathbf x_1)^\top \mathbf{J}(\mathbf x_1),       \ldots,      \mathbf{J}(\mathbf x_{n_x})^\top \mathbf{J}(\mathbf x_{n_x})
\end{bmatrix}^\top
$ and 
 	$
 \mathbf W^{m,n}_d\triangleq
 \mathbf I_m \otimes  \begin{bmatrix}
 	\mathbf I_n \otimes (\mathbf e_d^1)^\top,
 	\dots,
 	\mathbf I_n \otimes (\mathbf e_d^d)^\top
 \end{bmatrix}^\top
 $.
$n_\theta$ denotes the dimensionality of ${\boldsymbol{\theta}}$.
$\mathbf e_d^i$ is the  $d$-dimensional column vector 
with only  one nonzero unit  element at the $i$-th position.
$\text{mat}(\cdot)$ reconstructs a  matrix  from a vector,
which is   the inverse of the operator $\text{vec}(\cdot)$. 

Since $\mathbf  1^\top_{n_xn_y} \mathbf  p=n_p$, a constant,
for rows of $\mathbf B$
containing identical elements,
%equal to scaled versions  of $\mathbf  1^\top_{n_xn_y}$ 
the result of them multiplied by $\mathbf p$  can be replaced by  constants.
Also, redundant rows can  be removed.
Since $\text{mat}(\mathbf B\mathbf p)$ is a symmetric matrix,
$\mathbf B$ surely   contains redundant rows.
Based on the above  analysis,
we therefore   denote $\mathbf B_2$ 
as the matrix formed as a result of $\mathbf B$ 
removing such  rows.
%It can be verified that for
%2D similarity/affine transformations and 3D scaling + translation transformation
(Please refer to Sec.  \ref{sec:app_one} and  \ref{sec:app_two} for examples of $\mathbf B_2$).
% for different types of transformations. 
%$\mathbf \Xi$ and $\mathbf D$ contains such rows.
%
Consequently,
$E$ can be  rewritten as \cite{LIAN2023126482}
%目标函数可简写成如下形式：
\begin{gather}
	E(\mathbf p, \boldsymbol\theta)=\boldsymbol\theta^\top \left[\text{mat} ( \mathbf K \mathbf B_2 \mathbf p   
	)+ \mathbf C\right] \boldsymbol{\theta}  
	-
	2\boldsymbol\theta^\top \mathbf A \mathbf p  +
	\boldsymbol\rho^\top \mathbf p   \label{eng_case_one_vec}
%	\\
%	\mathbf p\in \Omega, \quad  
%	\underline{\boldsymbol{\theta}}\le \boldsymbol{\theta}\le \overline{\boldsymbol{\theta}}
\end{gather}
where the nonzero elements of the constant matrix $\mathbf C$ correspond to the  rows of $\mathbf B$ containing identical elements. 
(Please refer to Sec.  \ref{sec:app_one} and  \ref{sec:app_two} for examples of $\mathbf C$).
The elements of the constant matrix $\mathbf K$ 
%satisfies $B=F\mathbf B_2$,
take  values  in $\{0,1\}$ and   record the correspondences  between the rows of $\mathbf B_2 $ and those of $\mathbf B$.

 \begin{figure*}
 		\begin{adjustwidth}{-1.8cm}{0cm}
 	\centering
 	\begin{tikzcd}[row sep=scriptsize, column sep=-0.9em]
 		E(\mathbf p, \boldsymbol\theta)=&	\textcolor{red}{\boxed{\boldsymbol\theta^\top \text{mat} (  \mathbf K \mathbf B_2 \mathbf p   
 				) \boldsymbol\theta }}
 		\arrow[d,"\text{define }\boldsymbol{\Theta}\triangleq \boldsymbol\theta\boldsymbol\theta^\top\text{,}" swap, "\boldsymbol\Gamma \triangleq \text{mat}(K\mathbf B_2p)"]
 		&+   \boldsymbol\theta^\top \mathbf C  \boldsymbol\theta
 		& \textcolor{blue}{\boxed{-
 				2\boldsymbol\theta^\top  \mathbf A \mathbf p }} 
 		\arrow[d,"\text{define}" swap, "\boldsymbol\eta \triangleq -2\mathbf A\mathbf p" ]
 		&{+} 
 		\boldsymbol\rho^\top\mathbf p   
 		\\		
 		E(\mathbf p,\boldsymbol\theta,\boldsymbol{\Theta},\boldsymbol\Gamma,\boldsymbol\eta)= &\textcolor{red}{				\boxed{
 				\text{trace} ( \boldsymbol\Gamma\boldsymbol{\Theta})
 		} }	 
 		\arrow[d,"\text{bilinear}" swap,"\text{relaxation}"]
 		&+ \boldsymbol\theta^\top  \mathbf C  \boldsymbol\theta 
 		&+
 		\textcolor{blue}{				\boxed{\boldsymbol\theta^\top \boldsymbol\eta  }} 
 		\arrow[d,"\text{bilinear}" swap,"\text{relaxation}"]
 		&+ 
 		\boldsymbol\rho^\top\mathbf p  \\
 		E_l(\mathbf p,\boldsymbol\theta,\boldsymbol{\Theta},\boldsymbol\Gamma,\boldsymbol\eta)=&\textcolor{red}{
 			\boxed{
 				\text{trace}(\mathbf H^0\boldsymbol{\Theta})
 				+\text{trace}(\mathbf H^1 \boldsymbol\Gamma)
 				+\text{trace}(\mathbf H^2 \mathbf 1 \mathbf 1^\top )
 		}}	
 		\arrow[d,"\text{substitute back }\boldsymbol\theta\boldsymbol\theta^\top\Leftarrow \boldsymbol{\Theta}\text{,}" swap, "\text{mat}(K\mathbf B_2p)\Leftarrow\boldsymbol\Gamma"]
 		&
 		+ \boldsymbol\theta^\top  \mathbf C  \boldsymbol\theta 
 		&
 		+\textcolor{blue}{				\boxed{
 				%		\begin{matrix}	
 				(\mathbf g^0)^\top	\boldsymbol{\theta}   
 					+ (\mathbf g^1)^\top \boldsymbol\eta  
 					+ \mathbf 1_{n_\theta}^\top \mathbf g^2 
 					%		\end{matrix}		
 		}}
 		\arrow[d,"\text{substitute back}" swap, "-2\mathbf A\mathbf p \Leftarrow \boldsymbol\eta" ]
 		&+ 
 		\boldsymbol\rho^\top\mathbf p \\
 		E_l(\mathbf p,\boldsymbol\theta)=&\textcolor{red}{
 			\boxed{
 				\boldsymbol\theta^\top \mathbf H^0 \boldsymbol\theta
 				+\text{trace}(\mathbf H^1 \text{mat} (\mathbf K \mathbf B_2 \mathbf p   
 				)  )
 				+\text{trace}(\mathbf H^2 \mathbf 1 \mathbf 1^\top )
 		}}	&
 		+ \boldsymbol\theta^\top  \mathbf C  \boldsymbol\theta 
 		&
 		+
 		\textcolor{blue}{				\boxed{
 				%		\begin{matrix}	
 				 (\mathbf g^0)^\top	\boldsymbol{\theta}  
 					+ (\mathbf g^1)^\top (-2 \mathbf A \mathbf p  ) 
 					+ \mathbf 1_{n_\theta}^\top \mathbf g^2 
 					%		\end{matrix}	
 		}}	
 		&+ 
 		\boldsymbol\rho^\top\mathbf p
 	\end{tikzcd}
 	\end{adjustwidth}
 	\caption{Derivation of a lower bound function $E_l(\mathbf p,\boldsymbol\theta)$ of  $E(\mathbf p,\boldsymbol\theta)$ via variable substitution and  bilinear relaxation.
 		\label{bil_relax_derive}
 	}
 \end{figure*}

\section{Optimization \label{sec:optimize}}

This section details the optimization strategy for our Branch-and-Bound (BnB) algorithm. We begin by deriving a lower bound for our objective function using bilinear relaxation (Sec. \ref{subsec:lb_fun}). We show that the minimization of this lower bound can be efficiently decoupled into two independent subproblems. Next, we present our branching strategy (Sec. \ref{subsec:branch_case_one}), demonstrating that it's sufficient to branch on only the low-dimensional transformation parameters. Finally, we discuss how to compute a more effective upper bound (Sec. \ref{subsec:ub_case_one}) and outline the complete BnB algorithm (Sec. \ref{sec:BnB_aff}).

%In this section, first, 
%we use the bilinear relaxation formula as introduced in Sec. \ref{sec:bil_lb}
%to derive a lower bound function of our objective function.
%The optimization of  lower bound function  can be decomposed into separate
%optimization over $\boldsymbol\theta$ and  $\mathbf p$.
%(Sec.\ \ref{subsec:lb_fun}).
%Next, we show we only need to  branch over $\boldsymbol\theta$ (Sec.\ \ref{subsec:branch_case_one})
%and 
%give a more effective way of  computing  the upper bound (Sec.\ \ref{subsec:ub_case_one}).
%Finally, we  present the   BnB algorithm (Sec.\ \ref{sec:BnB_aff})
%which  globally optimizes the objective function.
%together with  its convergence property (Sec.\ \ref{sec:converge_2d}).

\subsection{Lower bound function derivation
	\label{subsec:lb_fun}}

We derive a lower bound function of $E(\mathbf p,\boldsymbol\theta)$ via variable substitution and bilinear relaxation, as illustrated in Fig. \ref{bil_relax_derive}.
In the following,
we explain each step of the approach in detail.

By introducing new  variables $\boldsymbol{\Theta}\triangleq\boldsymbol\theta \boldsymbol\theta^\top$, $\boldsymbol\Gamma \triangleq \text{mat}(\mathbf K \mathbf B_2 \mathbf p)$ and $\boldsymbol\eta\triangleq -2\mathbf A\mathbf p$,
we can rewrite the above objective function  as a quadratic function:
% constrained quadratic program:
\begin{gather}
%	\min_{ \boldsymbol{\Theta}, \boldsymbol\theta,p}
	  E(\mathbf p,\boldsymbol\theta,\boldsymbol{\Theta},\boldsymbol\Gamma,\boldsymbol\eta)=
	 \text{trace} ( \boldsymbol\Gamma \boldsymbol{\Theta} )
+	\boldsymbol\theta^\top \mathbf C \boldsymbol\theta
	+
	\boldsymbol\theta^\top \boldsymbol\eta   {+} 
	\boldsymbol\rho^\top\mathbf p  
%		\notag\\
%	s.t.\ \boldsymbol{\Theta} = \boldsymbol\theta \boldsymbol\theta^\top,\
%	\boldsymbol\Gamma = \text{mat}(K \mathbf B_2 \mathbf p),\
%	\boldsymbol\eta= -2\mathbf A\mathbf p,\
%	\mathbf p\in \Omega, \
%	\underline{\boldsymbol{\theta}}\le \boldsymbol{\theta}\le \overline{\boldsymbol{\theta}}
\end{gather} 

We note that   $E$ contains bilinear terms   $\text{trace}(\boldsymbol\Gamma \boldsymbol{\Theta})$ and $\boldsymbol\theta^\top \boldsymbol\eta $, 
which interconnect $\boldsymbol\Gamma$ and $\boldsymbol{\Theta}$ (resp. $\boldsymbol\theta$ and $\boldsymbol\eta$), 
making optimization challenging. 
To tackle this challenge, 
we use the bilinear relaxation formula introduced in Sec. \ref{sec:bil_lb}
to relax these terms into linear ones.
%, as depicted in Figure \ref{bil_relax_derive}. 
The first relaxation is:
	\begin{equation} % [H]
\begin{tikzcd}[row sep=scriptsize, column sep=-0.9em]				
			\textcolor{red}{				\boxed{
					\text{trace} ( \boldsymbol\Gamma\boldsymbol{\Theta})
			} }	
\arrow[d,"\text{bilinear}" swap,"\text{relaxation}"]				 \\	 
			\textcolor{red}{
				\boxed{
					\text{trace}(\mathbf H^0\boldsymbol{\Theta})
					+\text{trace}(\mathbf H^1 \boldsymbol\Gamma  )
					+\text{trace}(\mathbf H^2 \mathbf 1 \mathbf1^\top )
			}}						
\end{tikzcd}
%		\label{lb_derive}
	\end{equation} 
where
in accordance with equation % \eqref{bil_cov_env} and 
\eqref{bil_lin_lb},
%and utilizing $\boldsymbol\Gamma = \text{mat}(K \mathbf B_2 \mathbf p)$, 
the coefficients can be calculated as 
% according to \eqref{bil_cov_env} and \eqref{bil_lin_lb},
%the coefficients can be computed as
\begin{gather}
\mathbf H^0=\frac{\underline{\boldsymbol\Gamma}+\overline{\boldsymbol\Gamma}}{2}
%=\frac{ \underline{\text{mat} (\mathbf K \mathbf B_2 \mathbf p   
%		)}+
%	 \overline{\text{mat} (\mathbf K \mathbf B_2 \mathbf p   
%		)}  
%}{2}
,
\mathbf H^1=\frac{\underline{\boldsymbol{\Theta}}+\overline{\boldsymbol{\Theta}}}{2},
\mathbf H^2=-\frac{ \underline{\boldsymbol\Gamma} \underline{\boldsymbol{\Theta}}+
	\overline{\boldsymbol\Gamma}  \overline{\boldsymbol{\Theta}} 
}{2}
%=
%-\frac{ \underline{\text{mat} (\mathbf K \mathbf B_2 \mathbf p   
%		)} \underline{\boldsymbol{\Theta}}+
%	 \overline{\text{mat} (\mathbf K \mathbf B_2 \mathbf p   
%		)}  \overline{\boldsymbol{\Theta}} 
%}{2}
\end{gather}
Here we use underline 
(resp.  overline)
%$\overline{\boldsymbol{\Theta}}$
 to denote the lower (resp. upper) bounds of matrices $\boldsymbol{\Theta}$ and $\boldsymbol\Gamma$.
%  and $\text{mat} (\mathbf K \mathbf B_2 \mathbf p   )$.
%, as well as vectors $Ap$ and $\boldsymbol\theta$.
It is  also worth noting  that the range  
$[\underline{\boldsymbol{\Theta}}$, $\overline{\boldsymbol{\Theta}}]$ 
of $\boldsymbol{\Theta}$ 
can be derived  from the range $[\underline{\boldsymbol\theta}
,\overline{\boldsymbol\theta}]$ of $\boldsymbol\theta$ based on relation $\boldsymbol{\Theta}=\boldsymbol\theta \boldsymbol\theta^\top$, as explained in Sec. \ref{sec:range_of_Theta}.

The second relaxation  is:
	\begin{equation} % [H]
\begin{tikzcd}[row sep=scriptsize, column sep=-0.9em]	
			\textcolor{blue}{				\boxed{\boldsymbol\theta^\top \boldsymbol\eta }}
\arrow[d,"\text{bilinear}" swap,"\text{relaxation}"]				 \\	
			\textcolor{blue}{				\boxed{
					\begin{matrix}	
					 (\mathbf g^0)^\top	\boldsymbol{\theta}  
						+ (\mathbf g^1)^\top \boldsymbol\eta 
						+ \mathbf 1_{n_\theta}^\top \mathbf g^2 
					\end{matrix}		
			}}	
\end{tikzcd}
%		\label{lb_derive}
	\end{equation} 
where in accordance with equation %\eqref{bil_cov_env} and
 \eqref{bil_lin_lb},
%and utilizing $\boldsymbol\eta= -2\mathbf A\mathbf p$,
the coefficients can be calculated as 
\begin{gather}
\mathbf g^0=\frac{\underline{\boldsymbol\eta} + \overline{\boldsymbol\eta}}{2},
%=\frac{\left(\underline{-2\mathbf A\mathbf p}\right)+\left(\overline{-2\mathbf A\mathbf p}\right) }{2},
\mathbf g^1=\frac{\underline{\boldsymbol\theta}+\overline{\boldsymbol\theta}}{2},
\mathbf g^2=
\frac{\underline{\boldsymbol\eta}\underline{\boldsymbol\theta}+\overline{\boldsymbol\eta }\overline{\boldsymbol\theta}}{2}
%=\frac{\left(\underline{-2\mathbf A\mathbf p}\right)\underline{\boldsymbol\theta}+\left(\overline{-2\mathbf A\mathbf p}\right)\overline{\boldsymbol\theta}}{2}
\end{gather}
Here, the underline and overline have the same meaning as before.

Taking into account the above relaxations,
we get the following  lower bound  function:
% where the objective takes a linear form:
%\begin{subequations}
\begin{align}
%	\min_{p, \boldsymbol\theta,\boldsymbol{\Theta},\boldsymbol\Gamma,\boldsymbol\eta} \ 
	E_l(\mathbf p,\boldsymbol\theta,\boldsymbol{\Theta},\boldsymbol\Gamma,\boldsymbol\eta)=&
	\text{trace}(\mathbf H^0\boldsymbol{\Theta})
	+\text{trace}(\mathbf H^1 \boldsymbol\Gamma ) 
		+\text{trace}(\mathbf H^2 \mathbf 1 \mathbf 1^\top ) 
	 \notag\\
	&+ \boldsymbol\theta^\top  \mathbf C \boldsymbol\theta
	+
	\boldsymbol\theta^\top \mathbf g^0 + (\mathbf g^1)^\top \boldsymbol\eta +\mathbf 1^\top \mathbf g^2
	{+} 
	\boldsymbol\rho^\top\mathbf p  	 %\\
%	s.t.\ 	\boldsymbol{\Theta} = \boldsymbol\theta \boldsymbol\theta^\top,\
%	\boldsymbol\Gamma = \text{mat}(K \mathbf B_2 \mathbf p),\
%	\boldsymbol\eta= -2\mathbf A\mathbf p,\
%\mathbf p\in \Omega,\
%\ \underline{\boldsymbol\theta}\le \boldsymbol\theta\le \overline{\boldsymbol\theta}
\end{align}
%\end{subequations}

By substituting back  $\boldsymbol\theta \boldsymbol\theta^\top\Leftarrow \boldsymbol{\Theta}$, $ \text{mat}(K \mathbf B_2 \mathbf p)\Leftarrow \boldsymbol\Gamma $ and $-2\mathbf A\mathbf p\Leftarrow \boldsymbol\eta$,
we get the final form of the  lower bound function:
\begin{align}
	E_l(\mathbf p,\boldsymbol\theta)=&
%	\textcolor{red}{		\boxed{
			\boldsymbol\theta^\top \mathbf H^0 \boldsymbol\theta
			+\text{trace}(\mathbf H^1 \text{mat} (\mathbf K \mathbf B_2 \mathbf p   
			)  )
			+\text{trace}(\mathbf H^2 \mathbf 1 \mathbf 1^\top )
%	}}
	\notag \\
&	+ \boldsymbol\theta^\top  \mathbf C  \boldsymbol\theta 	
	+
%	\textcolor{blue}{				\boxed{
			%		\begin{matrix}	
				\boldsymbol{\theta}^\top  \mathbf g^0 
				+ (\mathbf g^1)^\top (-2 \mathbf A \mathbf p  ) 
				+ \mathbf 1_{n_\theta}^\top \mathbf g^2 
				%		\end{matrix}	
%	}}	
	+ 
	\boldsymbol\rho^\top\mathbf p
\end{align}

%It's clear that minimizing ElEl​ under constraints \eqref{k_card_P_const} can be broken down into separate optimizations for pp and θθ.

It is apparent that the minimization of $E_l$ under constraints \eqref{k_card_P_const} can be  decomposed into  separate optimizations over  $\mathbf p$ and   $\boldsymbol\theta$.

\textbf{1. Optimization over $\mathbf{p}$}.
This subproblem is a linear program in terms of the correspondence vector $\mathbf{p}$.
%The optimization over $\mathbf p$ is
\begin{subequations}
	\begin{gather}
		\min  _{ p} \ 
		\text{trace}(\mathbf H^1 \text{mat} (\mathbf K \mathbf B_2 \mathbf p   
		)  )		
		+ (\mathbf g^1)^\top (-2\mathbf A\mathbf p) 
		+ 
		\boldsymbol\rho^\top\mathbf p  	 \\
		s.t.\ 	\mathbf p\in \Omega
	\end{gather}
\end{subequations}
where $\Omega$ denotes the feasible region of $\mathbf p$ as defined in 
\eqref{k_card_P_const}.
This is  an $n_p-$cardinality linear assignment problem
and  can be 
transformed into  a standard
linear assignment problem \cite{k_card_assign_transform} and then
efficiently solved
by combinatorial optimization algorithms such as LAPJV 
% the Jonker-Volgenant algorithm 
\cite{LAPJV}.

\textbf{2. Optimization over $\boldsymbol\theta$}.
This subproblem is a quadratic program in terms of the transformation vector $\boldsymbol\theta$.
%The optimization over  
%$\boldsymbol\theta$ is:
\begin{subequations}
\begin{gather}
	\min  _{  \boldsymbol\theta} \ 
	\boldsymbol\theta^\top (\mathbf H^0 +\mathbf C) \boldsymbol\theta 
	+
	\boldsymbol\theta^\top \mathbf g^0 	 \label{subprob_one}\\
	s.t.\ 	 \underline{\boldsymbol\theta}\le \boldsymbol\theta\le \overline{\boldsymbol\theta}
	\label{subprob_one_const}
\end{gather}
\end{subequations}
In the next section, we will delve into the solution for this problem.
%We postpone to next section to discuss the solution of this problem.

%

\subsection{Branching strategy
	\label{subsec:branch_case_one}
}

According to  the result stated in Sec. \ref{sec:bil_lb},
for the bilinear term $\boldsymbol{\theta}^\top \boldsymbol\eta $,
for  its linear lower bound function to converge to it,
it suffices to  branch over $\boldsymbol{\theta}$
and leave the range of $\boldsymbol\eta$ fixed.
Here, 
utilizing $\boldsymbol\eta= -2\mathbf A\mathbf p$,
 the fixed range $[\underline{\boldsymbol\eta}, \overline{\boldsymbol\eta}]$
 of $\boldsymbol\eta$ can be computed as 
\begin{equation}
\underline{\eta} _i
%=\underline{-2\mathbf A\mathbf p}
=	\min_{\mathbf p\in \Omega} (-2\mathbf A \mathbf p)_i
% \label{range_prob_1_a} 
,\quad 
%\le (-2\mathbf A \mathbf p)_i \le
\overline{\eta}_i
%=\overline{-2\mathbf A\mathbf p}
= \max_{\mathbf p\in \Omega}  (-2\mathbf A \mathbf p)_i  \label{range_prob_1}
\end{equation}

Likewise,
%based on the result in Sec.\ \ref{sec:lb_tri},
for the  bilinear term
$\text{trace}(\boldsymbol\Gamma \boldsymbol{\Theta})$,
for its linear lower bound function to converge to it,
we only need to branch over  
$\boldsymbol{\theta}$ (since the range of $\boldsymbol\theta$ determines the range of $\boldsymbol{\Theta}$
as explained  in Sec. \ref{sec:range_of_Theta})
and leave the range of $\boldsymbol\Gamma$
fixed.
%since %for our problem,
%$\boldsymbol{\boldsymbol{\Theta}}$ constitutes two sets of variables that make up 
%$\text{trace}(\text{mat} ( K  \mathbf B_2  p) \boldsymbol{\Theta})$.
%thus,
%we only need to branch the ranges of  
%
Here, 
 utilizing $\boldsymbol\Gamma = \text{mat}(K \mathbf B_2 \mathbf p)$, 
 the fixed  range 
$[\underline{\boldsymbol\Gamma},
\overline{\boldsymbol\Gamma} ]$
of $\boldsymbol\Gamma $ can be computed as follows:
%$\min_{\mathbf p\in \Omega} [\text{mat} ( \mathbf K\mathbf B_2\mathbf p   
%)-\mathbf D]_{ij} \le 
%[\text{mat} ( \mathbf K\mathbf B_2\mathbf p   
%)-\mathbf D]_{ij}
%\le
%\max_{\mathbf p\in \Omega}  [\text{mat} ( \mathbf K\mathbf B_2\mathbf p   
%)-\mathbf D]_{ij}$
%via solving linear assignment problems.
%We recognize these as linear assignment problems which can be efficiently solved.
%
%
\begin{gather}
\underline{\Gamma}	_{ij}
%=\underline{\text{mat} (\mathbf K \mathbf B_2 \mathbf p   	)}	
=\min_{\mathbf p\in \Omega} [\text{mat} ( \mathbf K\mathbf B_2\mathbf p   
	)]_{ij},\quad  %\label{range_prob_2_a} \\
\overline{\Gamma} _{ij}
%=\overline{\text{mat} (\mathbf K \mathbf B_2 \mathbf p   	)}	
=\max_{\mathbf p\in \Omega}  [\text{mat} ( \mathbf K\mathbf B_2\mathbf p   
	)]_{ij} \label{range_prob_2}
\end{gather}

With these choice of $\underline{\boldsymbol\Gamma}$ and $\overline{\boldsymbol\Gamma}$, it has been empirically demonstrated  that
\begin{equation}
	\mathbf H^0+\mathbf C=
	\frac{\left( \underline{\boldsymbol\Gamma}+\mathbf C\right)+
		\left( \overline{\boldsymbol\Gamma}+\mathbf C\right)
	}{2} \succeq 0 \label{SDP_quad_term}
\end{equation}
Please refer to Sec. \ref{subsec:PSD}
for relevant experiments and
  \ref{sec:appendix} for a weaker assertion.
This  result implies that
problem \eqref{subprob_one} \eqref{subprob_one_const} is now a convex quadratic program.
Consequently,
utilizing solvers for convex quadratic programs, such as MATLAB's quadprog function, guarantees  global optimality.
We will address the proof of \eqref{SDP_quad_term} 
%is deferred to 
in our future research.

% \eqref{range_prob_1_a},
%\eqref{range_prob_2_a}
We also note  that  \eqref{range_prob_1}  and \eqref{range_prob_2}  are $n_p-$cardinlaity linear assignment problems,
which  can be efficiently solved by the aforementioned  algorithms.

In conclusion, in our BnB algorithm,
we only need to branch over $\boldsymbol{\theta}$.
Therefore, the dimension of the branching space of our BnB algorithm is low and the proposed algorithm can converge quickly.

\subsection{The range of $\Theta$ from the range of $\theta$
\label{sec:range_of_Theta}
}

Given the relationship 
  $\boldsymbol{\Theta}=\boldsymbol\theta \boldsymbol\theta^\top$,
and the specified range $\left[\underline{\boldsymbol\theta},
\overline{\boldsymbol\theta}\right]$  of $\boldsymbol\theta$, 
we can determine the range $\left[\underline{\boldsymbol{\Theta}},
\overline{\boldsymbol{\Theta}}\right]$ of $\boldsymbol{\Theta}$ as follows:

\[
[\underline{\Theta_{ij}},\overline{\Theta_{ij}}]=
\begin{cases}
	[\underline{\theta_i},\overline{\theta_i}]
	\cdot
	[	\underline{\theta_j}, \overline{\theta_j}],	
	& \text{if}\ i\neq j	
	\\
[\underline{\theta_i},\overline{\theta_i}]^2,
	 &\text{if}\ i= j
\end{cases}
\]
Here, employing interval arithmetic, the product operator is defined as:
\[
[\underline{\theta_i},\overline{\theta_i}]
\cdot
[	\underline{\theta_j}, \overline{\theta_j}] =
[\min(\underline{\theta_i}\underline{\theta_j},
\underline{\theta_i}\overline{\theta_j},
\overline{\theta_i}\underline{\theta_j},
\overline{\theta_i}\overline{\theta_j}
) ,
\max(\underline{\theta_i}\underline{\theta_j},
\underline{\theta_i}\overline{\theta_j},
\overline{\theta_i}\underline{\theta_j},
\overline{\theta_i}\overline{\theta_j}
)
]
\]
Additionally, the square operator is defined as:
\[
[\underline{\theta_i},\overline{\theta_i}]^2 =
\begin{cases}
[	\min(\underline{\theta_i}^2,
\overline{\theta_i}^2
),
	\max(\underline{\theta_i}^2,
\overline{\theta_i}^2
)
],
& \text{if}\  0\notin[\underline{\theta_i},
\overline{\theta_i}]
\\
[	0,
\max(\underline{\theta_i}^2,
\overline{\theta_i}^2
)
],
& \text{if}\  0\in[\underline{\theta_i},
\overline{\theta_i}]
\end{cases}
\]

\subsection{\label{subsec:ub_case_one}Upper bound} 

When we substitute the solutions for $\mathbf{p}$ and $\boldsymbol{\theta}$ obtained during the computation of the lower bound into $E(\mathbf{p}, \boldsymbol{\theta})$, we can derive an upper bound. However, the solution for $\boldsymbol{\theta}$ obtained in this process might significantly differ from the optimal $\boldsymbol{\theta}$, especially in the early iterations of the BnB algorithm. This discrepancy can lead to a poor computed upper bound.

To address this issue, we can rewrite the objective function $E(\mathbf{p}, \boldsymbol{\theta})$ as a function of $\mathbf{p}$ alone by eliminating $\boldsymbol{\theta}$, based on a result from \cite{LIAN2023126482}:
\begin{gather}
	E(\mathbf{p}) = -\mathbf{p}^\top \mathbf{A}^\top [\text{mat}(\mathbf{K} \mathbf{B}_2 \mathbf{p}) + \mathbf{C}]^{-1} \mathbf{A}\mathbf{p} + \boldsymbol{\rho}^\top \mathbf{p}. \label{E_p_for_upper_bound}
\end{gather}
Therefore, by substituting the computed $\mathbf{p}$ into this function, we obtain a value for $E$, which serves as an upper bound.

\subsection{Branch-and-Bound}
\label{sec:BnB_aff}
With the preceding preparations in place, we now apply the Branch-and-Bound (BnB) algorithm to optimize the objective function $E$. We begin by forming an initial hypercube $M$ using the initial range of the parameters $\boldsymbol{\theta}$.

In each iteration, the algorithm selects the hypercube with the lowest lower bound for subdivision. This process progressively tightens the global lower bound. Simultaneously, the upper bound is refined by evaluating $E(\mathbf{p})$ using solutions derived during the lower bound computation. For a complete overview of the process, see Algorithm \ref{tri_BnB_algo}.

\begin{algorithm}
	\caption{A BnB algorithm  for minimizing $E$ \label{tri_BnB_algo}}
%	\textbf{Initialization}

	Set tolerance error $\epsilon>0$	and
	initial  hypercube $M$.	
	Let $\mathscr M_1=\mathscr N_1=\{M\}$	
	where $\mathscr M_k$ and $\mathscr N_k$ denote the collection of all hypercubes and the collection of active hypercubes at iteration $k$, respectively.
	
	\For{$k=1,2,\ldots$}{
		
		For each hypercube $M\in \mathscr N_k$,
		minimize the lower bound function $E_l$ 
		%		according to Sec.\ \ref{sec:lb_compute}
		to obtain the optimal point correspondence solution $\mathbf p(M)$ and the optimal value $\beta(M)$.
		$\beta(M)$ is a lower bound for  $M$.
		
		%		Use $\mathbf P(M)$ to compute the value of $E$ according to Sec.\ \ref{sec:comput_ub}
		%		which is  an upper bound for region $M$.

		Let $\mathbf p^k$ be the best among all feasible solutions so far encountered: $\mathbf p^{k-1}$ and all $\mathbf p(M)$ for $M\in\mathscr N_k$.
		Delete all $M\in\mathscr M_k$ such that $\beta(M)\ge E(\mathbf p^k)-\epsilon$.
		Let $\mathscr R_k$ be the remaining collection of hypercubes.
		
		If $\mathscr R_k=\emptyset$,
		terminate: $\mathbf p^k $ is the global $\epsilon-$minimum solution.
		
		Select the hypercube $M$ yielding the lowest lower bound and divide it into two sub-hypercubes $M_1$, $M_2$
		by bisecting  the longest edge.
		
		Let $\mathscr N_{k+1}=\{M_1,M_2\}$
		and $\mathscr M_{k+1}=(\mathscr R_k\backslash M) \cup \mathscr N_{k+1}$.

	}

\end{algorithm}

%\textit{Choice of $\epsilon$}:
%%In the context of registering two point sets, 
%The objective  function \eqref{RPM_obj}  represents the sum of squared distances between model points and their respective scene counterparts.
%If our aim is to ensure that the average distance between inlier model points and their corresponding scene points doesn't surpass $\epsilon_d$,
%then the tolerance error $\epsilon$ ought to be designated as $n_x \epsilon_d^2$.

%\subsection{Positive semidefiniteness  of $\mathbf H^0+\mathbf C$ 
%	\label{subsec:PSD}}
%To confirm the assertion that $\mathbf H^0+\mathbf C\succeq0$, we utilize the separate outliers and inliers test detailed in Sec. \ref{sec:2D_synth_test}, with an outlier to data ratio set at 0.3. We present histograms of the smallest eigenvalues of the matrix $\mathbf H^0+\mathbf C$ across 100 registration instances in Figure \ref{hist_2d}. The result consistently support the claim that $\mathbf H^0+\mathbf C\succeq0$, underscoring the efficacy of the proposed algorithm.
% \ref{sec:appendix} also provides a weak result about  positive semidefiniteness of matrix $\mathbf H^0+\mathbf C$.

\subsection{Positive Semidefiniteness of $\mathbf{H}^0 + \mathbf{C}$}
\label{subsec:PSD}
To verify the positive semidefiniteness of the matrix $\mathbf{H}^0 + \mathbf{C}$ ($\mathbf{H}^0 + \mathbf{C} \succeq 0$), we conducted an experiment using the separate outliers and inliers test, as detailed in Section \ref{sec:2D_synth_test}. The outlier-to-data ratio was set at 0.3.

Figure \ref{hist_2d} displays histograms of the smallest eigenvalues of $\mathbf{H}^0 + \mathbf{C}$ across 100 different registration instances. The results consistently show that the smallest eigenvalue is non-negative, providing strong evidence that $\mathbf{H}^0 + \mathbf{C}$ is indeed positive semidefinite. This outcome highlights the reliability and effectiveness of the proposed algorithm.

%\ref{sec:appendix} also provides a weak result about the positive semidefiniteness of the matrix $\mathbf{H}^0+\mathbf{C}$.

%To verify the claim that 
%$H^0+\mathbf C\succeq0$,
%we employ the separate outliers and inliers test described in Sec. \ref{sec:2D_synth_test} where the outlier to data ratio is chosen as $0.3$.
%The histograms of the smallest eigenvalues of matrix $H^0+\mathbf C$ over 100 registration instances are depicted in Fig. \ref{hist_2d}.
%The result indicate that $H^0+\mathbf C\succeq0$  for all the  registration instances.
%This 
%demonstrates  effectiveness of the proposed algorithm.

\begin{figure}[h]
	\begin{tabular}{@{}c@{} c}
		\includegraphics[width=.5\linewidth]{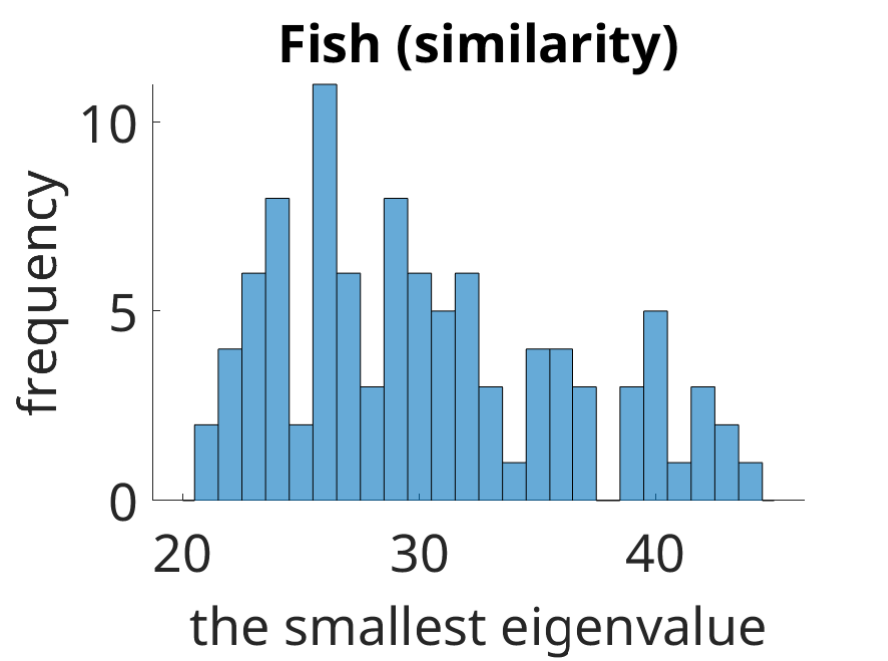}&
		\includegraphics[width=.5\linewidth]{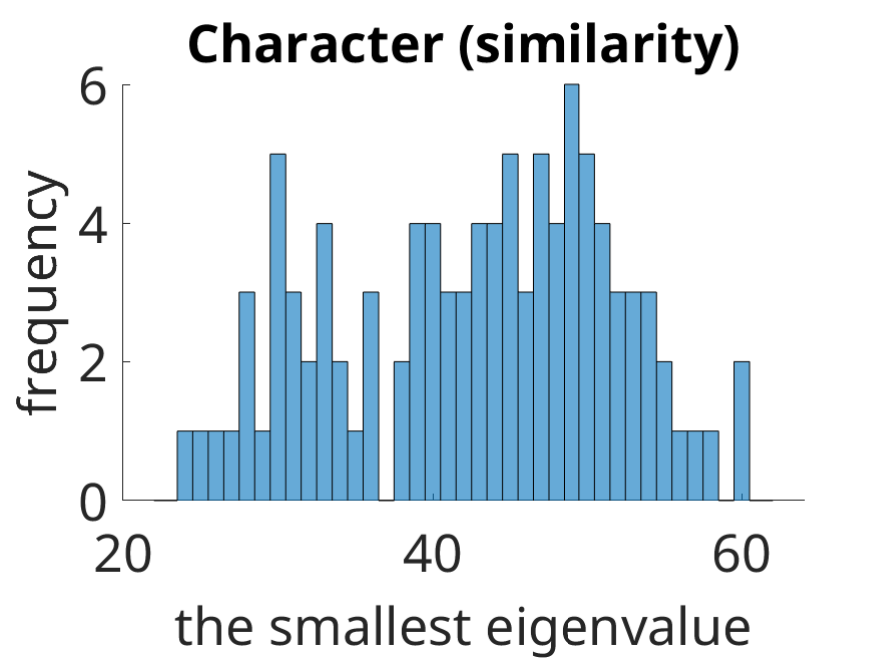}\\
	\end{tabular}
	\caption{Histograms depicting the smallest eigenvalue of $\mathbf H^0+\mathbf C$ across 100 registration instances.
	\label{hist_2d} }		
\end{figure}

\subsection{BnB Algorithm Convergence}
\label{sec:converge_2d}

We assessed the convergence of our BnB algorithm using the separate outliers and inliers test (Sec. \ref{sec:2D_synth_test}) with an outlier-to-data ratio of $0.3$. The results in Figure \ref{LB_UB_2D} reveal two key points:

\begin{enumerate}
	\item A smaller initial range for $\boldsymbol\theta$ yields a tighter duality gap, which improves both upper and lower bounds. This suggests that leveraging a tighter initial range can speed up convergence.
	\item Simpler problem instances (e.g., the fish test vs. the character test) have a tighter duality gap, leading to faster convergence.
\end{enumerate}

Based on observation 2, a fixed duality gap threshold is an unsuitable stopping criterion for all problem types, as it can cause either excessive computation time or premature termination. Therefore, consistent with \cite{LIAN2023126482}, we use the maximum branching depth as our stopping criterion.

\begin{figure} [t]
	\centering
	\newcommand{\scale}{0.5}

	\begin{tabular}{@{}c@{} c@{} c@{} c}	
		\includegraphics[width=\scale\linewidth]{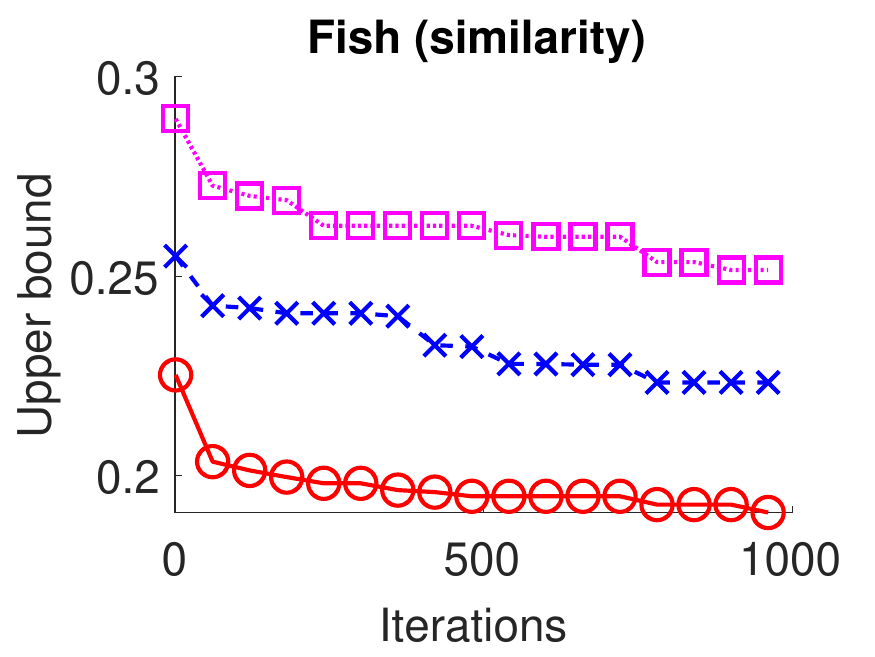}&	
		\includegraphics[width=\scale\linewidth]{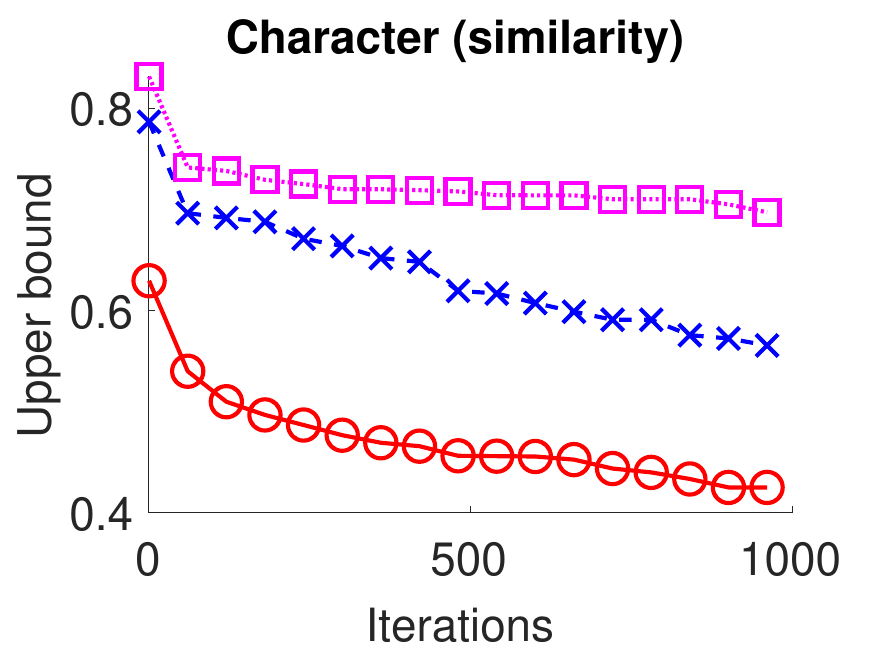}\\		
		\includegraphics[width=\scale\linewidth]{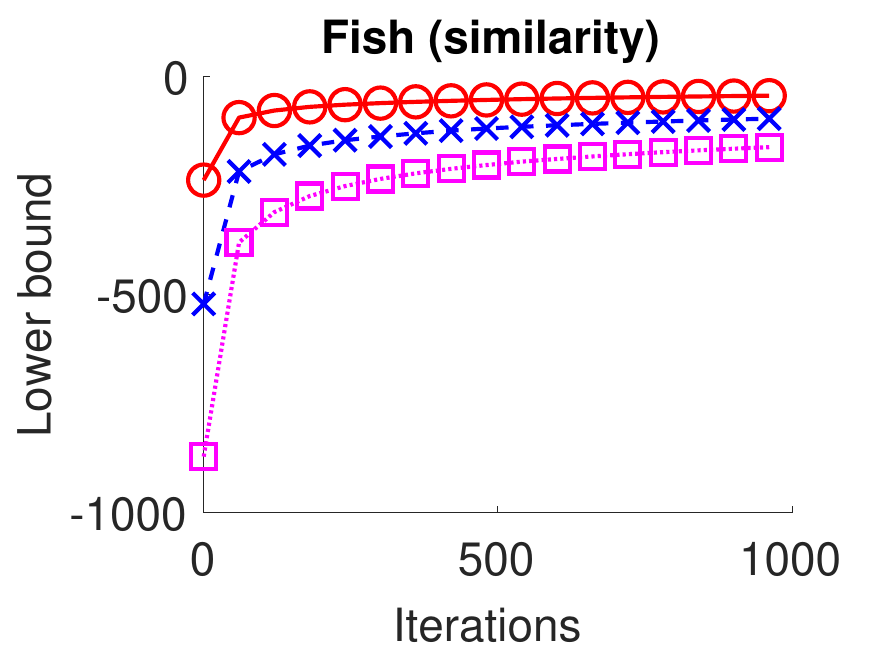}&		
		\includegraphics[width=\scale\linewidth]{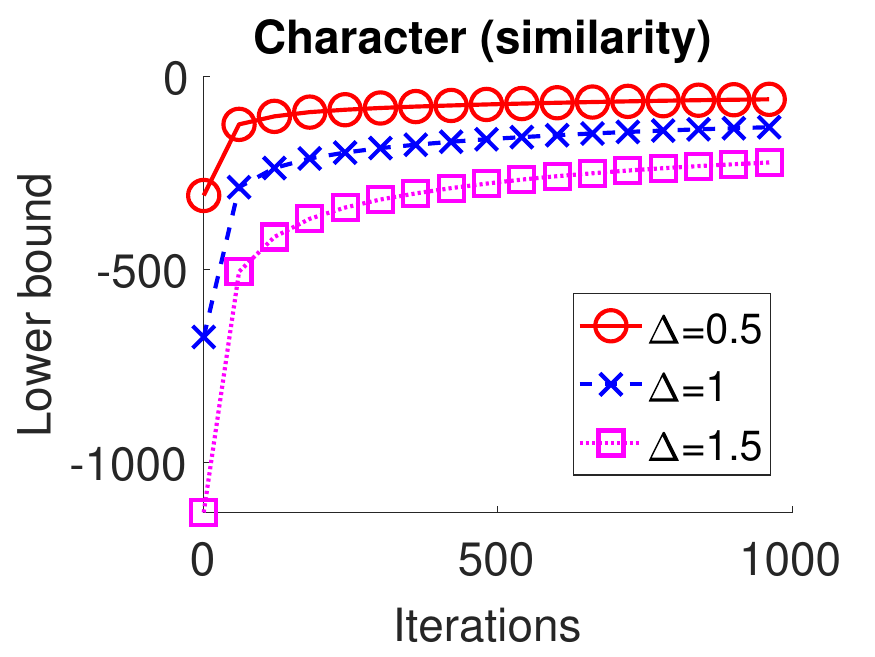}
	\end{tabular}
	\caption{
		Upper (first row) and lower bounds (second row) generated in each iteration of RPM-BP. Our method is tested with the $n_p$ value chosen as the ground truth and with varying initial ranges of $\boldsymbol\theta$: $[\boldsymbol\theta_{gt}-\Delta, \boldsymbol\theta_{gt}+ \Delta]$, where $\boldsymbol\theta_{gt}$ represents the ground truth $\boldsymbol\theta$ solution. Here, the margin $\Delta$ takes values of $0.5$, $1$, and $1.5$ respectively.		
		\label{LB_UB_2D}}
\end{figure}

\section{Application 1: 2D Similarity and Affine Registration \label{sec:app_one}}
Our algorithm is designed to handle various transformations as long as they can be expressed in the form $\mathbf T(\mathbf x|\boldsymbol\theta)=\mathbf J(\mathbf x)\boldsymbol\theta$. Both 2D similarity and affine transformations fit this structure, allowing for a straightforward application of our framework.

For a \textbf{2D similarity transformation}, the transformation matrix and parameter vector are:
\begin{gather}
	\mathbf T(\mathbf x|{\boldsymbol\theta})=\begin{bmatrix}
		\theta_1 & -\theta_2 \\ \theta_2 & \theta_1
	\end{bmatrix}
	\begin{bmatrix}
		x^1\\x^2 
	\end{bmatrix}
	+\begin{bmatrix}
		\theta_3 \\ \theta_4
	\end{bmatrix}
	=\begin{bmatrix}x^1&-x^2&1&0\\x^2&x^1&0&1\end{bmatrix}\boldsymbol\theta 
\end{gather}
Here, $[\theta_3, \theta_4]^\top$ represents the translation, while $\theta_1$ and $\theta_2$ are derived from the scale $s$ and rotation angle $\phi$ as $\theta_1=s\cos(\phi)$ and $\theta_2=s\sin(\phi)$. For this transformation, it can be derived that $\mathbf B_2$ corresponds to rows $[1, 3, 4]$ of $\mathbf B$, and $\mathbf C=\text{diag}([0,0,n_p,n_p])$.

A \textbf{2D affine transformation} is a more general case:
\begin{gather}
	\mathbf T(\mathbf x|{\boldsymbol\theta})=\begin{bmatrix}
		\theta_1 & \theta_2 \\ \theta_3 & \theta_4
	\end{bmatrix}
	\begin{bmatrix}
		x^1\\x^2 
	\end{bmatrix}
	+\begin{bmatrix}
		\theta_5 \\ \theta_6
	\end{bmatrix}
	=\begin{bmatrix}
		x^1&x^2&0&0&1&0\\
		0&0&x^1&x^2&0&1
	\end{bmatrix} \boldsymbol\theta 
\end{gather}
For this transformation, $\mathbf B_2$ corresponds to rows $[1, 2, 5, 8, 11]$ of $\mathbf B$, and $\mathbf C=\text{diag}([0,0,0,0,n_p,n_p])$.

\section{Application 2: 3D Rigid Registration \label{sec:app_two}}
Our method is theoretically capable of handling 3D affine transformations, which has the form
\begin{gather}
	\setcounter{MaxMatrixCols}{20}
	T(\mathbf x|{\boldsymbol\theta})=\begin{bmatrix}
		\theta_1 & \theta_2 & \theta_3 \\
		\theta_4 & \theta_5 &\theta_6 \\
		\theta_7 & \theta_8 &\theta_9
	\end{bmatrix}
	\begin{bmatrix}
		x^1\\x^2 \\x^3
	\end{bmatrix}
	+\begin{bmatrix}
		\theta_{10} \\ \theta_{11} \\ \theta_{12}
	\end{bmatrix}   %\notag\\
	=\begin{bmatrix}
		%		x^1&x^2&x^3&  0&0&0&  0&0&0& 1
		x^1&x^2&x^3&  0&0&0&  0&0&0& 1&0&0\\
		0&0&0&x^1&x^2&x^3&0&0&0&0&1&0\\
		0&0&0&0&0&0&x^1&x^2&x^3&0&0&1
	\end{bmatrix} \boldsymbol\theta  \notag
\end{gather}
It can be derived that  $\mathbf B_2$ corresponds to rows $[1,2,3,10,14,15,22,27,34]$ of $\mathbf B$, and
$\mathbf C=\text{diag}([\underbrace{0,\ldots,0}_{9\text{ zeros}},n_p,n_p,n_p])$.
Nevertheless,
 the large number of parameters (12 for a full affine transformation) results in a high-dimensional branching space, significantly slowing down convergence. Therefore, for 3D point cloud registration, we focus on the more constrained \textbf{rigid transformation}.

A 3D rigid transformation involves a rotation and a translation. We represent the 3D rotation using the \textbf{angle-axis representation}, where a rotation is defined by a 3D vector $\mathbf r$. The corresponding $3\times 3$ rotation matrix $\mathbf R\in \mathbb{SO}_3$ can be obtained via the matrix exponential map \cite{Go-ICP_pami}:
\begin{gather}
	\mathbf R=\mathbf I_3+ \frac{[\mathbf r]_\times \sin\|\mathbf r\|}{\|\mathbf r\|}
	+\frac{[\mathbf r]_\times^2(1-\cos \|\mathbf r\|)}{\|\mathbf r\|^2} \label{rot_matrix}
\end{gather}
where $[\cdot]_\times$ is the skew-symmetric matrix representation:
\begin{gather}
	[\mathbf r]_\times= \begin{bmatrix}
		0 & -r_3& r_2\\
		r_3&0& -r_1 \\
		-r_2 & r_1 &0
	\end{bmatrix}
\end{gather}
with $r_i$ being the $i$-th element of $\mathbf r$.

We adapt our core algorithm from Section \ref{sec:optimize} to handle this non-linear rigid transformation. The key modifications are:
\begin{enumerate}
	\item Instead of branching over the full affine parameters, we branch over the 6-dimensional rigid parameters: the angle-axis vector $\mathbf r$ and the translation vector $\mathbf t$.
	\item In each iteration of the BnB algorithm, the range of the rotation matrix $\mathbf R$ is computed from the current range of the angle-axis vector $\mathbf r$ using the formula in Eq. \eqref{rot_matrix}.
\end{enumerate}

For the upper bound calculation, we continue to use the affine-based approach from Eq. \eqref{E_p_for_upper_bound}. While a dedicated method for rigid transformations exists (see Sec. 5.3 of \cite{LIAN2023126482}), it often yields higher upper bounds, which hinders the pruning process and slows down our algorithm's convergence. Therefore, we prioritize speed by using the affine upper bound, even though it doesn't strictly enforce the rigid constraint.

\subsection{Efficiently Computing the Range of $\mathbf R$ \label{sec:range_R_from_r}}
Given a range $[\underline{\mathbf r},\overline{\mathbf r}]$ for the angle-axis vector $\mathbf r$, finding the exact range of each element of the rotation matrix $\mathbf R$ is a computationally expensive non-linear optimization problem. Performing this calculation at every step of the BnB algorithm would be a significant bottleneck.

To address this, we use a \textbf{precomputation and grid-based approximation} method. Before the BnB algorithm begins, we create a regular grid over the initial search space of $\mathbf r$. For each grid point, we compute the corresponding rotation matrix $\mathbf R$ and store the values. Then, during the BnB process, when we need to find the range of $\mathbf R$ for a given hypercube $[\underline{\mathbf r},\overline{\mathbf r}]$, we simply find all precomputed grid points that fall within this hypercube and take the minimum and maximum of their corresponding $\mathbf R$ values. This provides a fast, though approximate, way to get the necessary bounds, significantly speeding up the overall algorithm.

\section{Experiments \label{sec:exp}}

We implemented the proposed algorithm, which we refer to as RPM-BP, using Matlab R2023b. All comparative experiments were conducted on a computer equipped with 6-core CPUs running at 3.2 GHz.

For competing methods that output only point correspondences, we used these correspondences to compute the optimal affine transformation between the two point sets. We defined the matching error as the root-mean-square distance between the transformed model inliers and their corresponding scene inliers.

To ensure efficient computation, the maximum branching depth of RPM-BP was set to 12.

\subsection{2D Registration}
We compared our method with RPM-HTB \cite{LIAN2023126482}, RPM-PA \cite{lian2021polyhedral}, and RPM-CAV \cite{RPM_model_occlude_PR}. These methods were chosen because they are all based on global optimization techniques, can handle partial overlap, and allow for arbitrary similarity transformations, making them suitable for direct comparison with our approach.

\subsubsection{2D Synthetic Data \label{sec:2D_synth_test}}
Synthetic data allows us to systematically evaluate an algorithm's resilience to specific disturbances. We performed five distinct types of tests:
\begin{enumerate}
	\item \textbf{Deformation test}: The prototype shape is non-rigidly deformed to create the scene point set.
	\item \textbf{Noise test}: Positional noise is added to the prototype shape to generate the scene point set.
	\item \textbf{Mixed outliers test}: Random outliers are superimposed on both the prototype and scene shapes.
	\item \textbf{Separate outliers test}: Random outliers are added to different areas of the prototype shape to generate the two point sets.
	\item \textbf{Occlusion + Outlier test}: The prototype shape is occluded to produce two point sets, to which random outliers (with a fixed outlier-to-data ratio of 0.5) are then added in different regions.
\end{enumerate}
These tests are visually illustrated in Fig. \ref{rot_2D_test_data_exa}. To assess a method's ability to handle arbitrary rotations and scaling, random rotations and scalings within the range of $[0.5, 1.5]$ were applied when generating the point sets. Examples of registration results from different methods are provided in Fig. \ref{rot_2D_syn_match_exa}.

The registration errors produced by each method are shown in the top two rows of Fig. \ref{2D_simi_sta}. Our results indicate that RPM-BP and RPM-HTB, when the number of correspondences $n_p$ is chosen as the ground truth, are robust against deformation, noise, and separate outliers. However, their performance degrades when outliers are mixed with inliers. This limitation is a consequence of our method not being $\epsilon$-globally optimal, as discussed in Sec. \ref{sec:converge_2d}.

When using different transformation types, RPM-BP demonstrates greater resilience to disturbances when employing a similarity transformation compared to an affine transformation. Furthermore, the performance of RPM-BP is significantly influenced by the chosen $n_p$ value, with results improving as $n_p$ approaches the ground truth. We also found that RPM-BP is less sensitive to variations in $n_p$ than RPM-HTB, especially in outlier and occlusion tests.

The average run times are depicted in the bottom row of Fig. \ref{2D_simi_sta}. RPM-BP is the most efficient method in the deformation, positional noise, and mixed inlier-outlier tests. Comparing transformation types, RPM-BP is more efficient with similarity transformations due to the smaller number of parameters.

Finally, the outlier tests also provide insight into how each method scales with problem size. RPM-BP, RPM-HTB, and RPM-CAV demonstrate the best scalability, followed by RPM-PA.

\begin{figure*} [t]
		\setlength{\abovecaptionskip}{-1pt plus 0pt minus 12pt} % Chosen fairly arbitrarily	
	%	\setlength\arrayrulewidth{1pt}
%\begin{minipage}{1\textwidth}	
	\centering
	\newcommand{\scale}{0.101	}

	\begin{tabular}{@{\hspace{-0mm}}c@{}|@{}c@{}|@{}c@{}|@{}c@{}|@{}c@{}|@{}c@{}|@{}c@{}|@{}c|@{}c }	
		\includegraphics[width=\scale\linewidth]{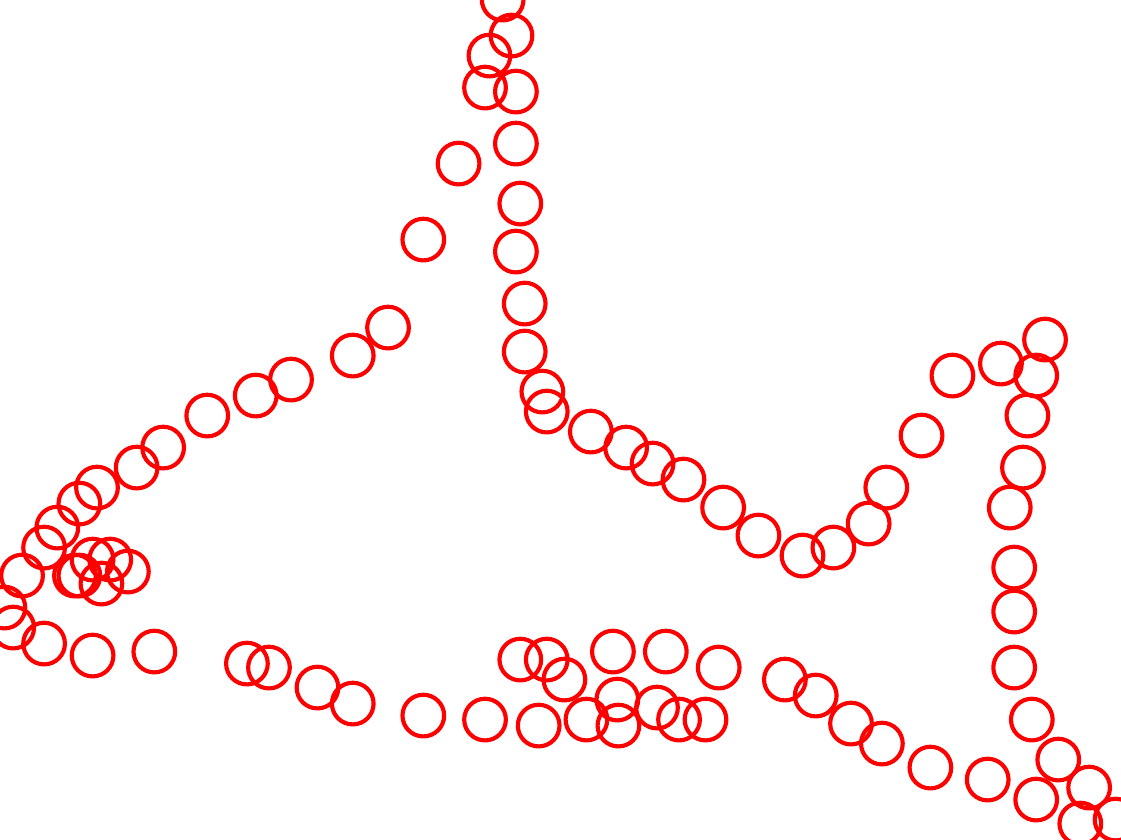}&		
		\includegraphics[width=\scale\linewidth]{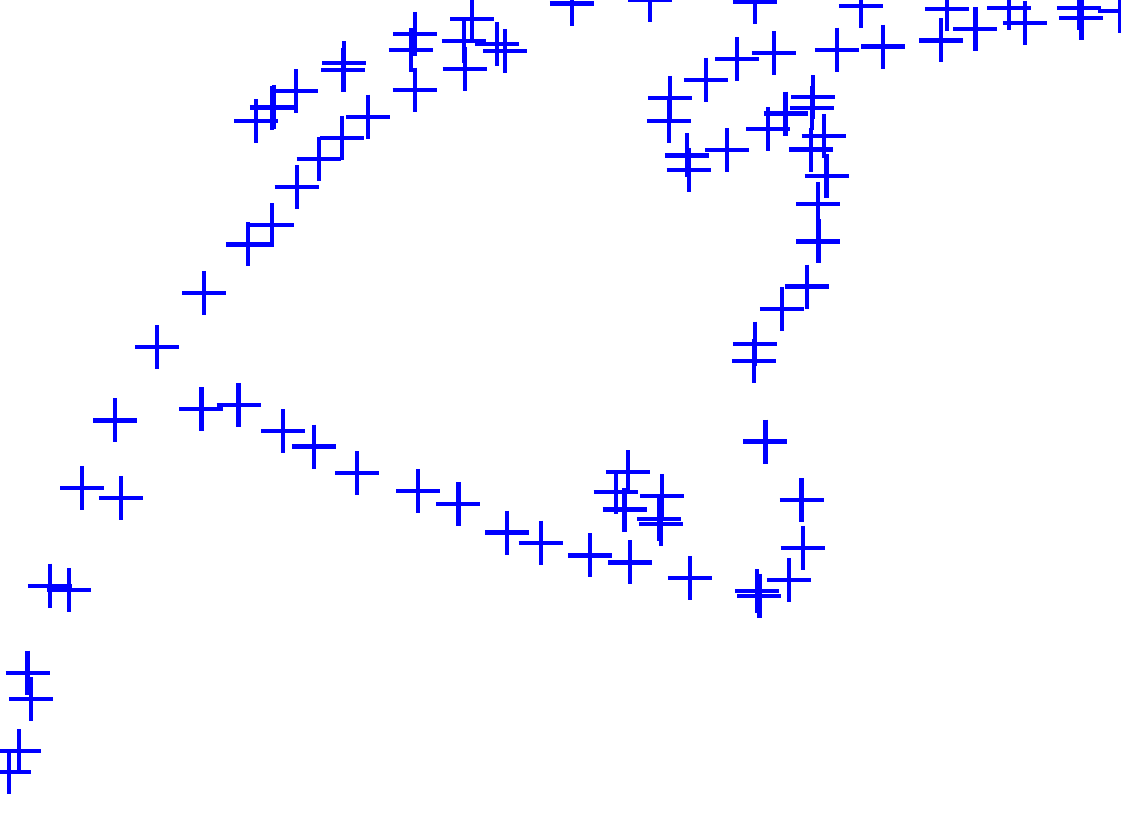}&
		\includegraphics[width=\scale\linewidth]{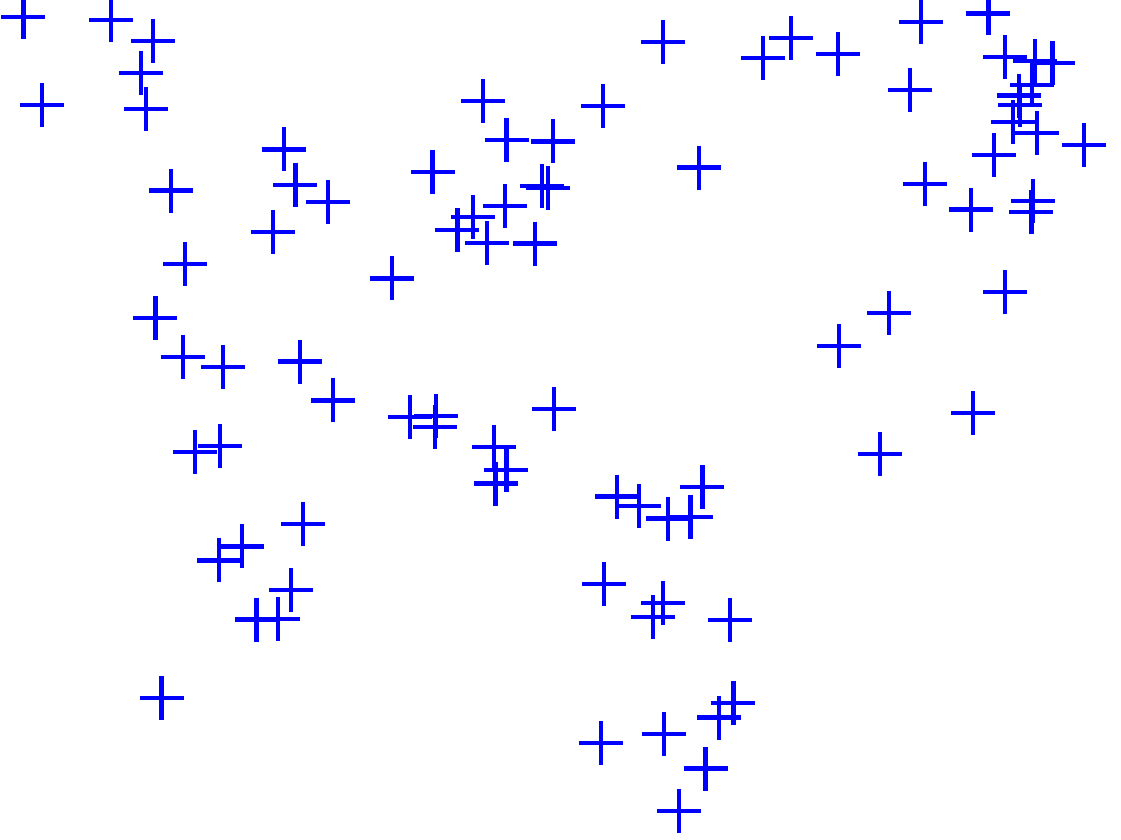}&		
		\includegraphics[width=\scale\linewidth]{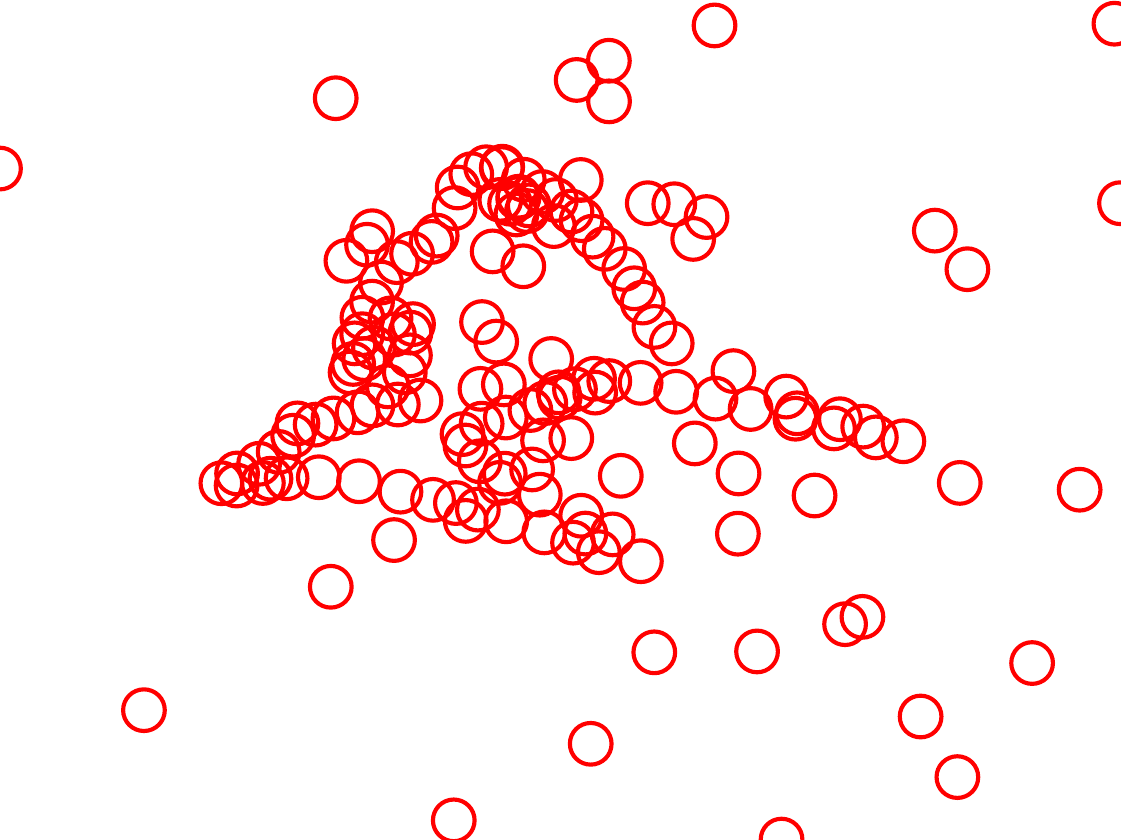}&
		\includegraphics[width=\scale\linewidth]{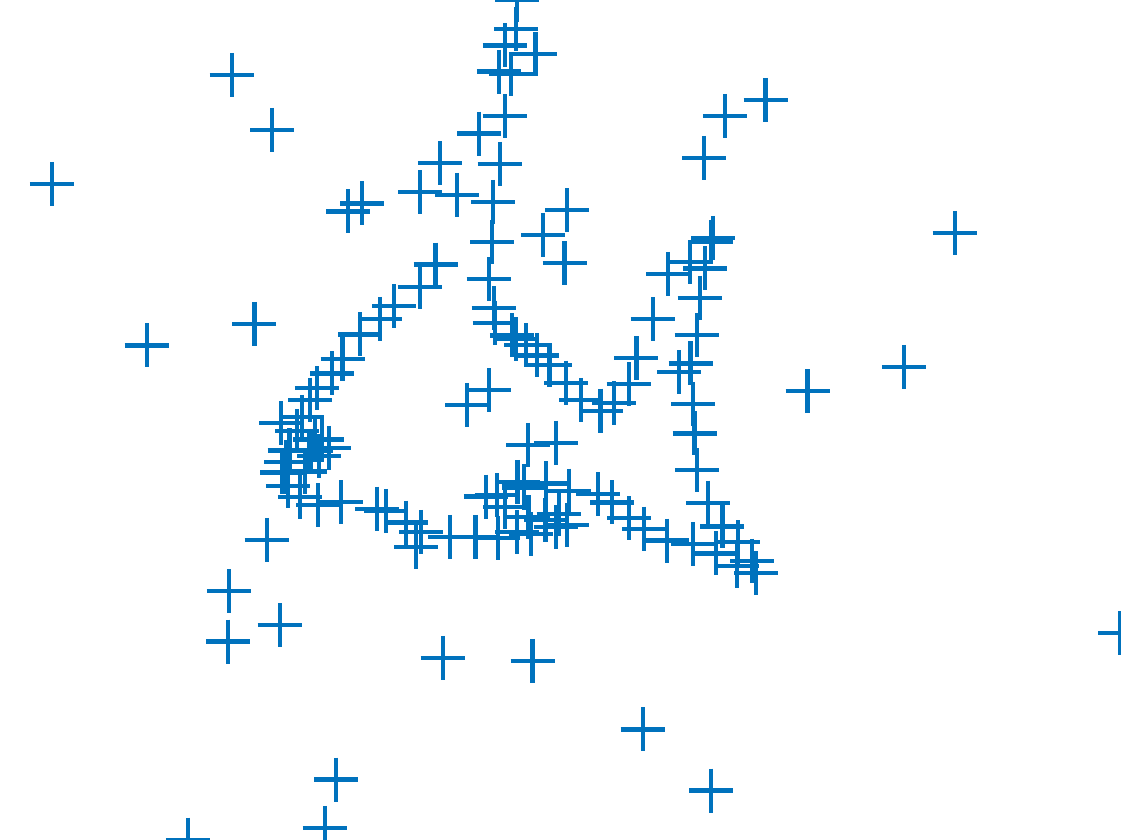}&
		\includegraphics[width=\scale\linewidth]{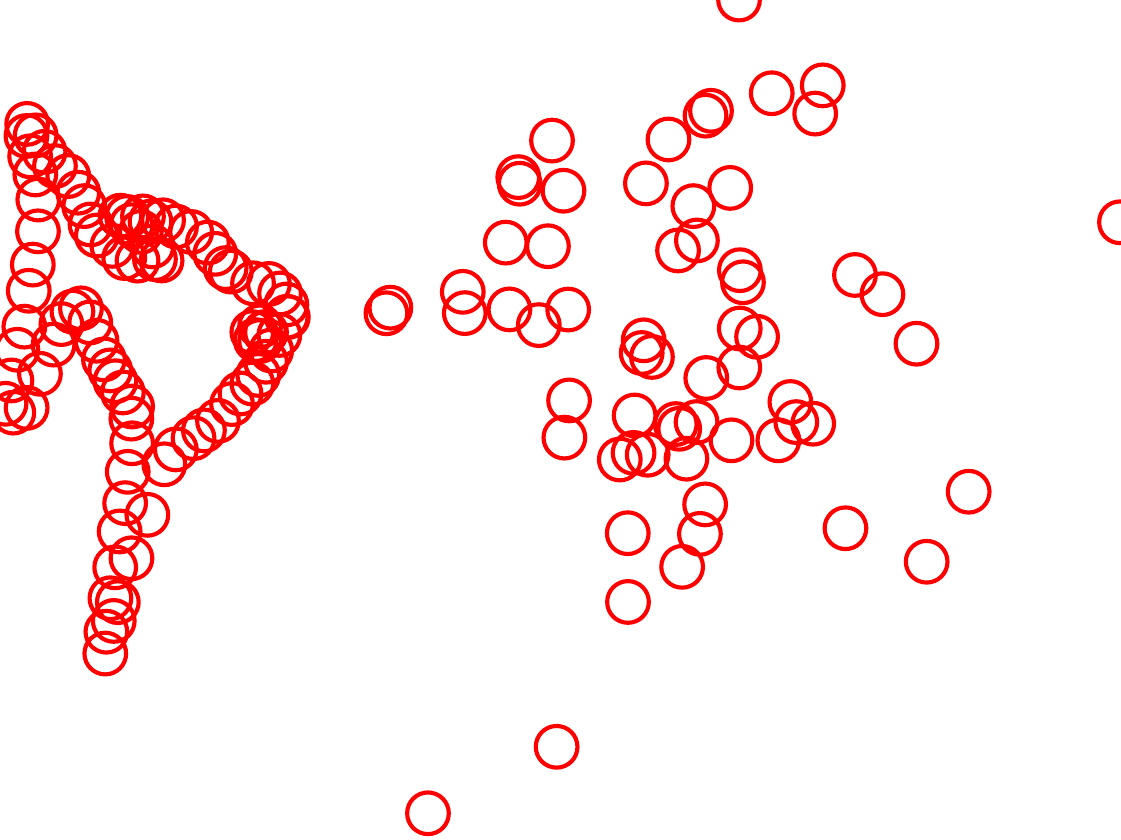}&
		\includegraphics[width=\scale\linewidth]{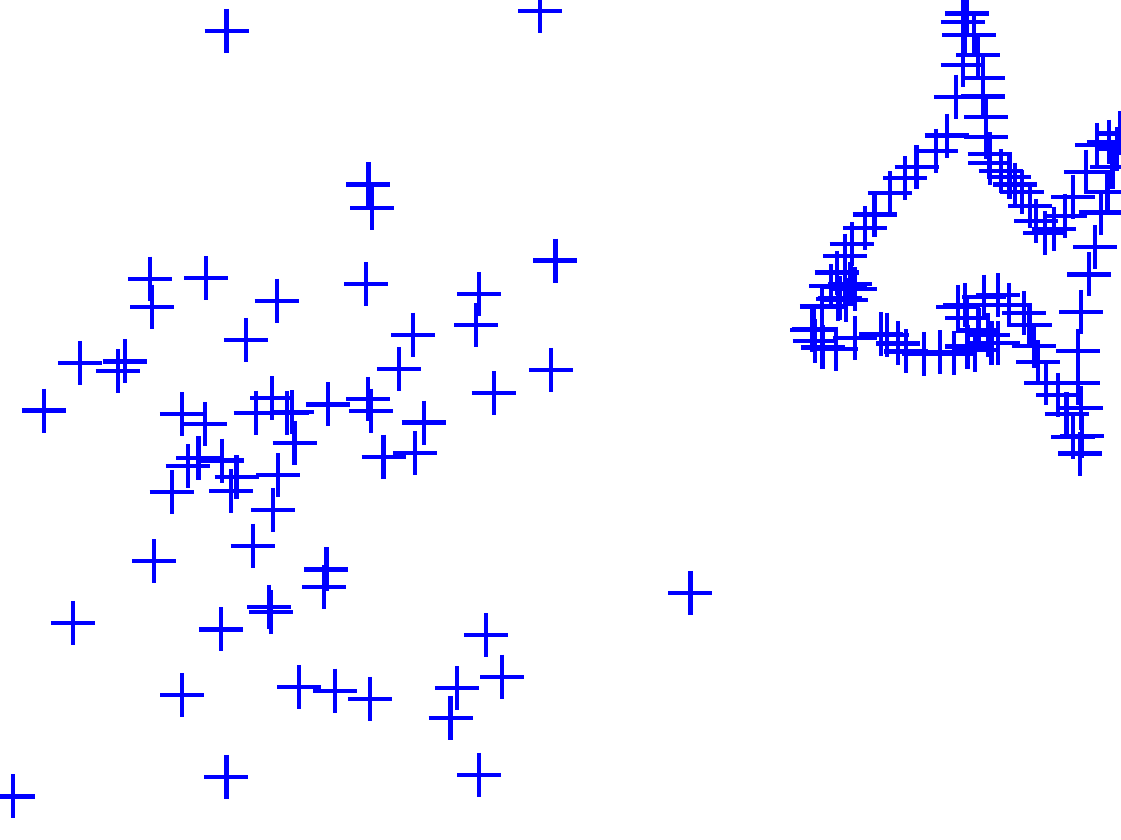}&
		\includegraphics[width=\scale\linewidth]{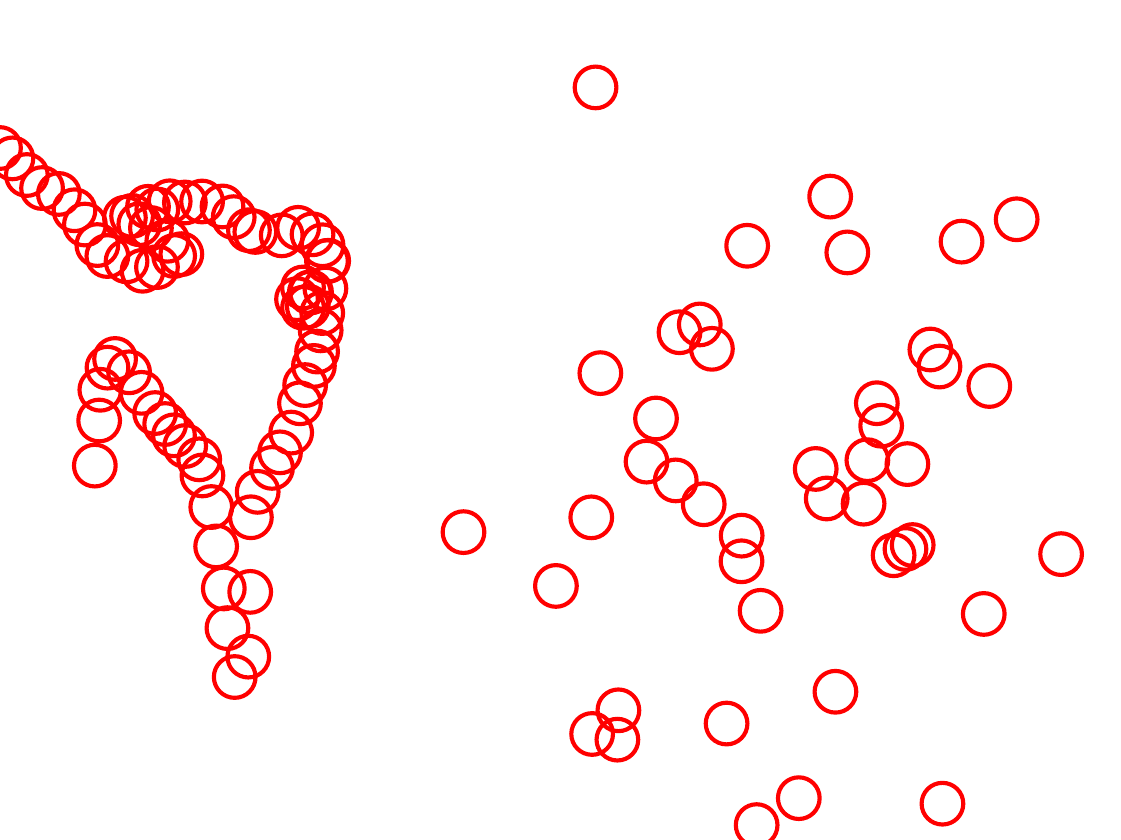}&
		\includegraphics[width=\scale\linewidth]{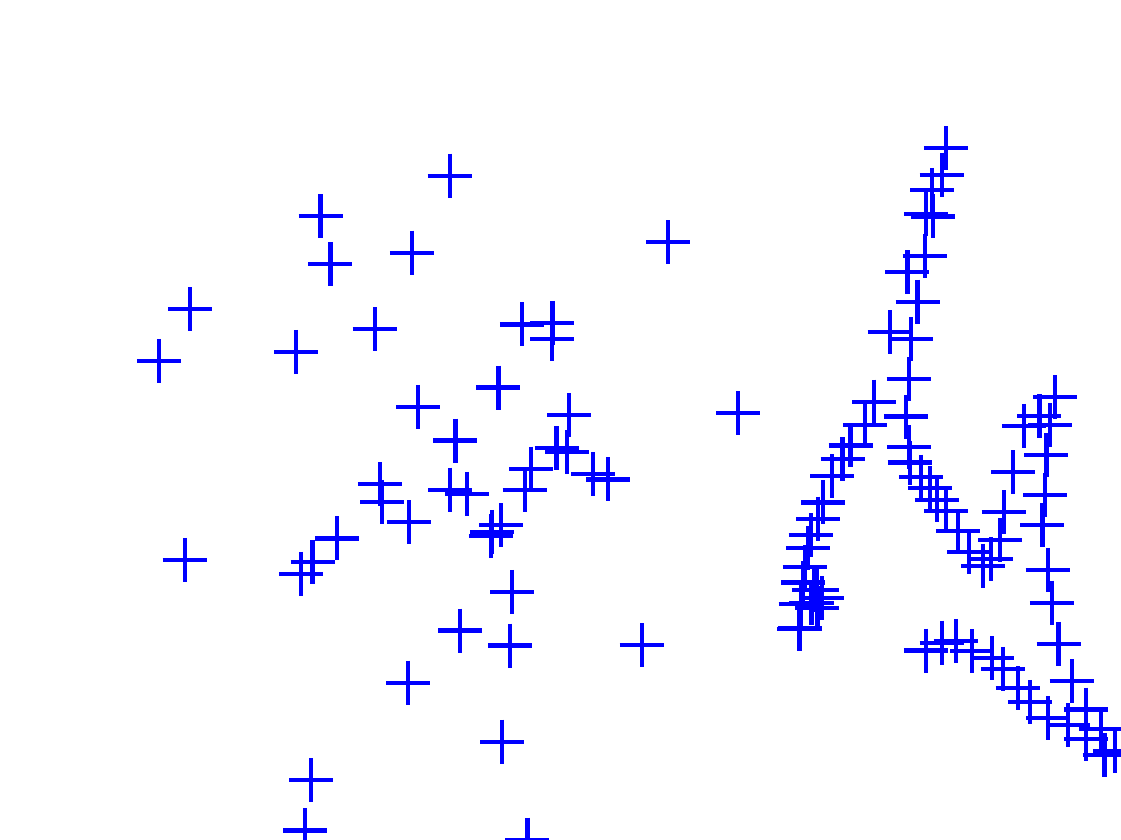} 
		\\	\hline
		%				\vspace{1mm}	 
		\subfigure[ ]{		\includegraphics[width=\scale\linewidth]{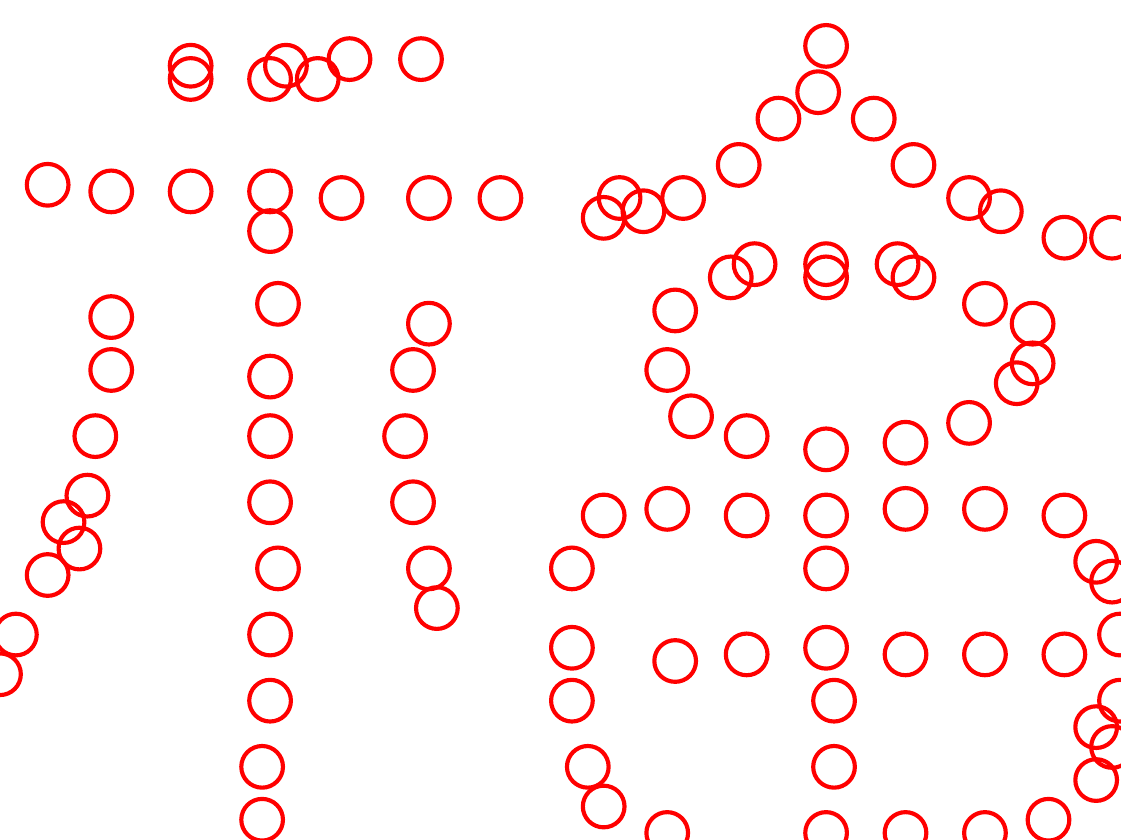}} &
		\subfigure[]{		
			\includegraphics[width=\scale\linewidth]{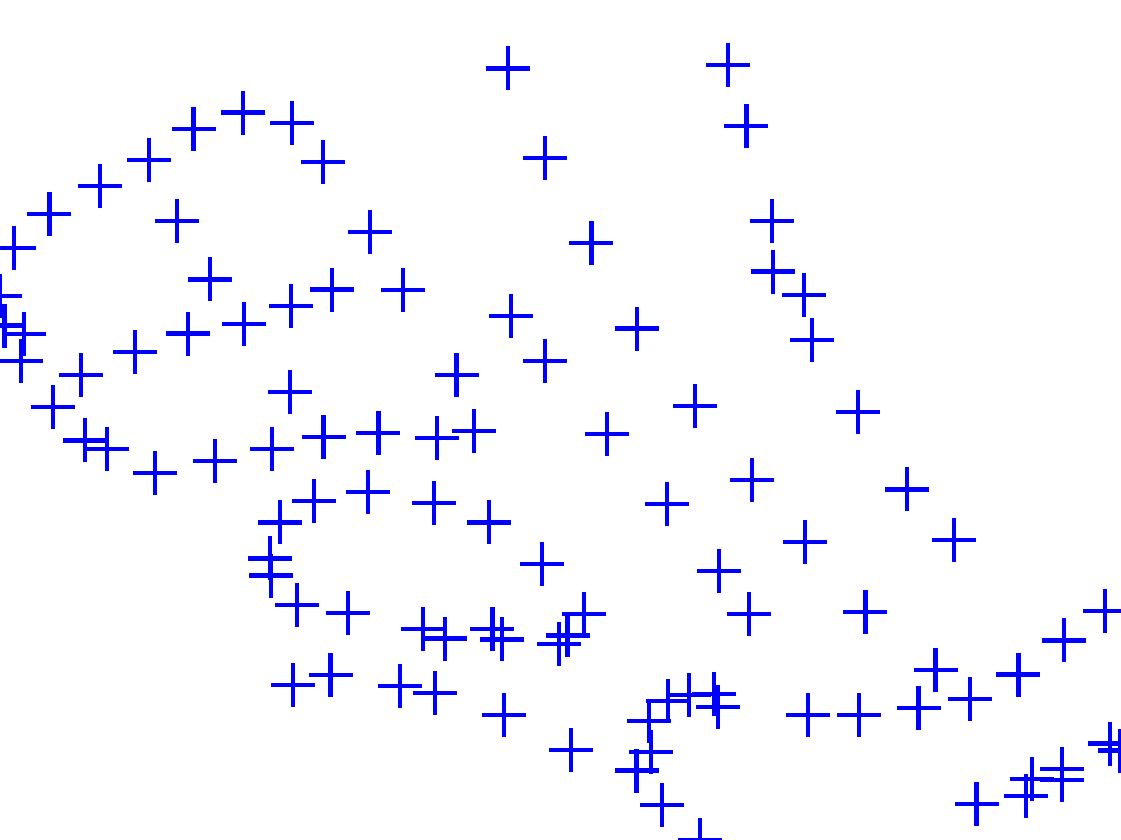}}&
		\subfigure[]{
			\includegraphics[width=\scale\linewidth]{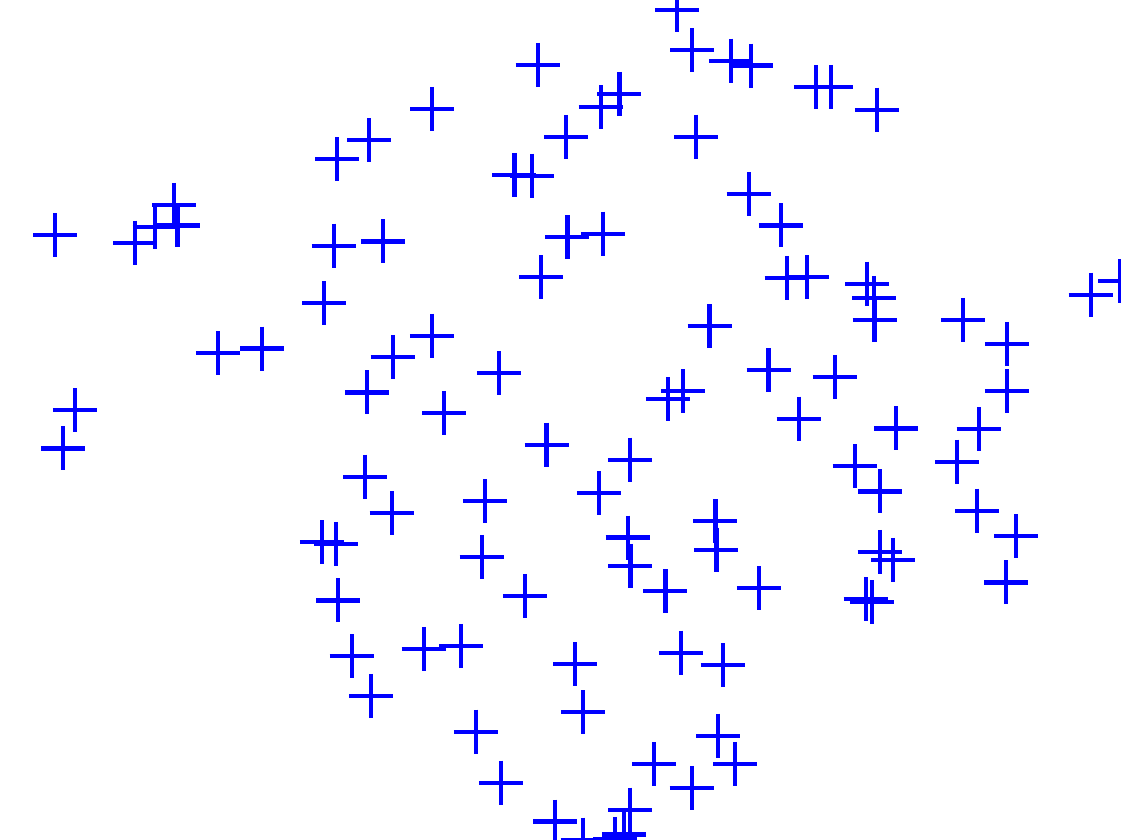}}&	\subfigure[]{		
			\includegraphics[width=\scale\linewidth]{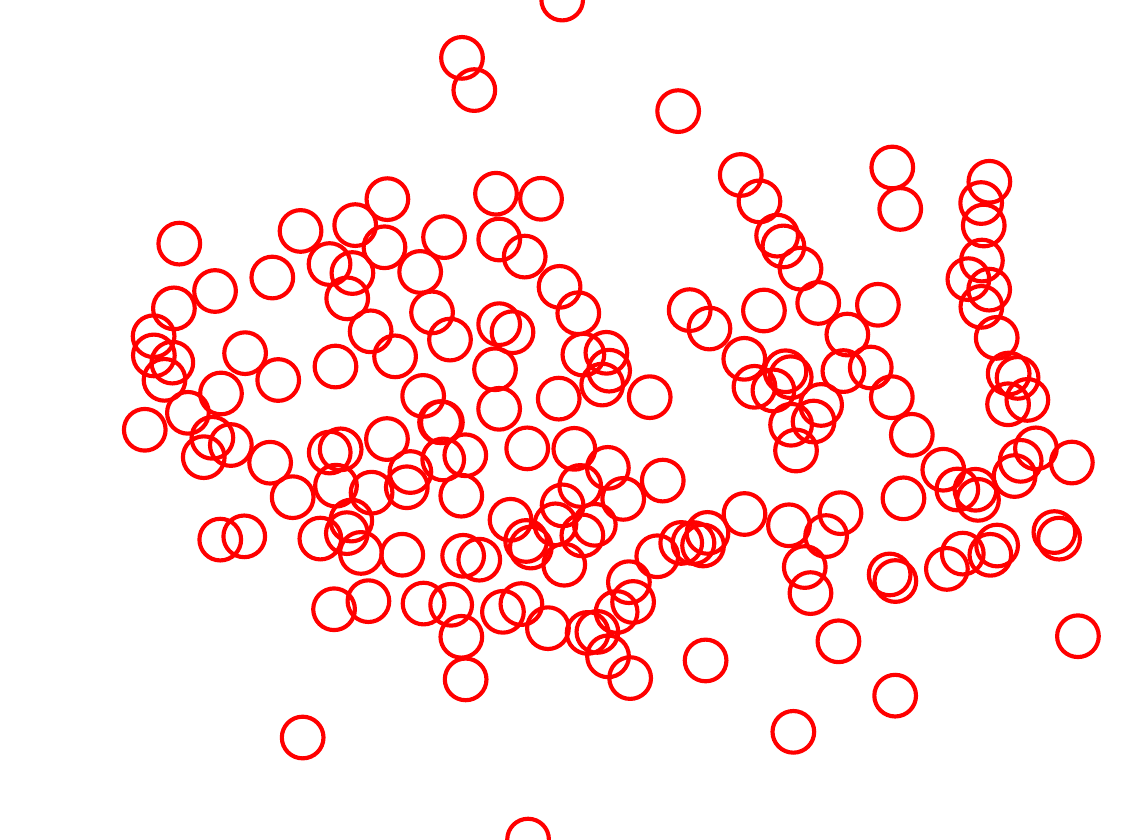}} &
		\subfigure[]{	\includegraphics[width=\scale\linewidth]{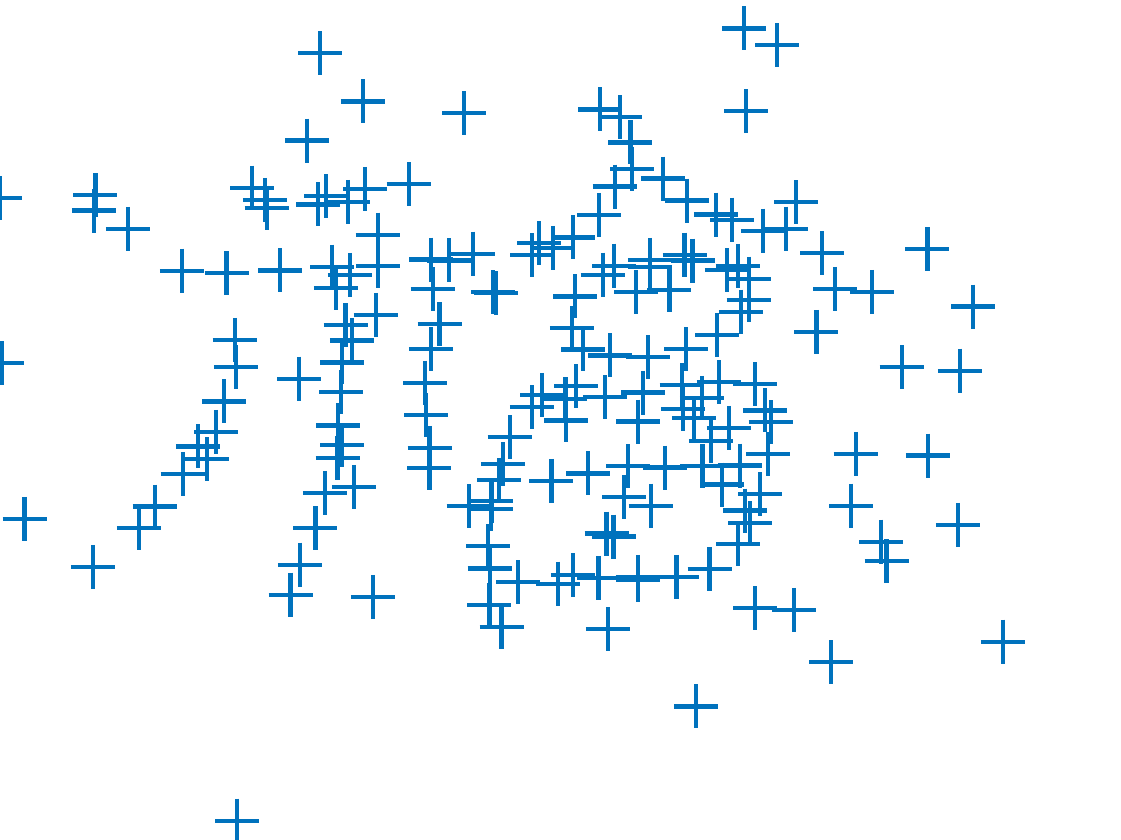}}&	  
		\subfigure[]{ \includegraphics[width=\scale\linewidth]{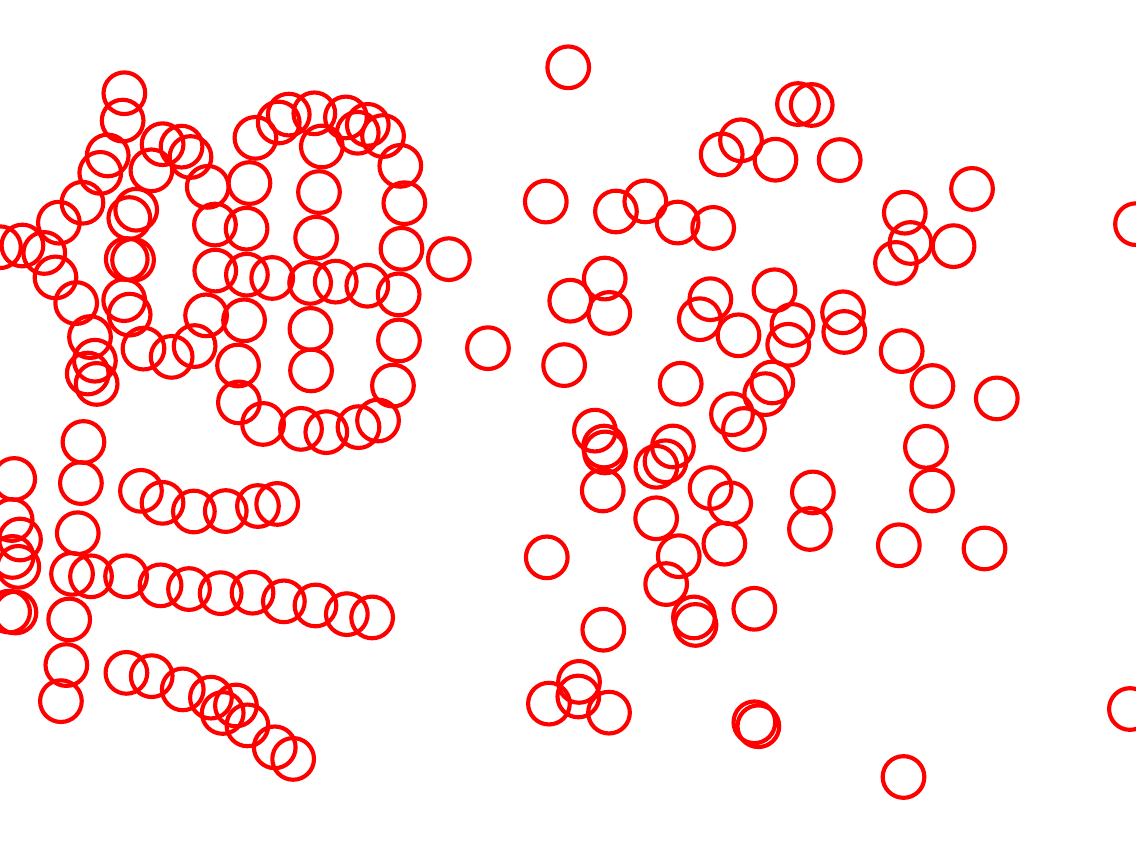}} &
		\subfigure[]{ \includegraphics[width=\scale\linewidth]{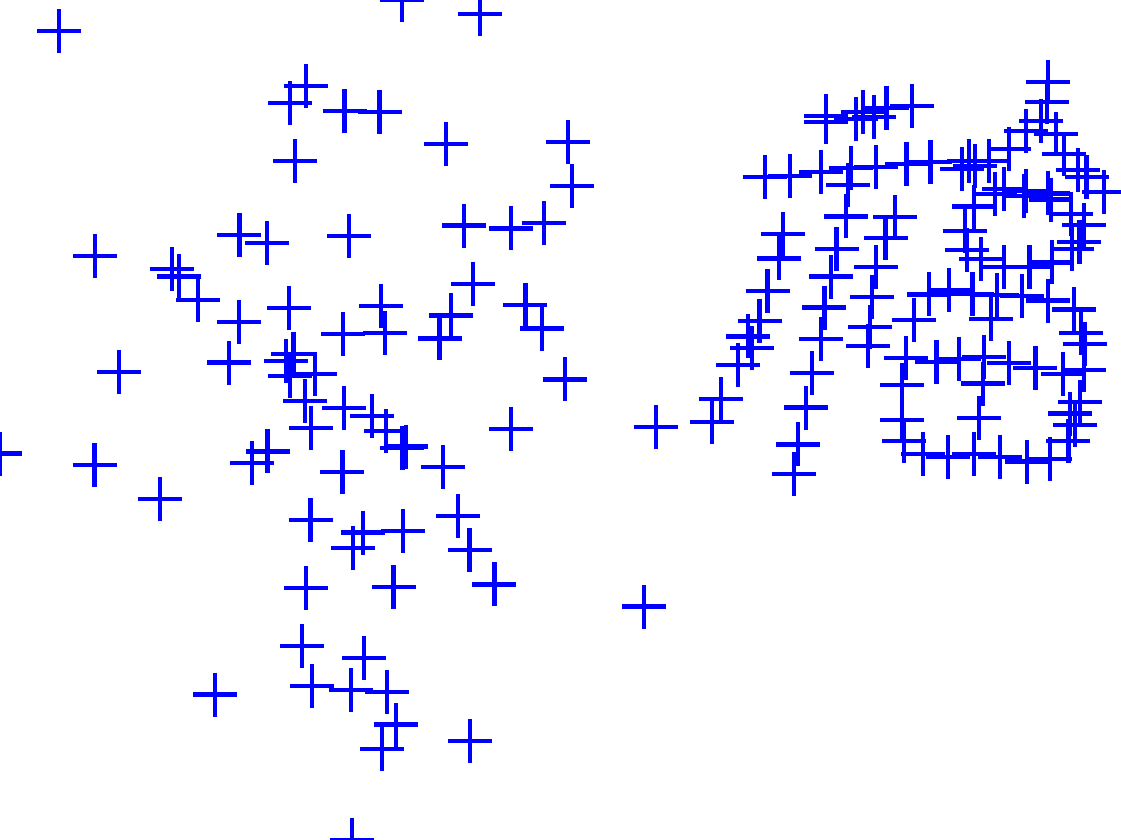}}&
		\subfigure[]{ \includegraphics[width=\scale\linewidth]{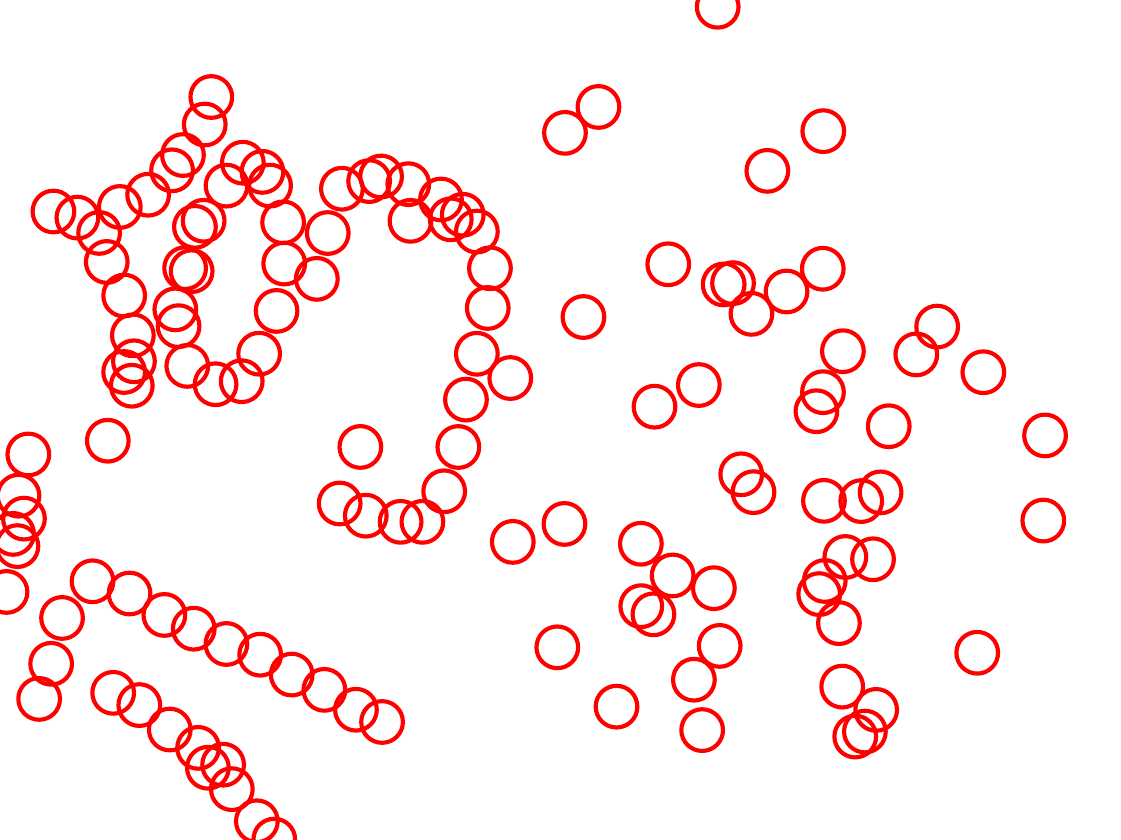}}&
		\subfigure[]		{ \includegraphics[width=\scale\linewidth]{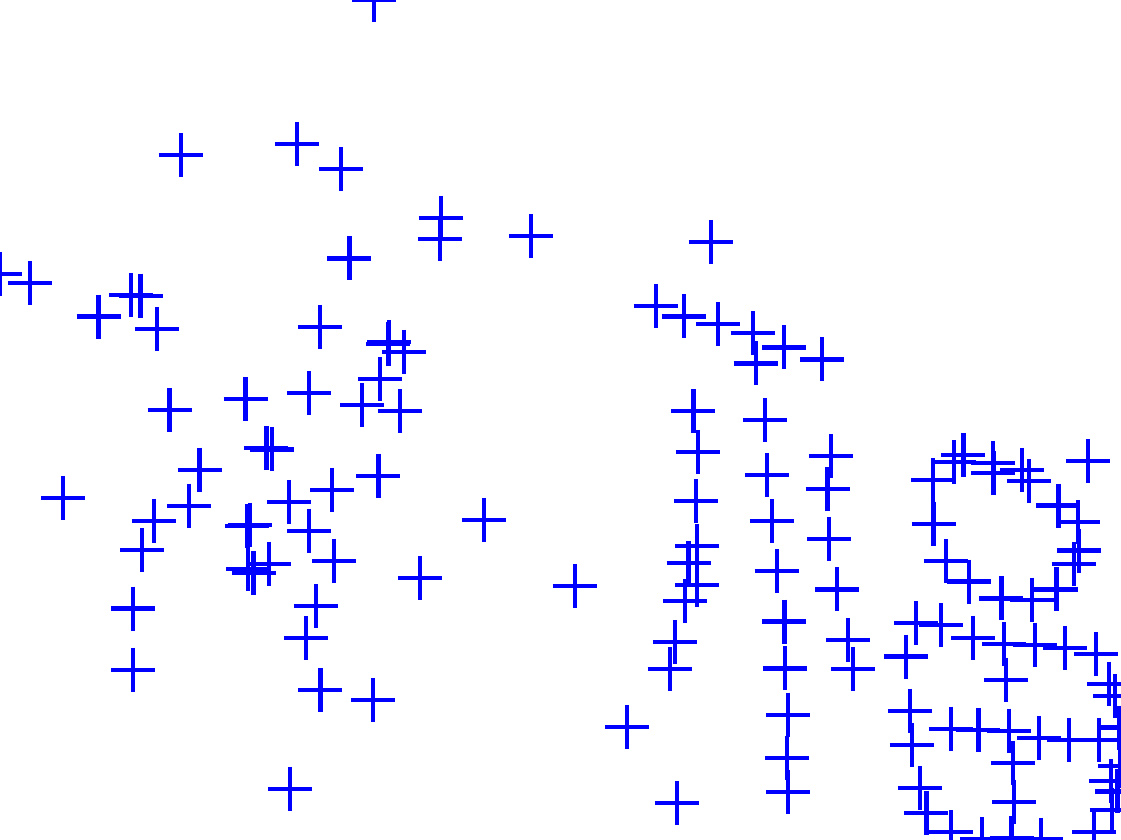}}  
	\end{tabular}
	\caption{
a) to (c): Model point sets and examples of scene point sets in the deformation and noise tests, respectively.
		(d) to (i): Examples of model and scene point sets in the mixed outliers and inliers test ((d), (e)), separate outliers and inliers test ((f), (g)), and occlusion+outlier test ((h), (i)), respectively.
		In all cases, model points are indicated by red circles, while scene points are represented by blue crosses.		
		\label{rot_2D_test_data_exa}}
%\end{minipage}
%	\end{figure*} 
%\begin{figure*} [th]
%	\begin{minipage}{\textwidth}
	\centering
	\newcommand\scaleGd{0.21}
\begin{tabular}{@{\hspace{-0mm}}c@{\hspace{-1.6mm}}c@{\hspace{-1.6mm}}c@{\hspace{-1.6mm}}c@{\hspace{-1.6mm}} c }
		\includegraphics[width=\scaleGd\linewidth]{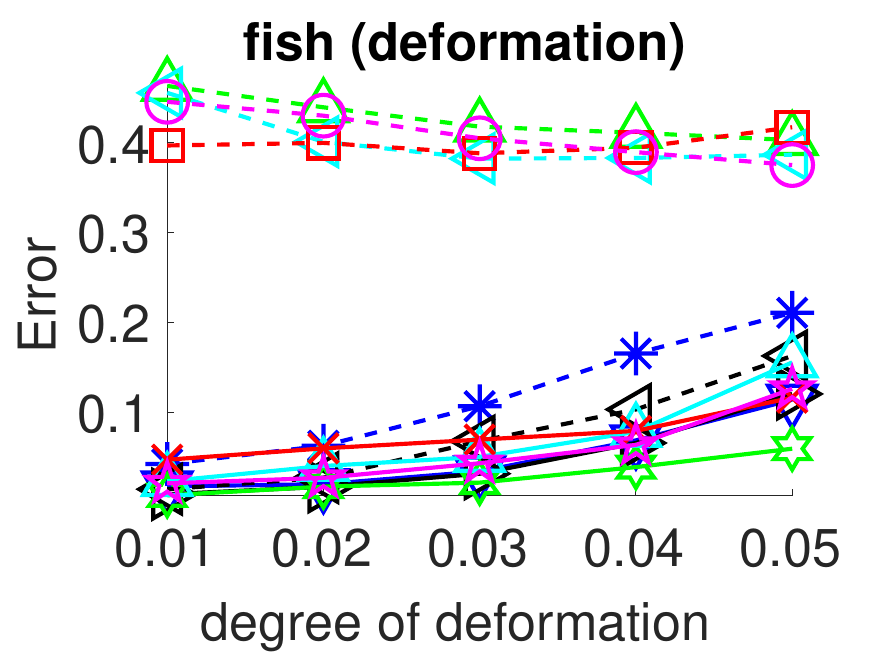}&
		\includegraphics[width=\scaleGd\linewidth]{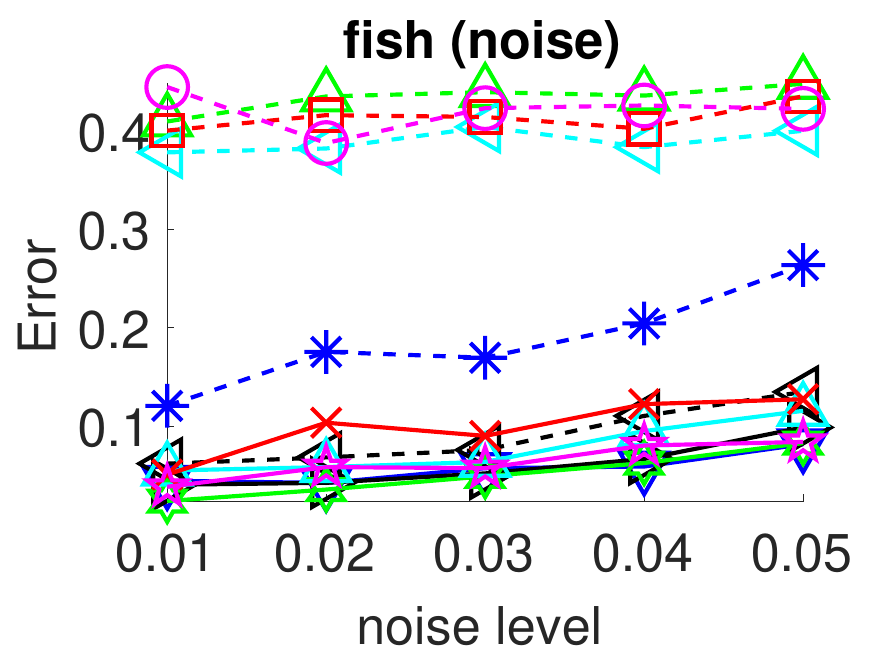}&
		\includegraphics[width=\scaleGd\linewidth]{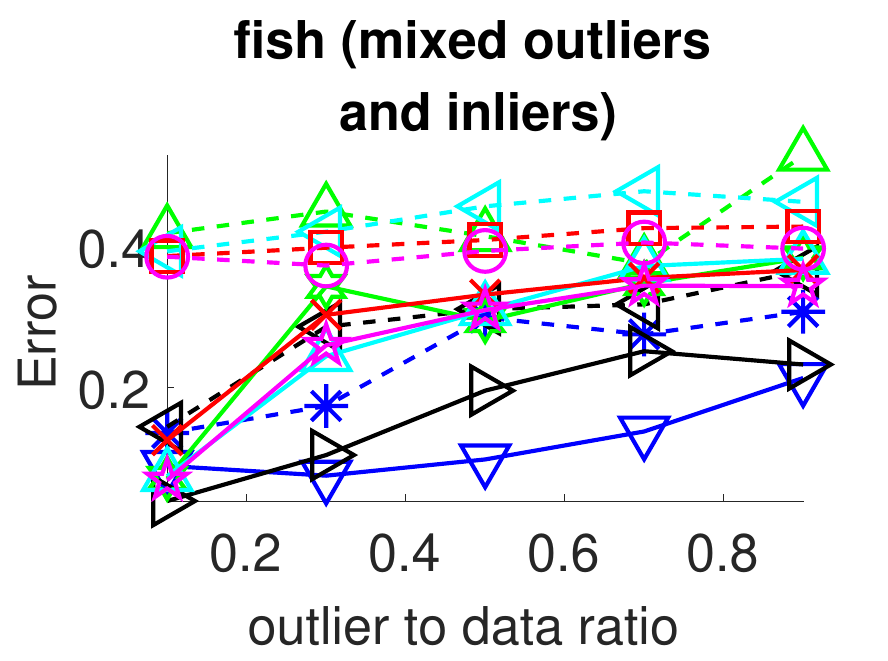}&			
		\includegraphics[width=\scaleGd\linewidth]{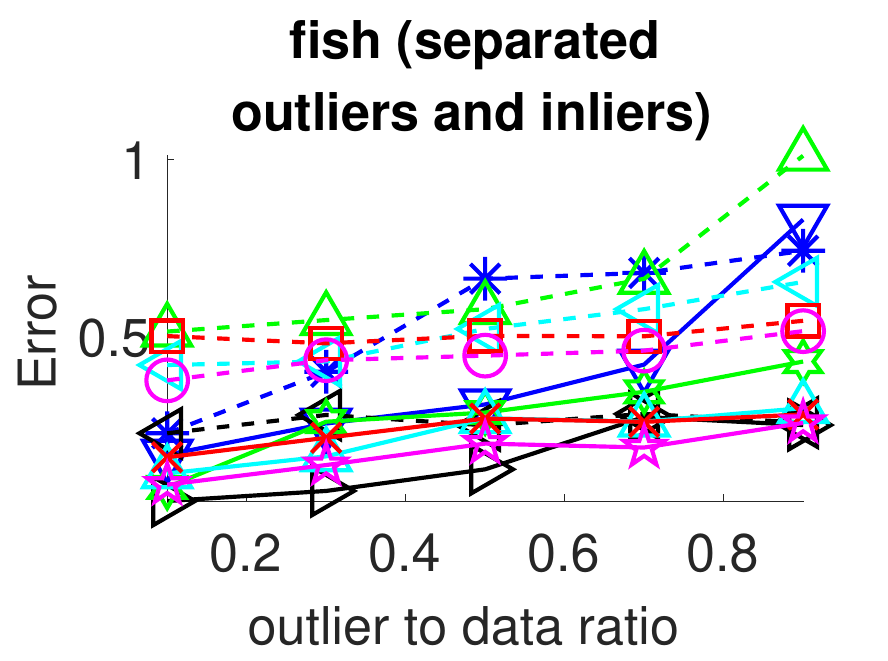}&
		\includegraphics[width=\scaleGd\linewidth]{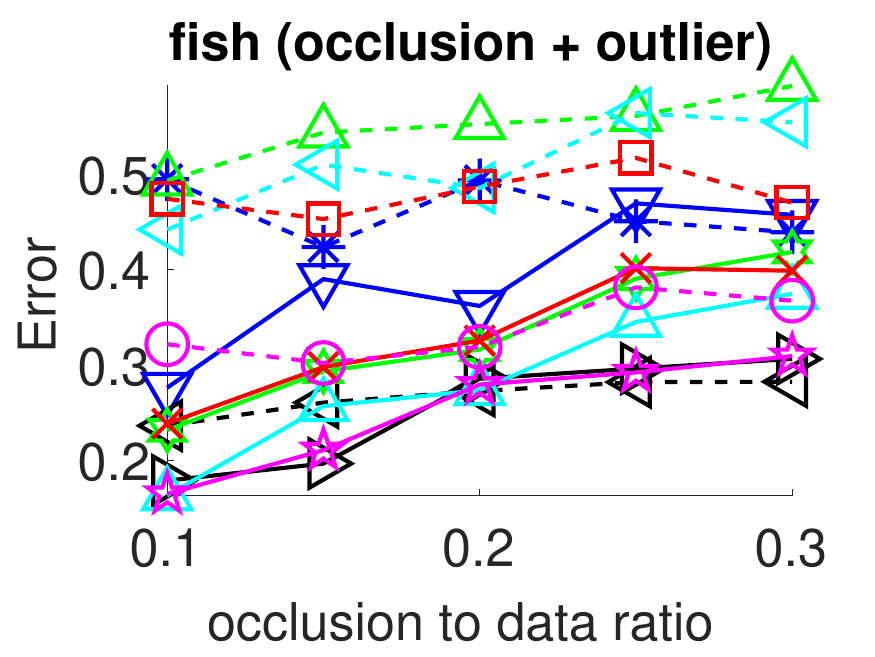}\\
		\includegraphics[width=\scaleGd\linewidth]{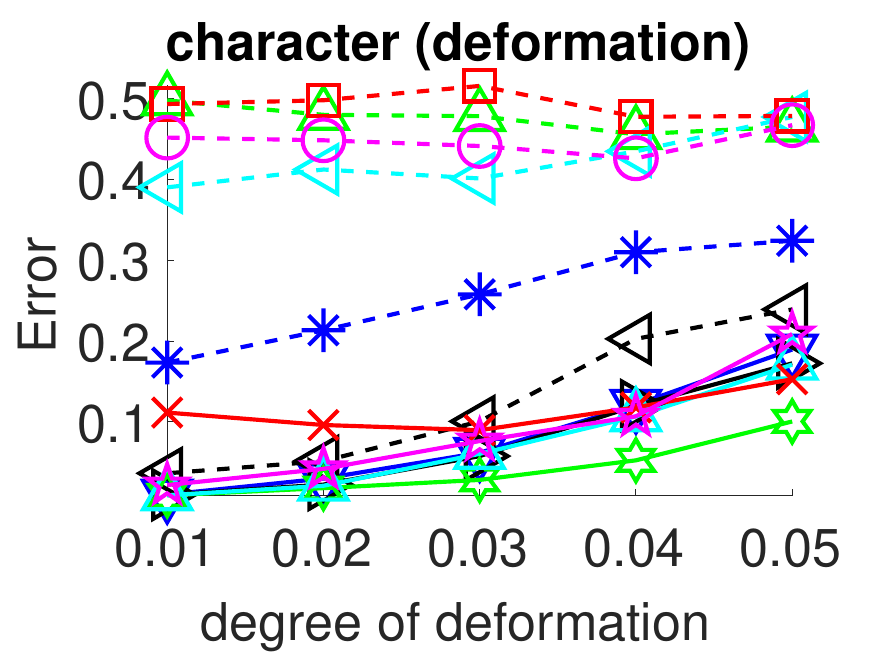}&	
		\includegraphics[width=\scaleGd\linewidth]{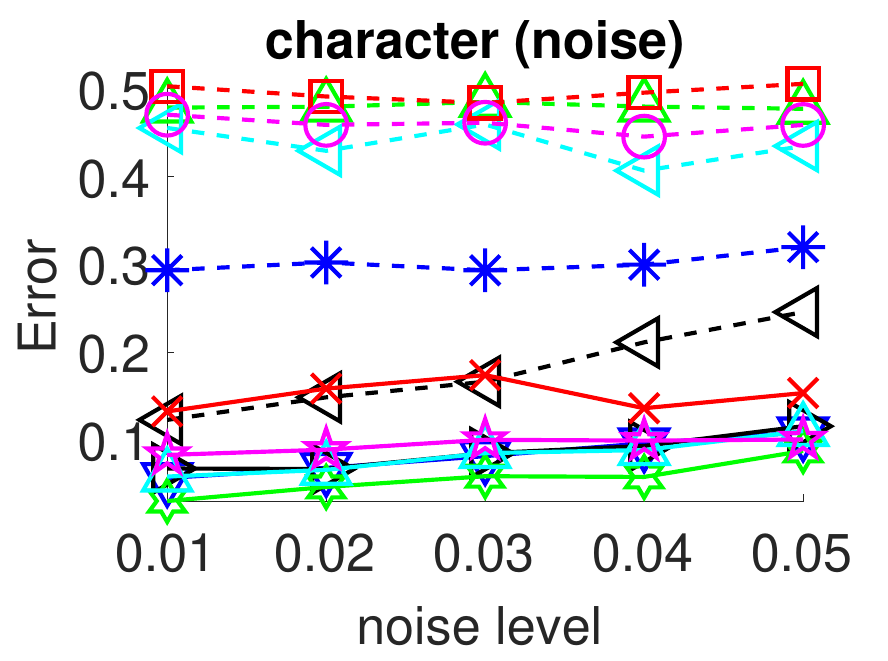}&		
		\includegraphics[width=\scaleGd\linewidth]{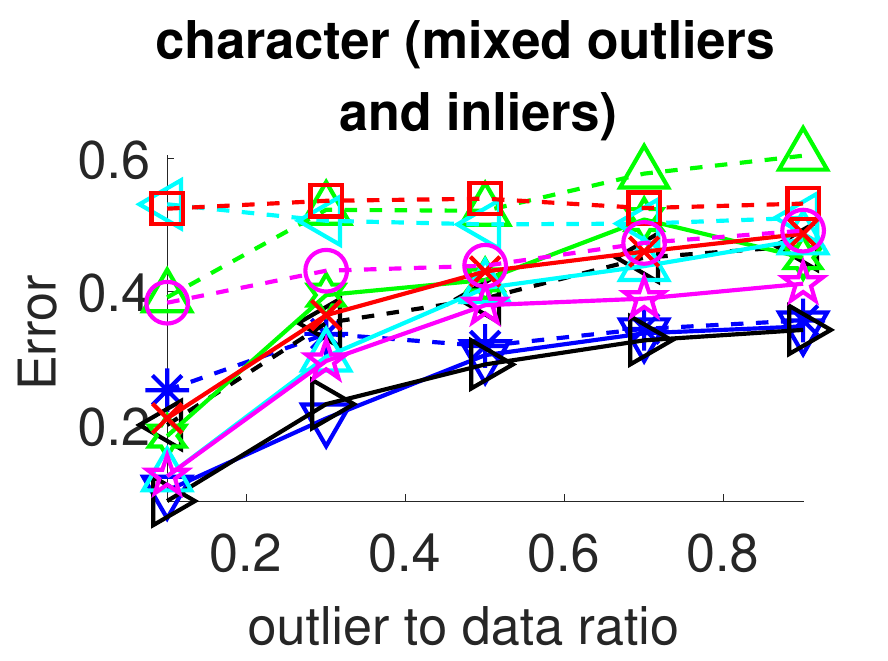}&
		\includegraphics[width=\scaleGd\linewidth]{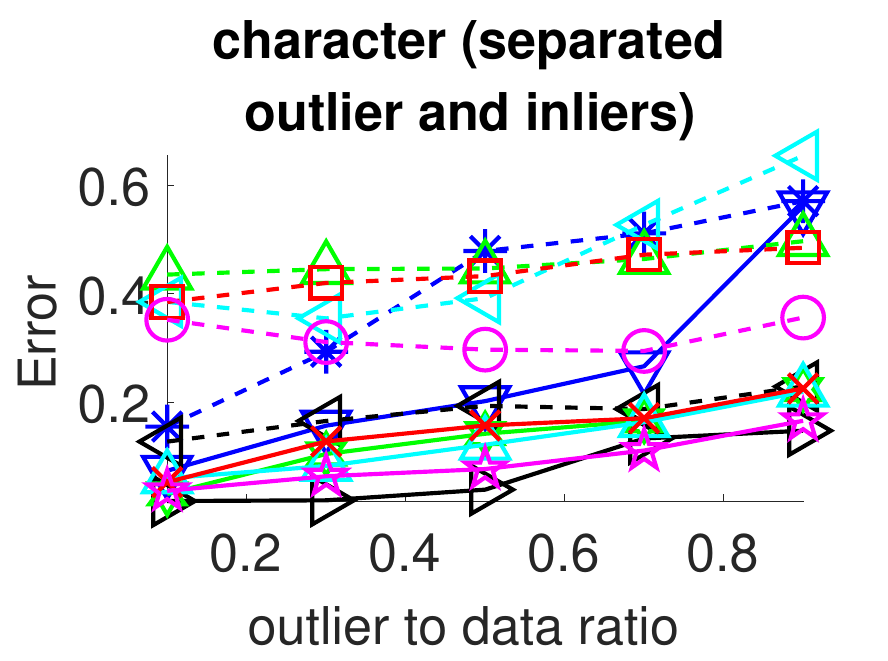}&		
		\includegraphics[width=\scaleGd\linewidth]{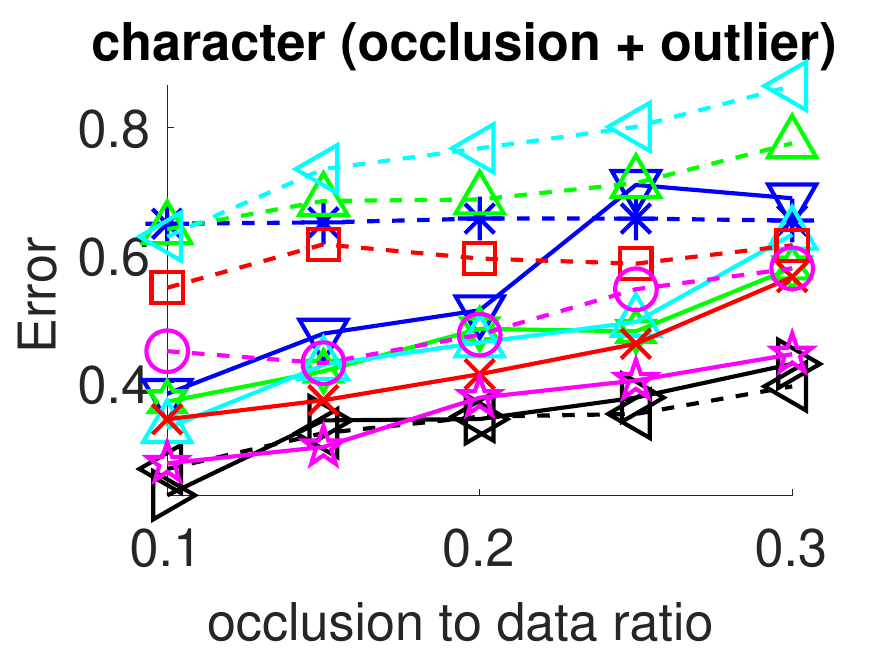}
	\end{tabular}
%	\begin{tabular}{c}
		\includegraphics[width=.9\linewidth]{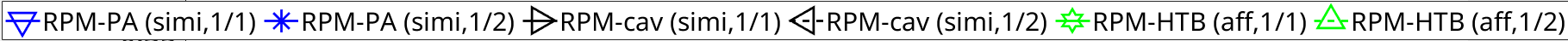}\\
		\vspace{-8pt}
	\includegraphics[width=.9\linewidth]{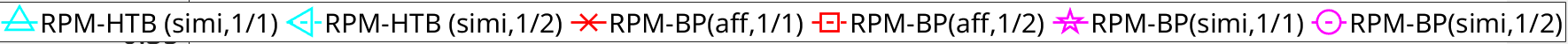}
%\end{tabular}	
\begin{tabular}{@{\hspace{-0mm}}c@{\hspace{-1.6mm}}c@{\hspace{-1.6mm}}c@{\hspace{-1.6mm}}c@{\hspace{-1.6mm}} c }
	\includegraphics[width=\scaleGd\linewidth]{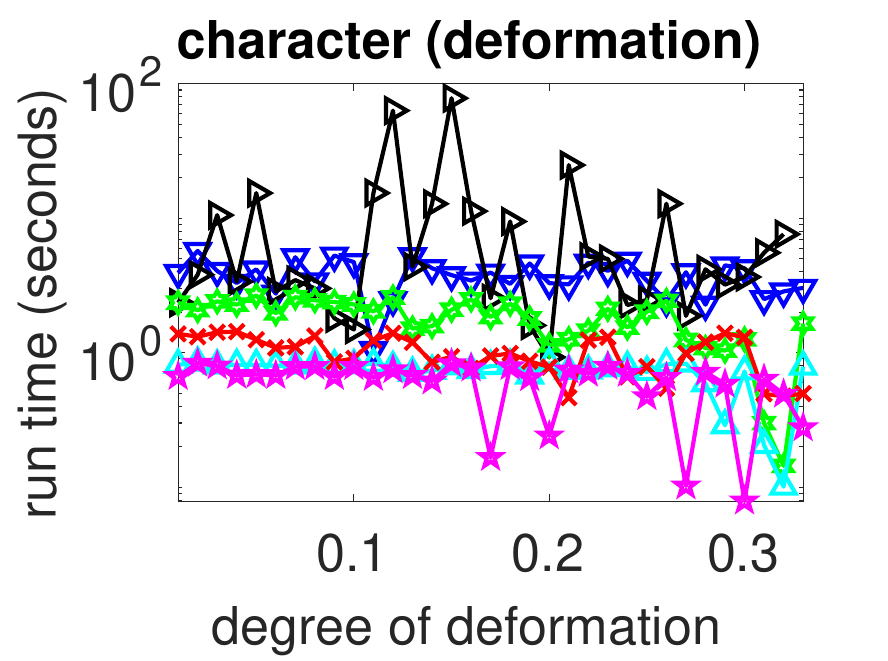}&
	\includegraphics[width=\scaleGd\linewidth]{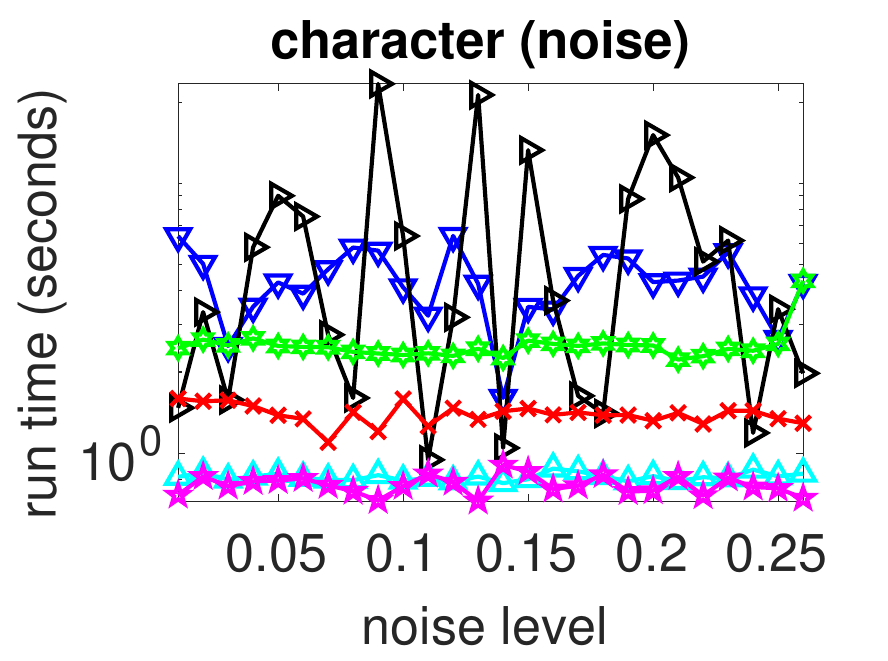}&
	\includegraphics[width=\scaleGd\linewidth]{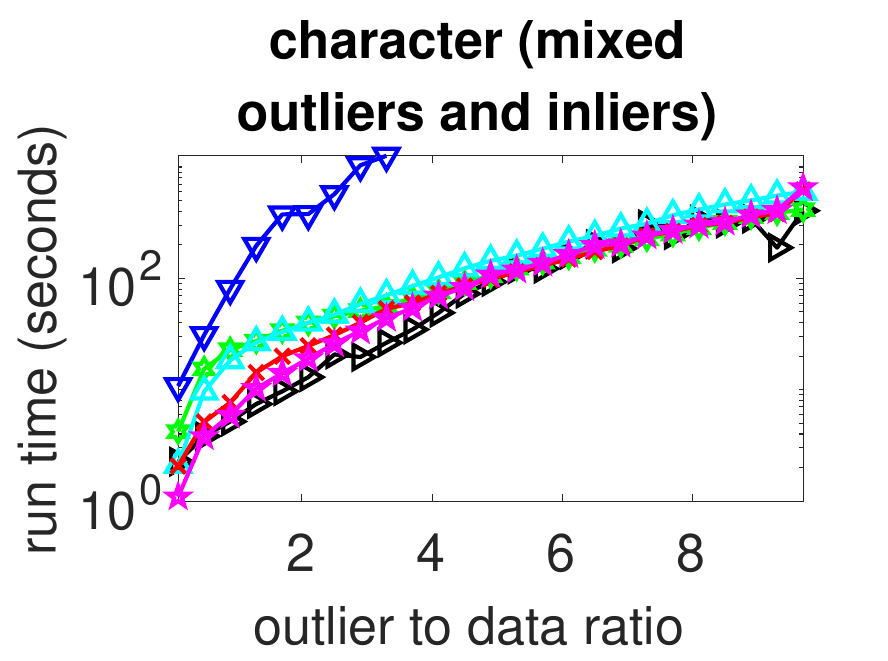}&
	\includegraphics[width=\scaleGd\linewidth]{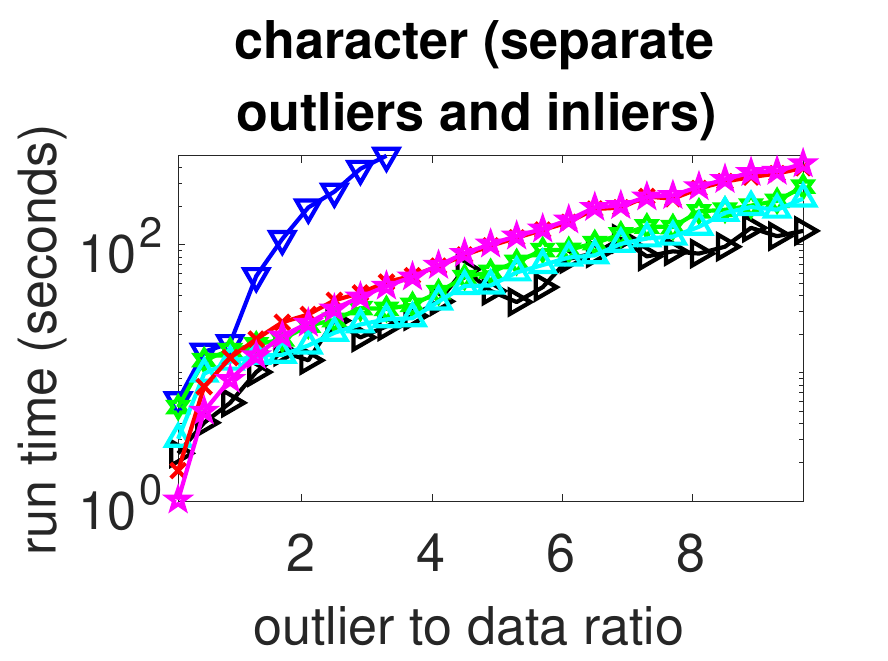}	&
	\includegraphics[width=\scaleGd\linewidth]{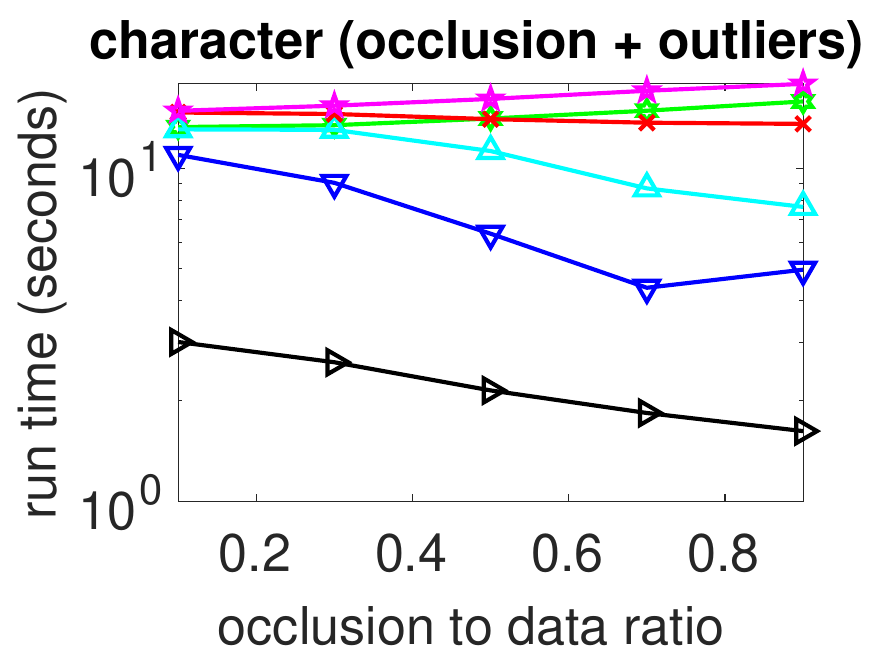}	\end{tabular}
\begin{tabular}{@{\hspace{-1.5mm}}c}	\includegraphics[width=.95\linewidth]{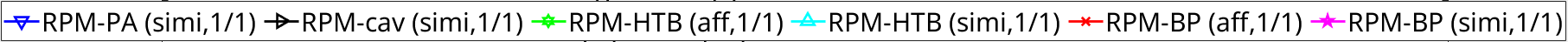}
\end{tabular}
\caption{
	Average registration errors (top 2 rows) and run times (bottom row) by 
	various methods
%	RPM-BP, RPM-HTB \cite{LIAN2023126482}, and RPM-CAV \cite{RPM_model_occlude_PR}
	 under various $n_p$ values (ranging from $1/2$ to $1/1$ of the ground truth value) over 100 random trials for 2D deformation, positional noise, mixed outliers and inliers, separate outliers and inliers, and occlusion+outlier tests.
	\label{2D_simi_sta}
}
	%	
%\end{minipage}	
%\end{figure*}
%\begin{figure*} [t]
%\begin{minipage}{1\textwidth}
	\centering
	\includegraphics[width=.95\linewidth]{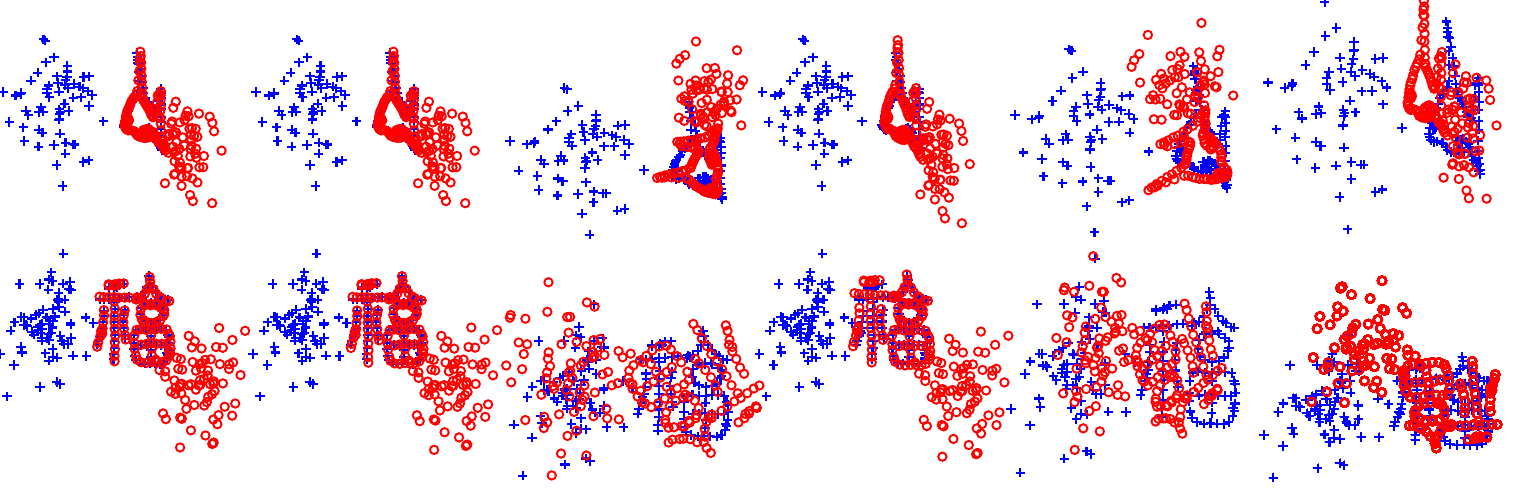}
\newcommand\sctwo{4}		
{\scriptsize
	\begin{tabular}{@{\hspace{0mm}}c @{\hspace{\sctwo mm}} c@{\hspace{\sctwo mm}}  c@{\hspace{\sctwo mm}}  c@{\hspace{\sctwo mm}}   c@{\hspace{\sctwo mm}} c}	
		(a)RPM-BP(simi) &  (b)RPM-BP(aff) & (c)RPM-HTB(simi) & (d)RPM-HTB(aff) & (e)RPM-PA & (f)RPM-CAV
	\end{tabular}
}
	\caption{
		Example of registration results from different methods in the separate outliers and inliers test, with $n_p$ chosen as ground truth for all methods.		
	\label{rot_2D_syn_match_exa}}	
%\end{minipage}
\end{figure*}

\subsubsection{2D Real-World Data}
Point sets extracted from images represent a realistic and challenging data type for evaluating the robustness and generalizability of registration methods, as they inherently contain noise, irregular contours, and complex structural features.

\paragraph{Experimental Setup}
To generate these point sets, we utilized the standard Canny edge detector on selected images sourced from the Caltech-256 \cite{caltech_database} and VOC2007 \cite{pascal-voc-2007} datasets. A critical test of rotational invariance was implemented by subjecting the model point set to a deliberate $180^\circ$ rotation prior to registration against the scene point set. For consistency, a Similarity transformation model was assumed and used for all tested methods. The visual setup for this experiment is detailed in Fig. \ref{rot_2D_canny_model}, with final alignment results shown in Fig. \ref{rot_2D_canny}.

\paragraph{Results}
The registration outcomes presented in Fig. \ref{rot_2D_canny} demonstrate that our method consistently achieves a  tighter  alignment to the underlying target objects than competing algorithms. This advantage is notably pronounced in the bike image test.

\begin{figure*}[!t]
	\setlength\arrayrulewidth{1pt}
	
	\centering
	\newcommand\scale{0.6}
	\includegraphics[width=\scale\linewidth]{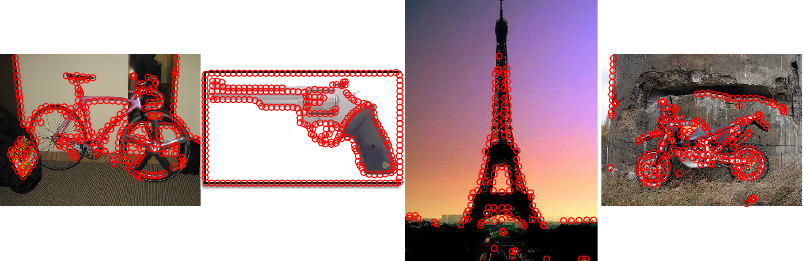}
	
	\caption{
		Model  images with model point sets superimposed.		
		\label{rot_2D_canny_model}}
%\end{figure*}
%\begin{figure*}[!t]
	\setlength\arrayrulewidth{1pt}	
	\centering
	\begin{tabular}{c@{\hspace{-0mm}}c} %{@{\hspace{0mm}}c@{\hspace{-0mm}}@{\hspace{.0mm}}c }
		\includegraphics[width=0.45\linewidth]{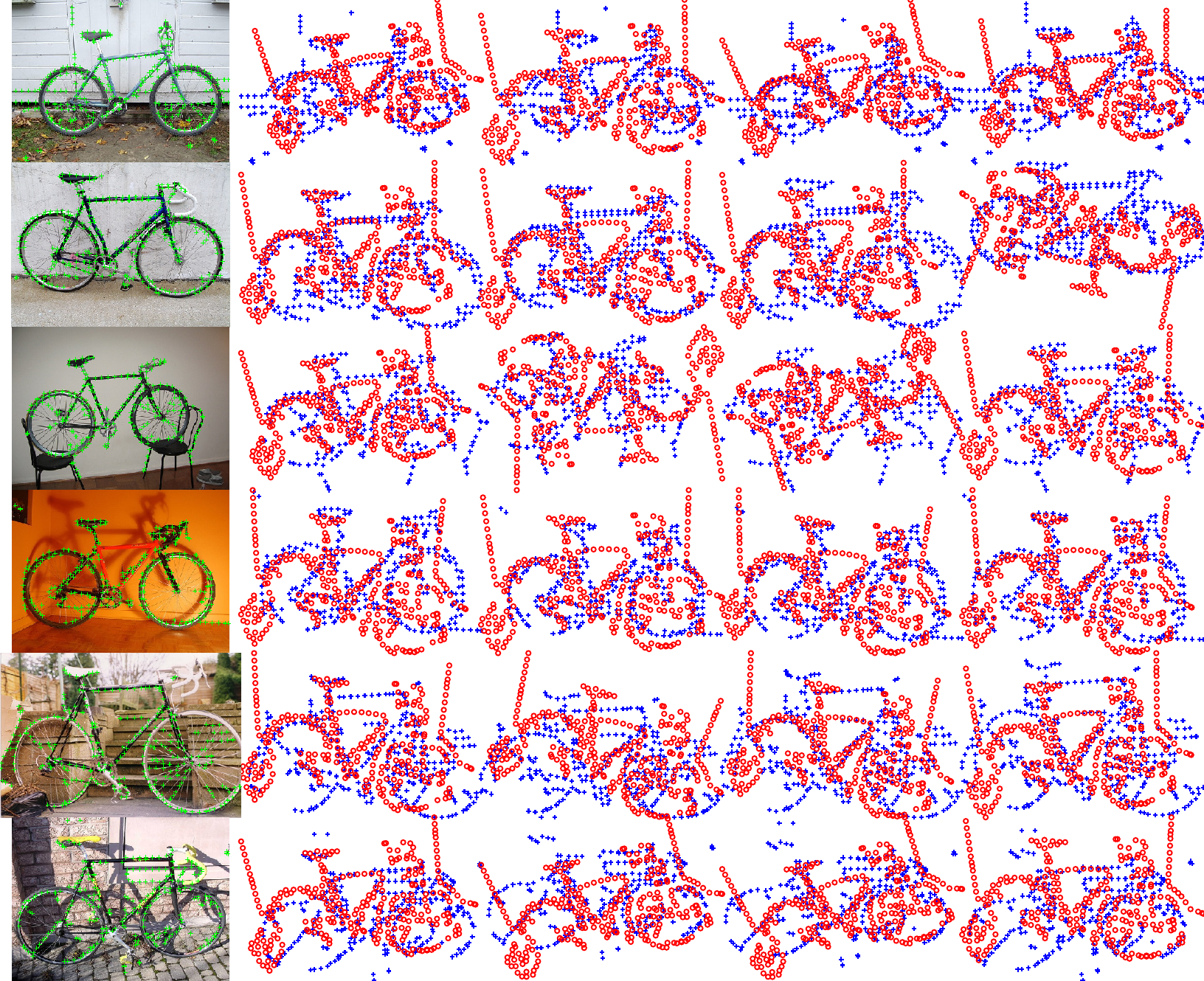} &
\includegraphics[width=.52\linewidth]{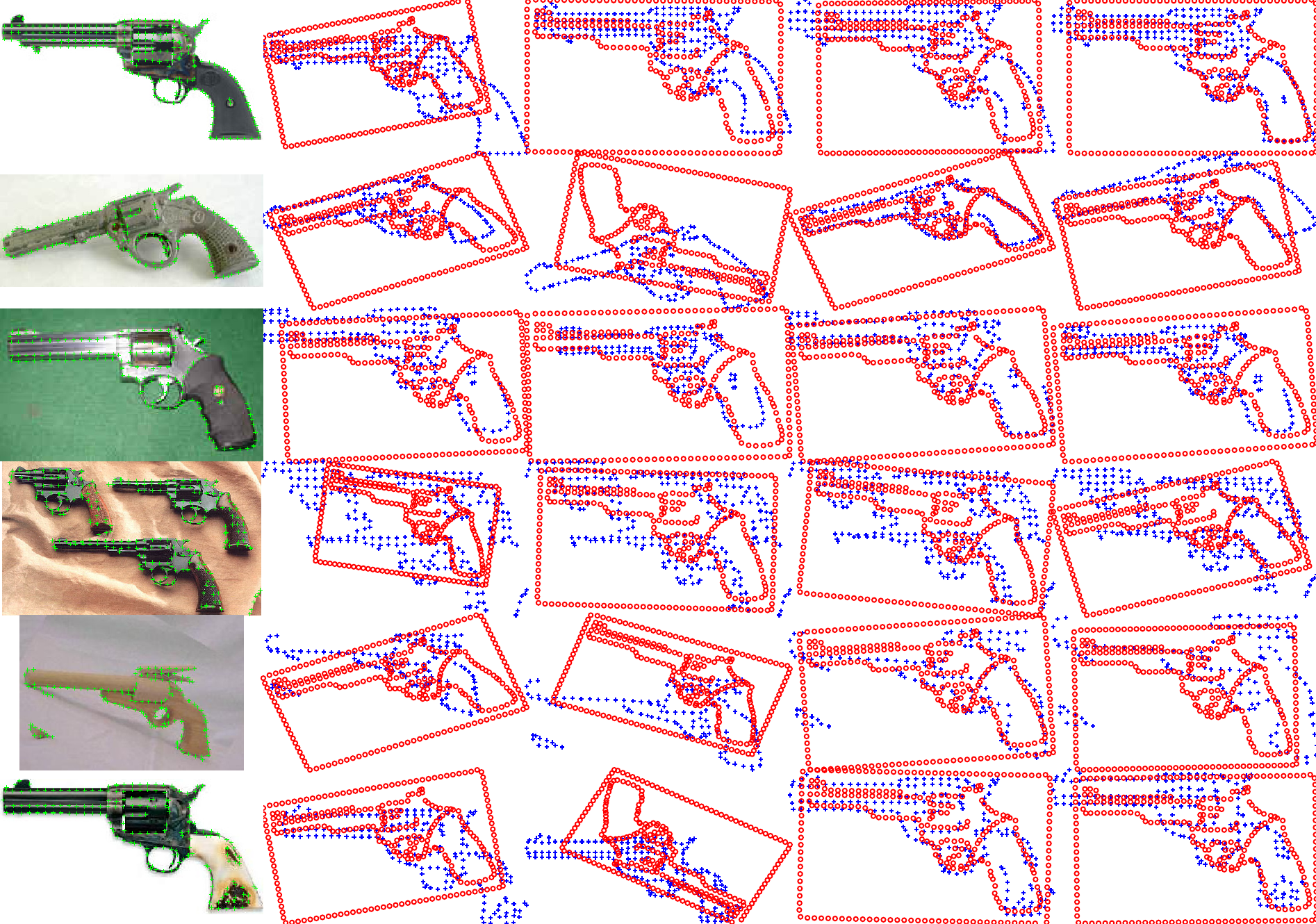} 		\\	
		\includegraphics[width=.43\linewidth]{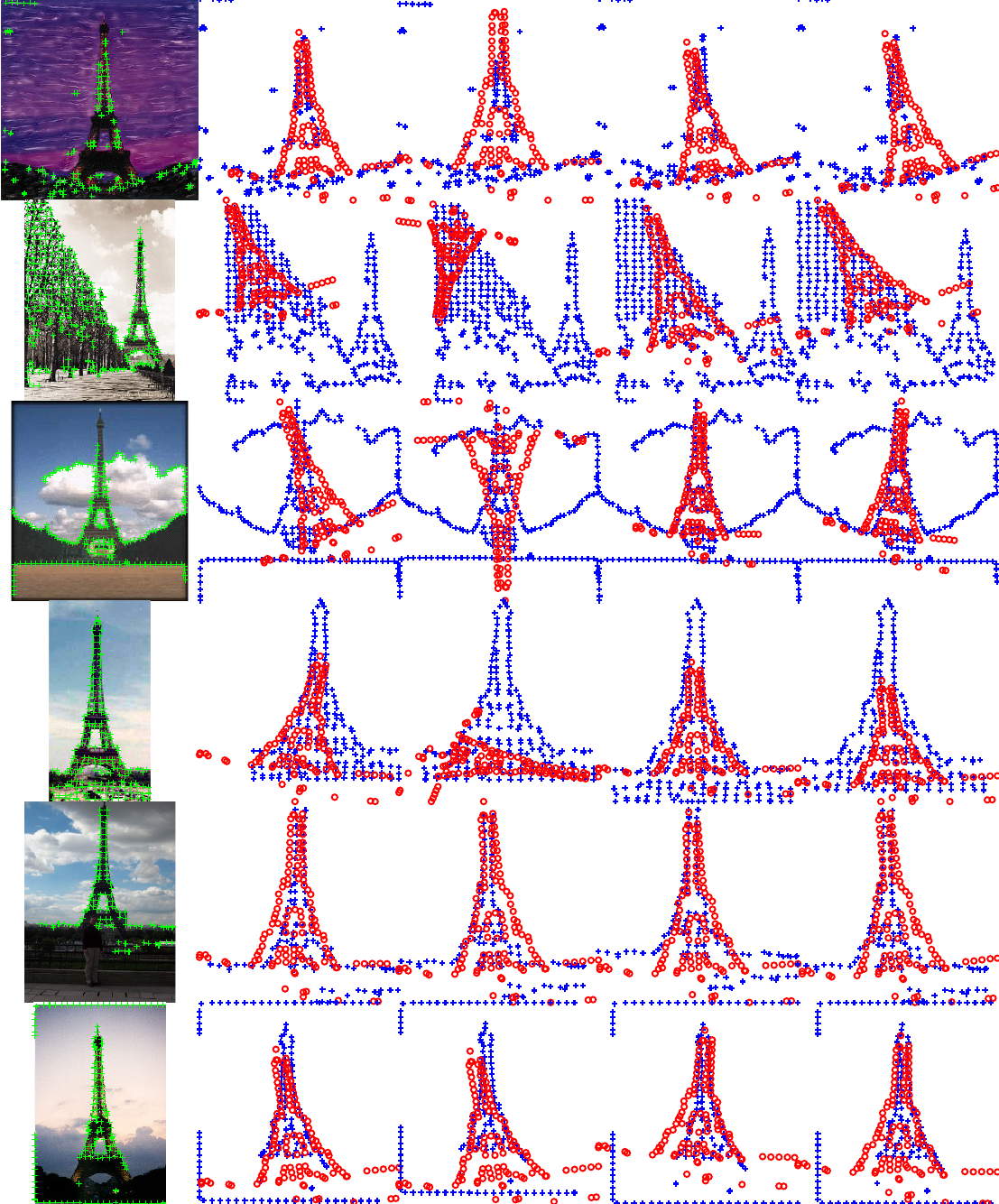} &
		\includegraphics[width=.56\linewidth]{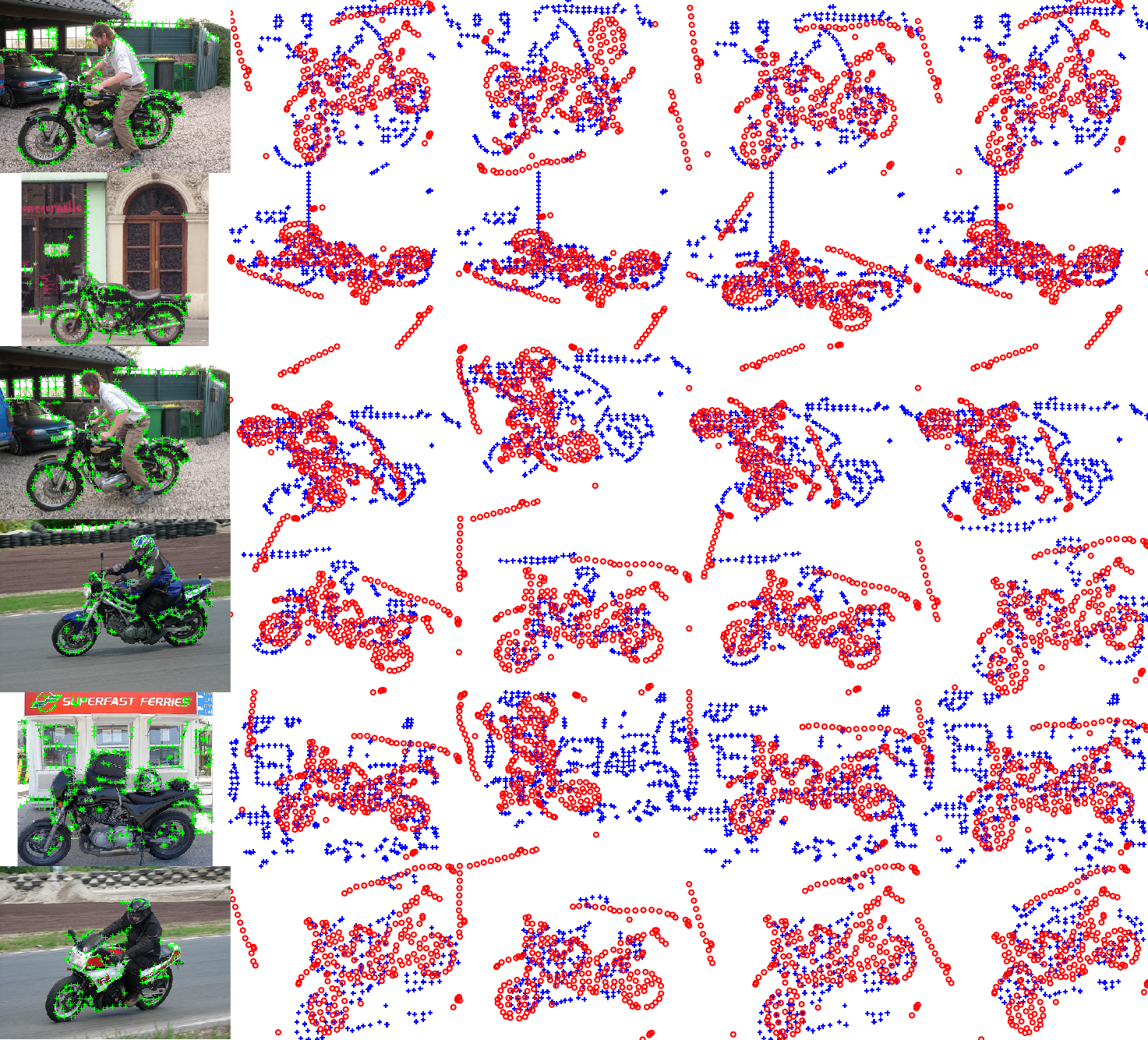} 	
	\end{tabular}
	%\small
	%\begin{tabular}{@{\hspace{-5mm}}c @{\hspace{0 mm}} c@{\hspace{0 mm}}  c@{\hspace{0 mm}}  c@{\hspace{0 mm}}   c@{\hspace{30mm}}}	
	%	(a) scene & (b) RPM-HTB (simi.) & (c) RPM-HTB & (d) Go-ICP & (e) FRS
	%\end{tabular}
	\caption{
		%		Left grid cell: model  images with model point sets superimposed.
		%		The remaining  cells:
		For each category: 
		scene images with scene point sets superimposed, registration results by RPM-BP,  RPM-HTB \cite{LIAN2023126482}, RPM-PA \cite{lian2021polyhedral} and RPM-CAV \cite{RPM_model_occlude_PR} 
		using similarity  transformation.
		The $n_p$ value for each method is chosen as $0.9$ the minimum of the cardinalities of two point sets.
%		The failure cases are encircled by black boxes.
		\label{rot_2D_canny}}
\end{figure*}

\subsection{3D Registration}
\subsubsection{3D Synthetic Data \label{sec:3D_synth_test}}
Similar to the 2D experiments, we conducted five types of tests to evaluate each method's resilience to various disturbances: \textit{i)} Deformation, \textit{ii)} Noise, \textit{iii)} Mixed outliers, \textit{iv)} Separate outliers, and \textit{v)} Occlusion + Outlier tests. Fig. \ref{rot_3D_test_data_exa} provides a visual illustration of these tests, while Fig. \ref{rot_3D_syn_match_exa} shows examples of registration results from different methods.

The registration errors for various methods are shown in the top two rows of Fig. \ref{3D_rigid_sta}. The results indicate that RPM-BP and RPM-HTB, when $n_p$ is set to the ground truth value, are robust against deformation, noise, and separate outliers. However, their robustness diminishes when outliers are mixed with inliers. In contrast, GORE and TEASER++ struggle with deformation, noise, and mixed outliers, likely because they rely on features. As in the 2D case, the performance of RPM-BP improves significantly when $n_p$ closely matches the ground truth.

The bottom row of Fig. \ref{3D_rigid_sta} shows the average run times. TEASER++ is the fastest, followed by Go-ICP, and then RPM-BP and RPM-HTB. GORE is consistently slower in most cases, while FRS experiences a rapid decrease in efficiency when the registration problem becomes more challenging, such as with an increasing number of outliers.

In terms of scalability with problem size, Go-ICP and GORE perform best, followed by TEASER++, RPM-BP, and RPM-HTB. FRS, on the other hand, shows the poorest scalability.

\begin{figure*} [t]
			\setlength{\abovecaptionskip}{-1pt plus 0pt minus 2pt} % Chosen fairly arbitrarily	
	\centering
	\newcommand\scale{0.104}		
	\begin{tabular}{@{\hspace{-3mm}}c@{}|@{}c@{}|@{}c@{}|@{}c@{}|@{}c@{}|@{} c@{}|@{} c@{} |c@{}|c }						
		\includegraphics[width=\scale\linewidth]{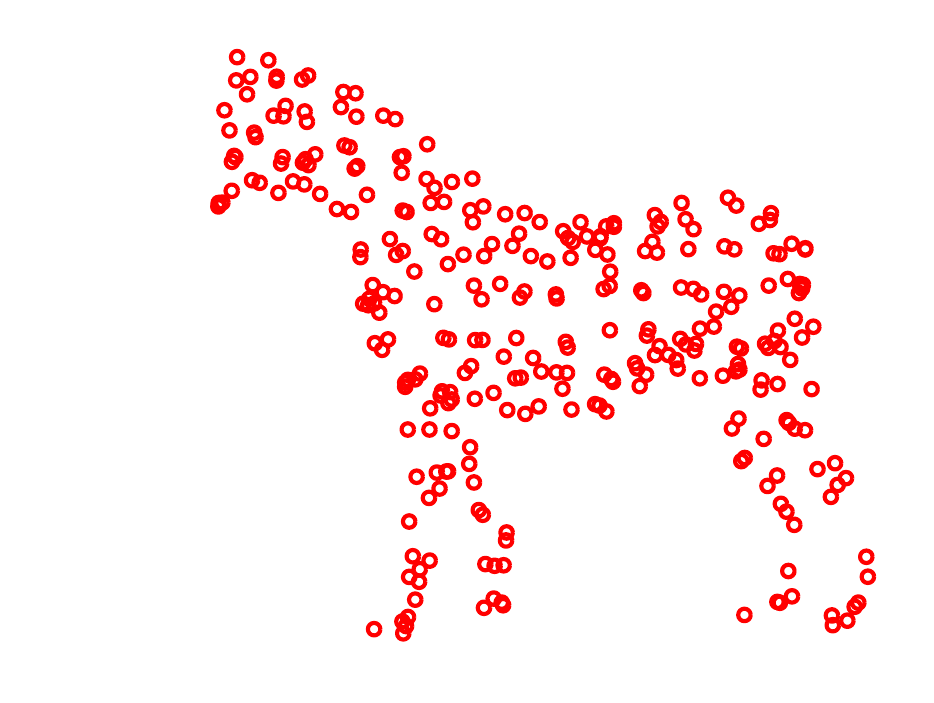}&
		\includegraphics[width=\scale\linewidth]{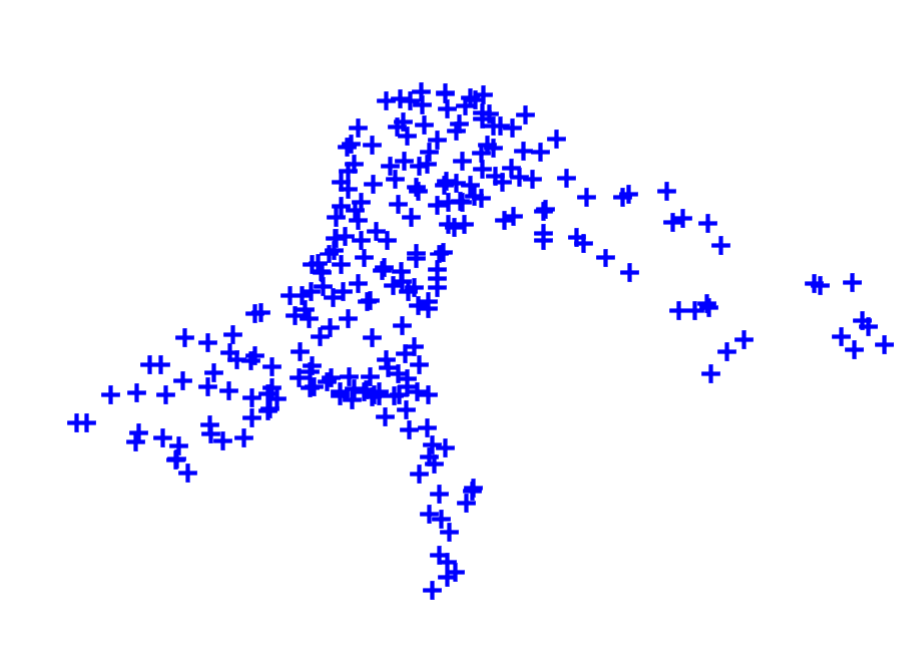}&
		\includegraphics[width=\scale\linewidth]{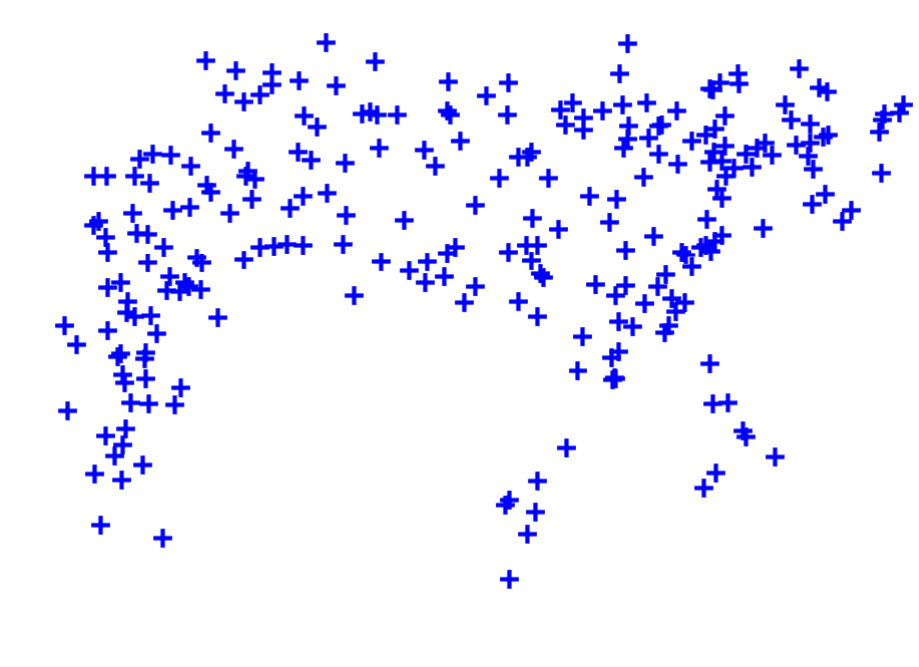}&
		\includegraphics[width=\scale\linewidth]{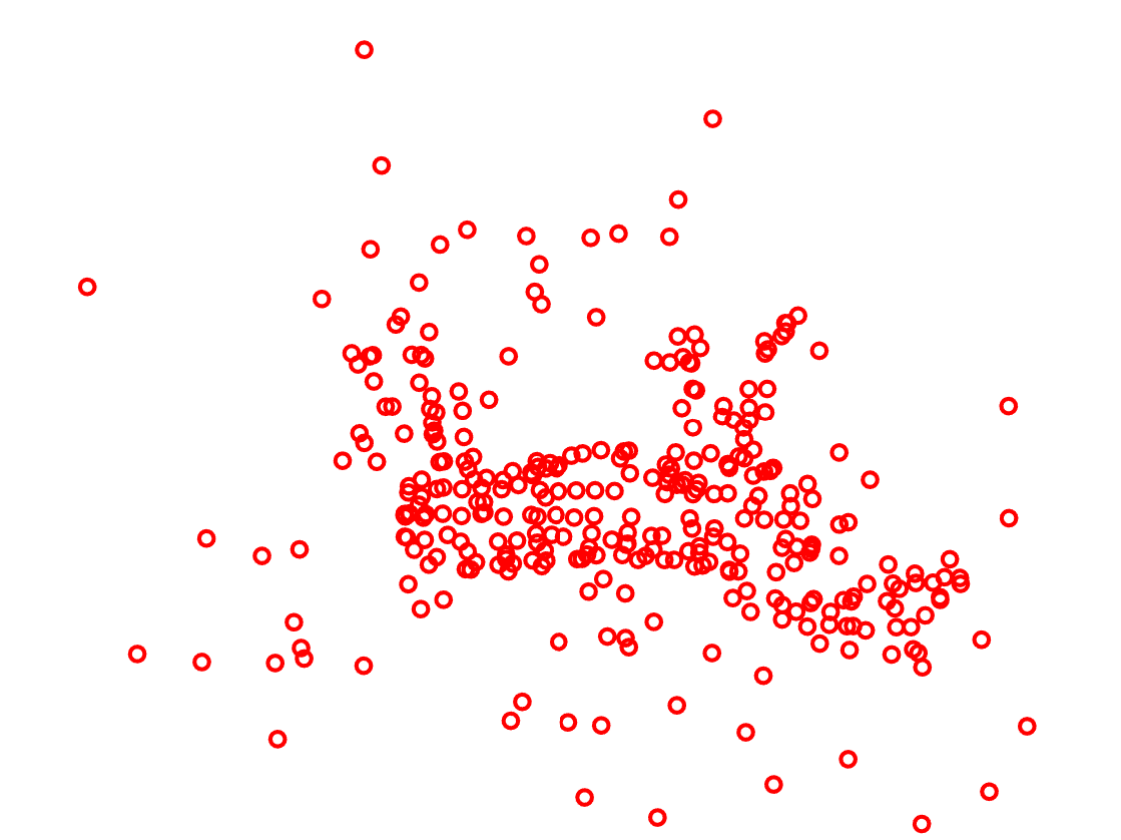}&	
		\includegraphics[width=\scale\linewidth]{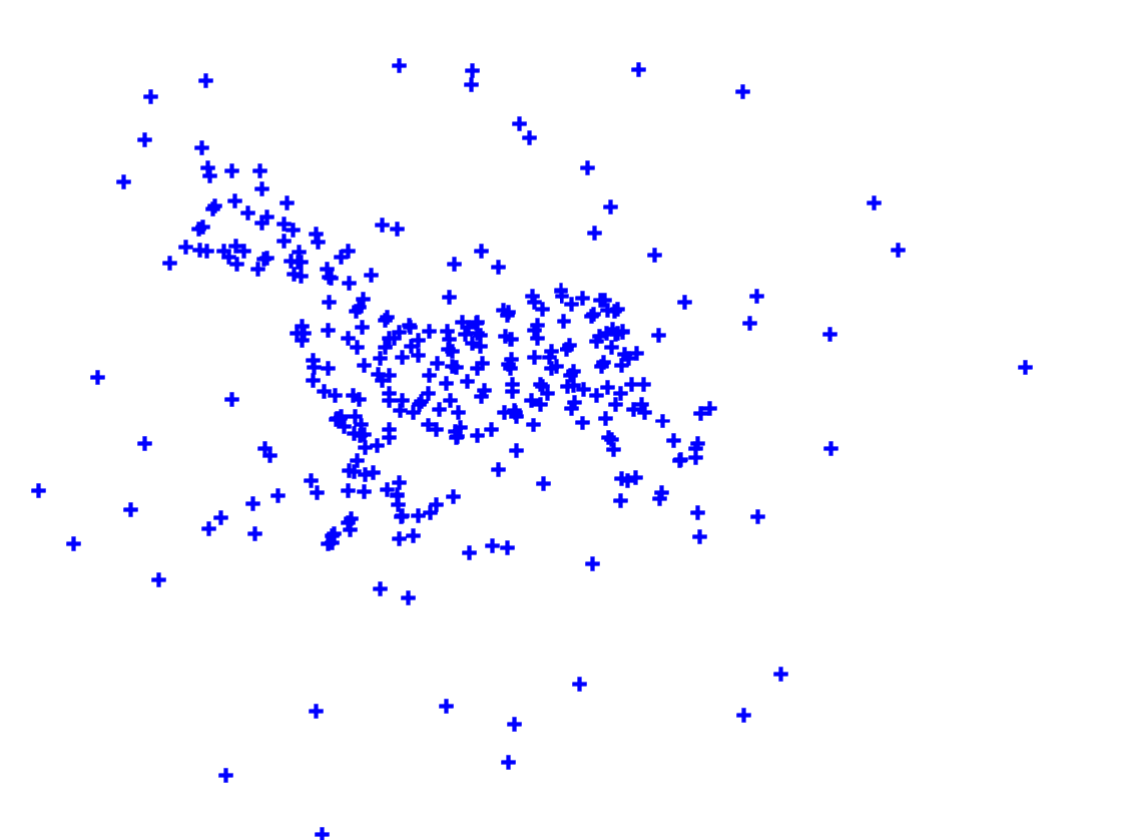}&					
		\includegraphics[width=\scale\linewidth]{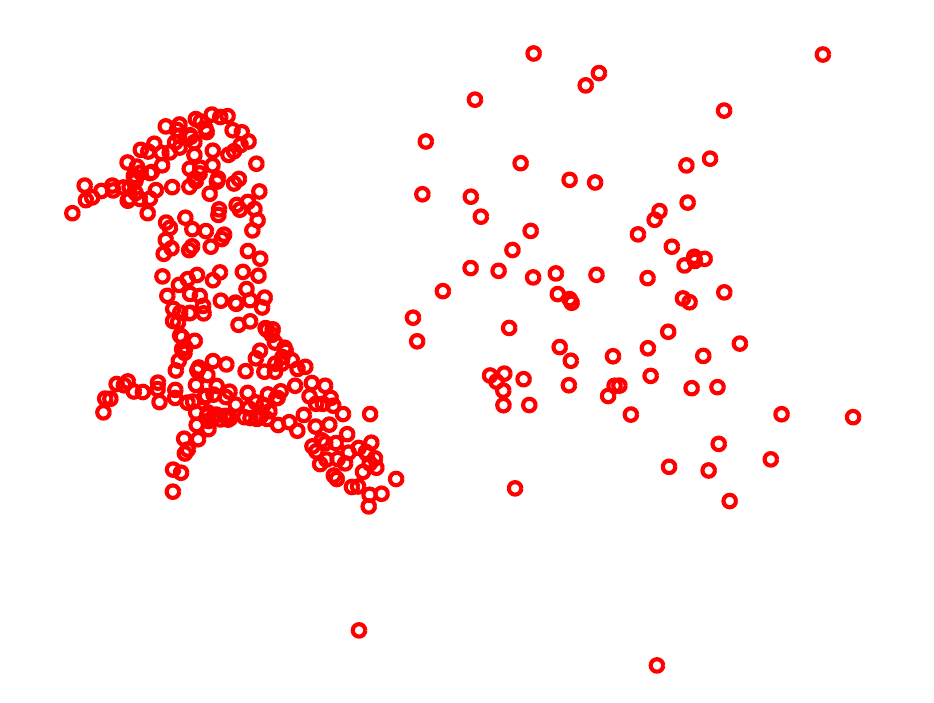}&
		\includegraphics[width=\scale\linewidth]{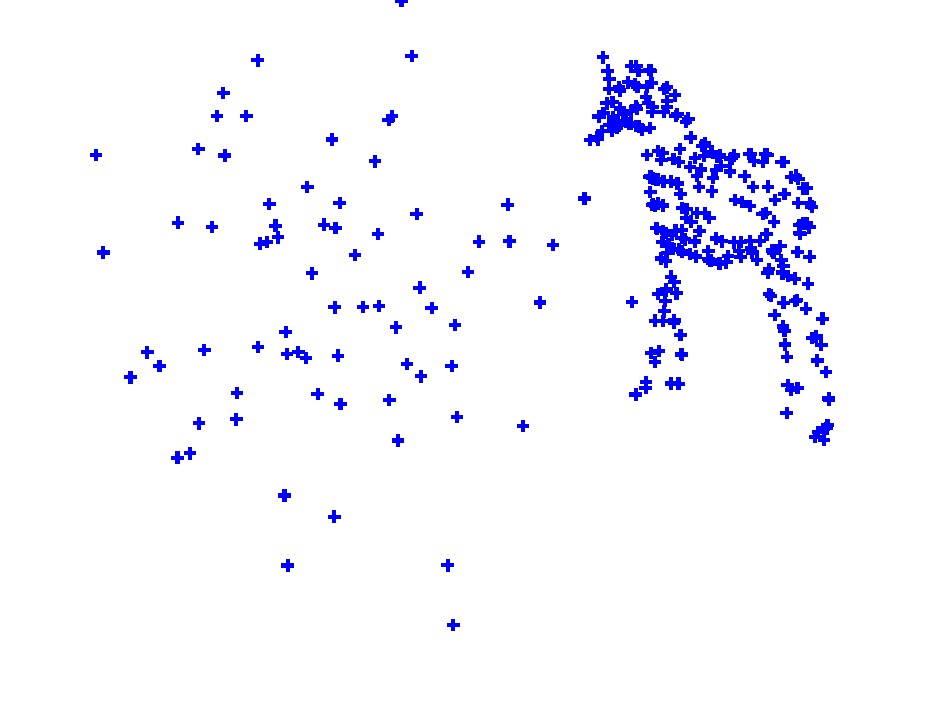}&
		\includegraphics[width=\scale\linewidth]{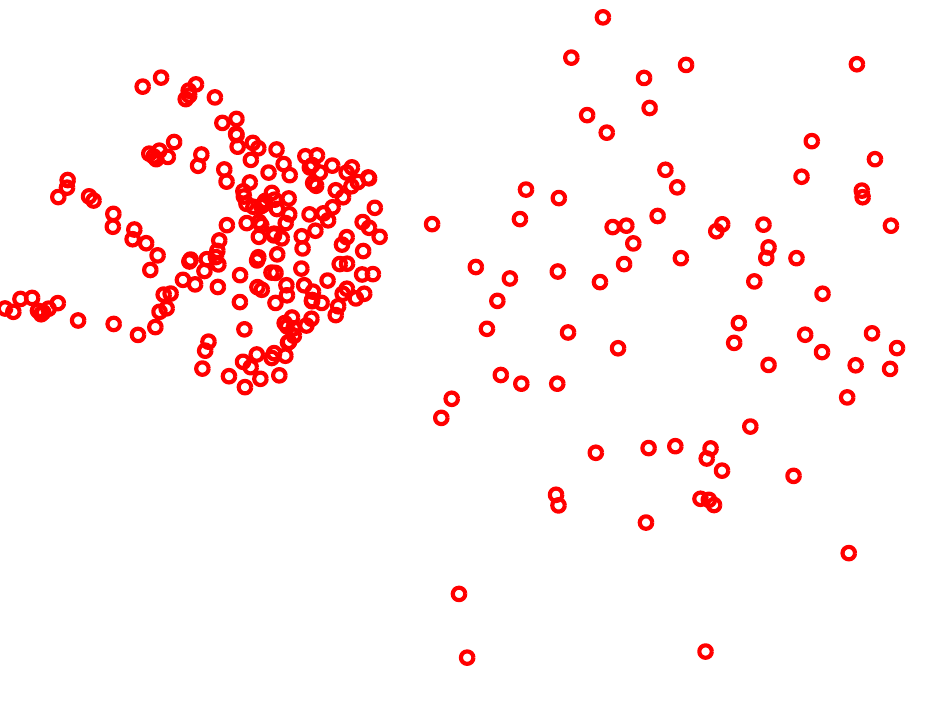}&
		\includegraphics[width=\scale\linewidth]{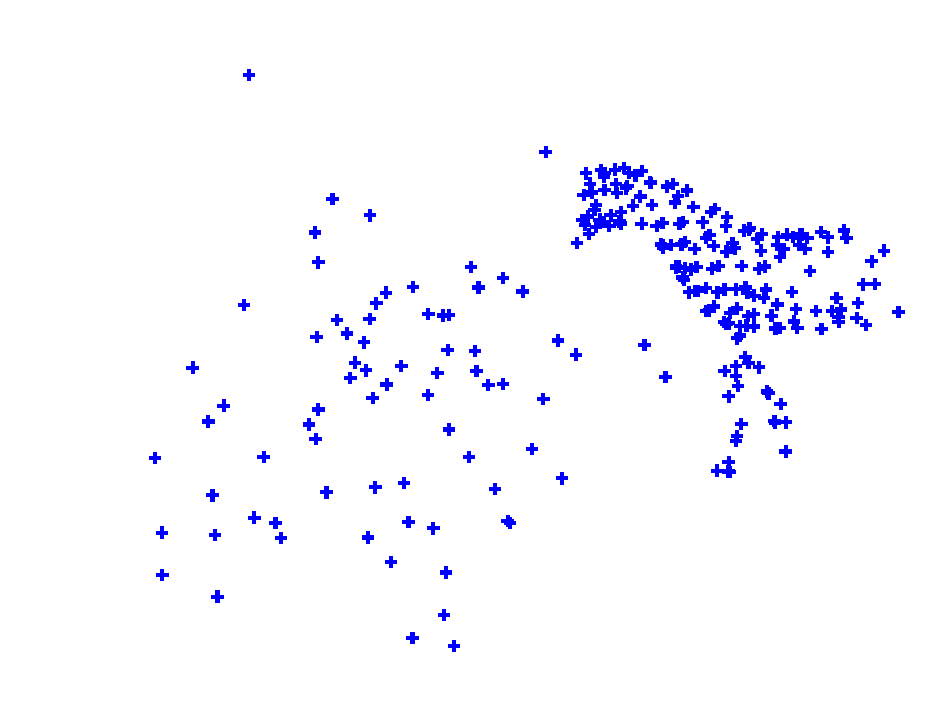} 
		\\\hline
		%	  \vspace{1mm}
		%
		\subfigure[]{		
			\includegraphics[width=\scale\linewidth]{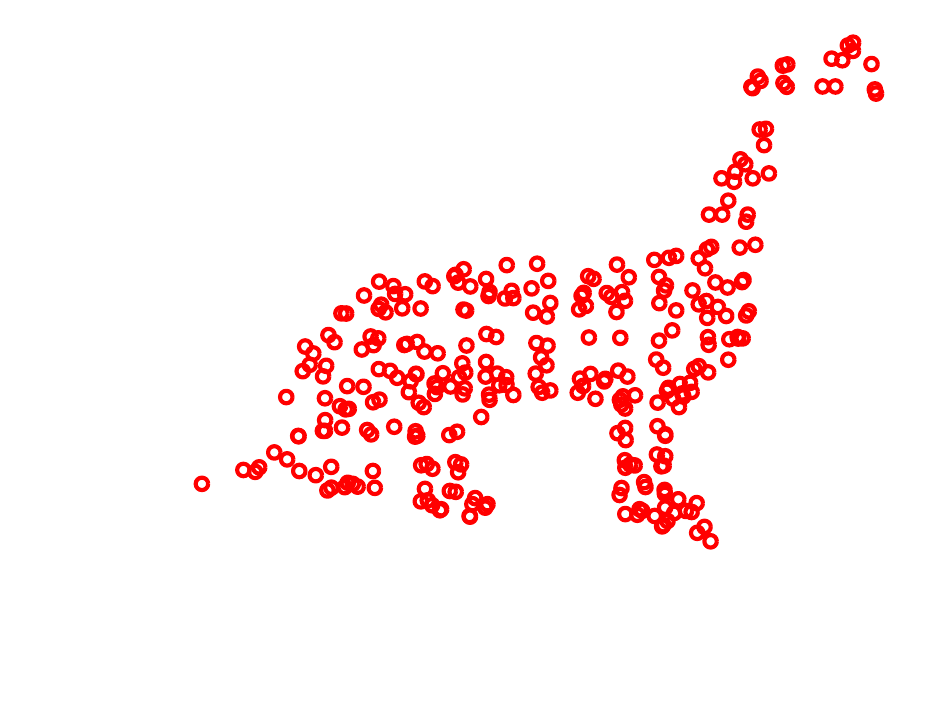}} &
		\subfigure[]{		
			\includegraphics[width=\scale\linewidth]{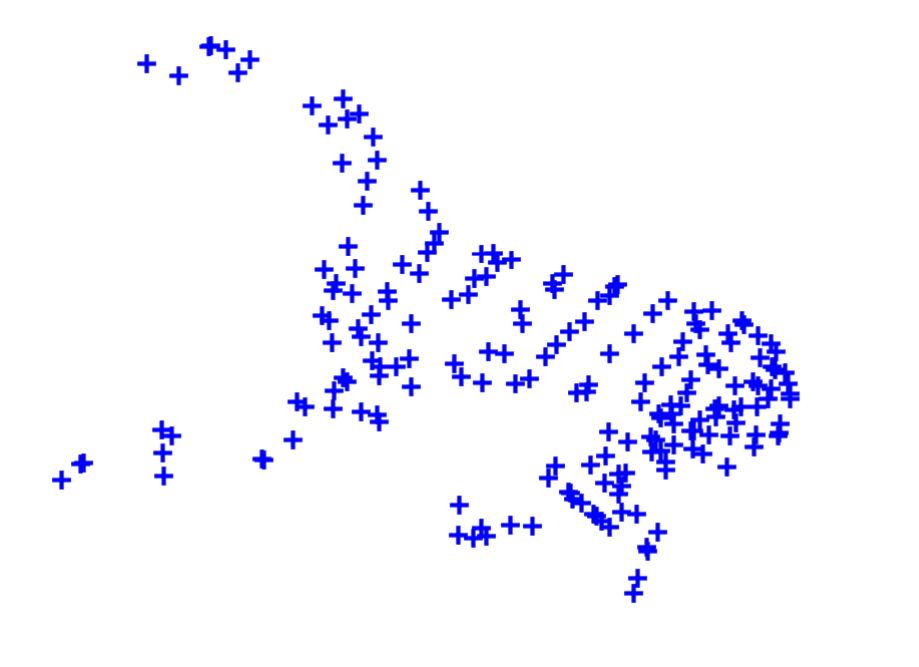}}&
		\subfigure[]{		
			\includegraphics[width=\scale\linewidth]{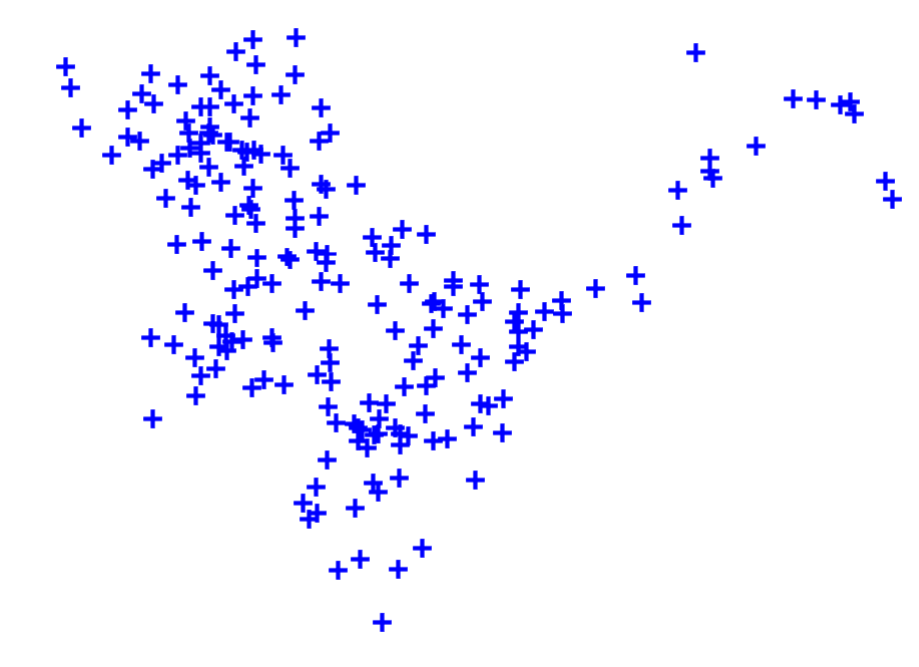}}&
		\subfigure[]{\includegraphics[width=\scale\linewidth]{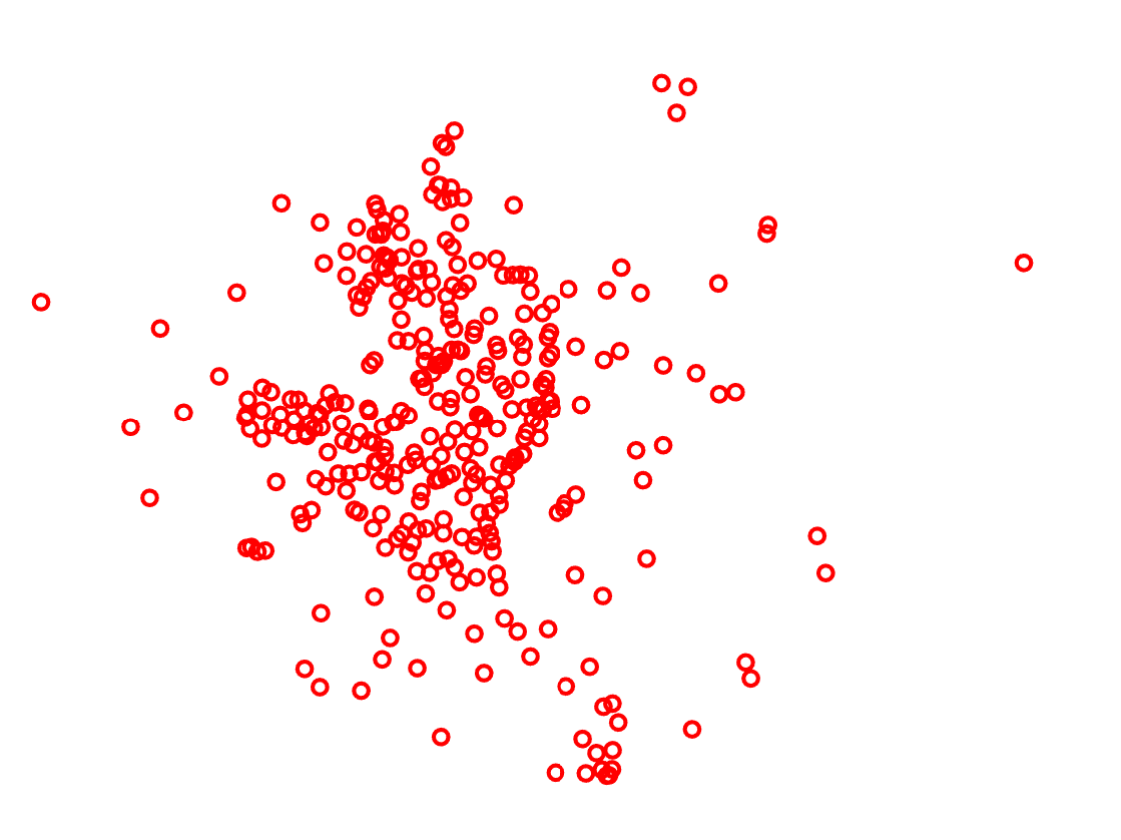}}&
		\subfigure[]{\includegraphics[width=\scale\linewidth]{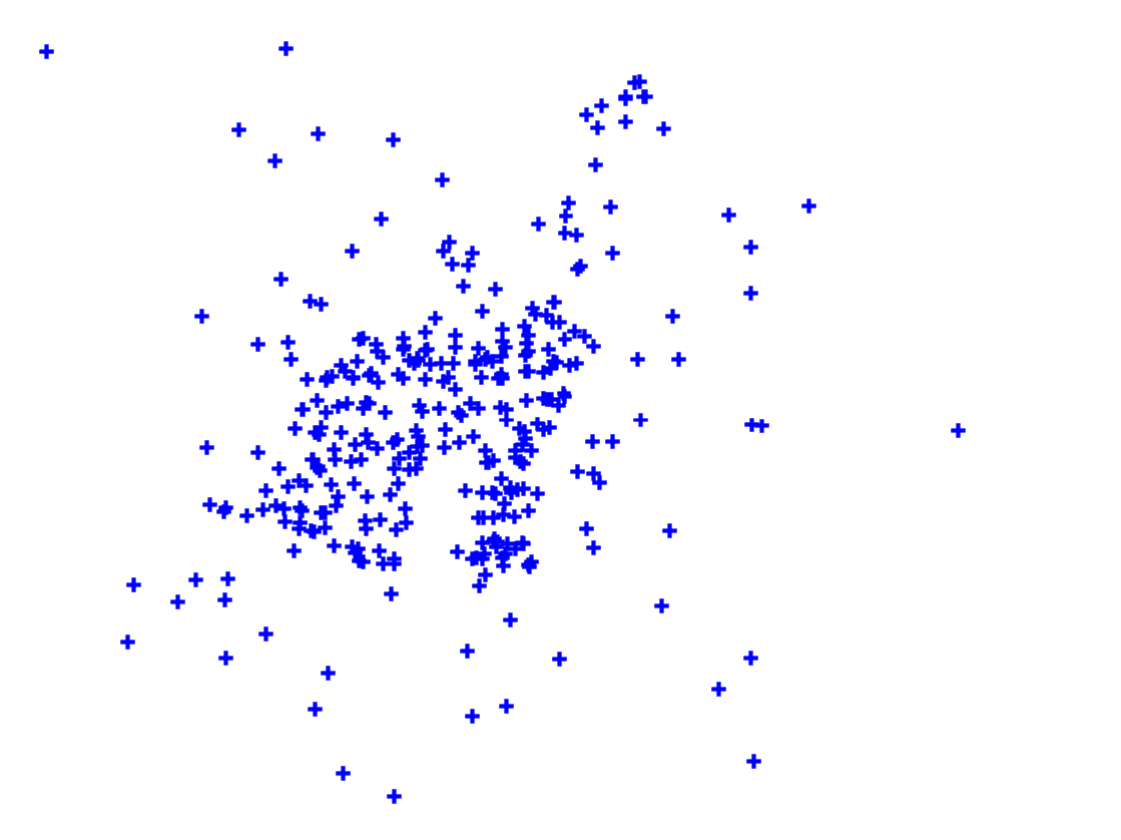}}&			
		\subfigure[]{		
			\includegraphics[width=\scale\linewidth]{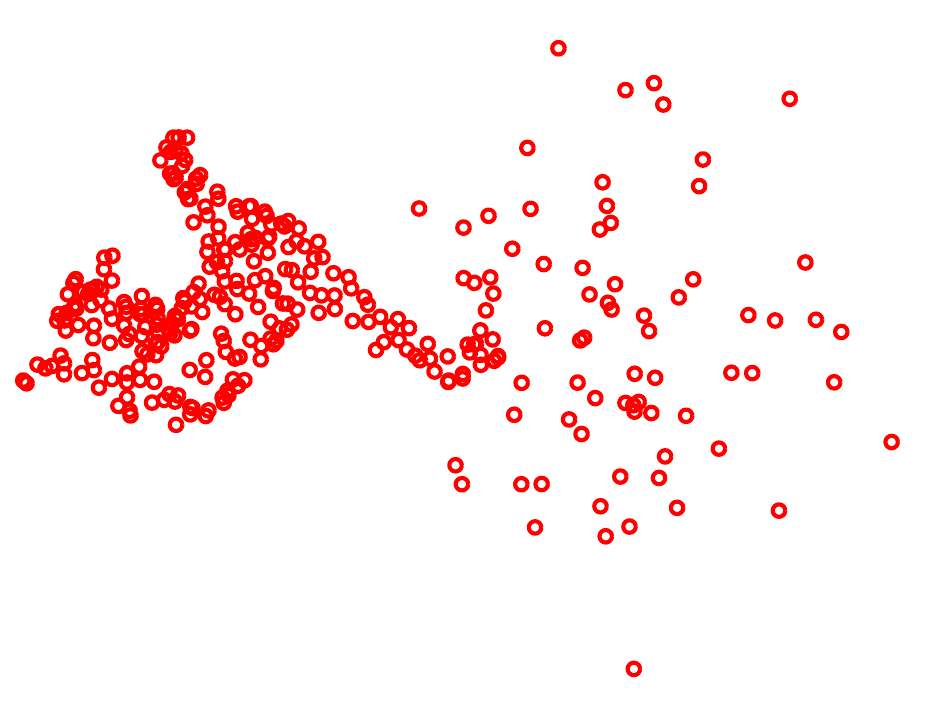}}&
		\subfigure[]{		
			\includegraphics[width=\scale\linewidth]{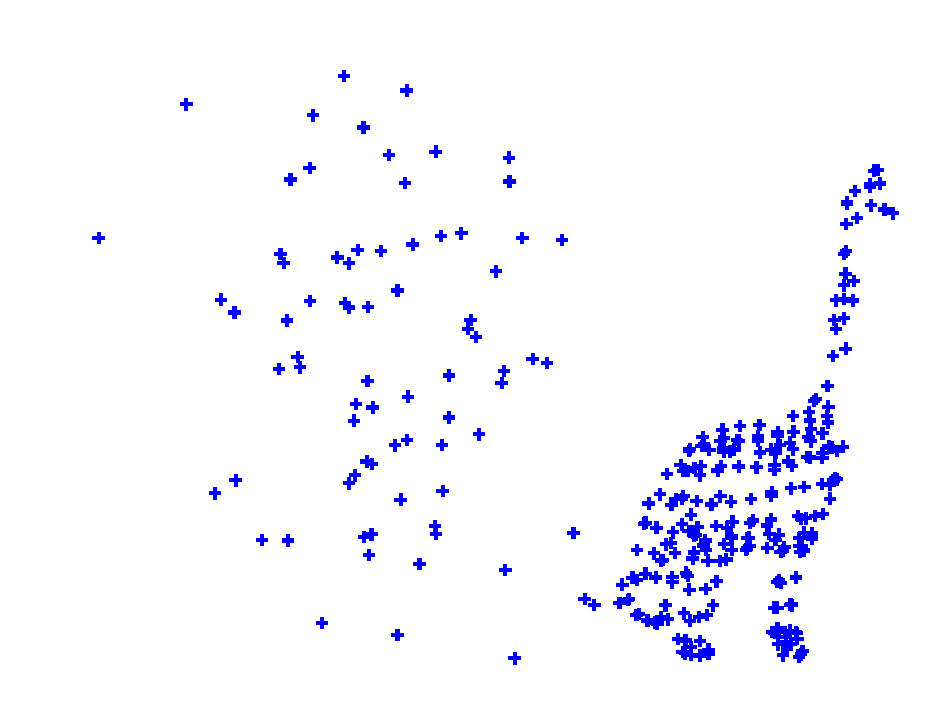}} &
		\subfigure[]{		
			\includegraphics[width=\scale\linewidth]{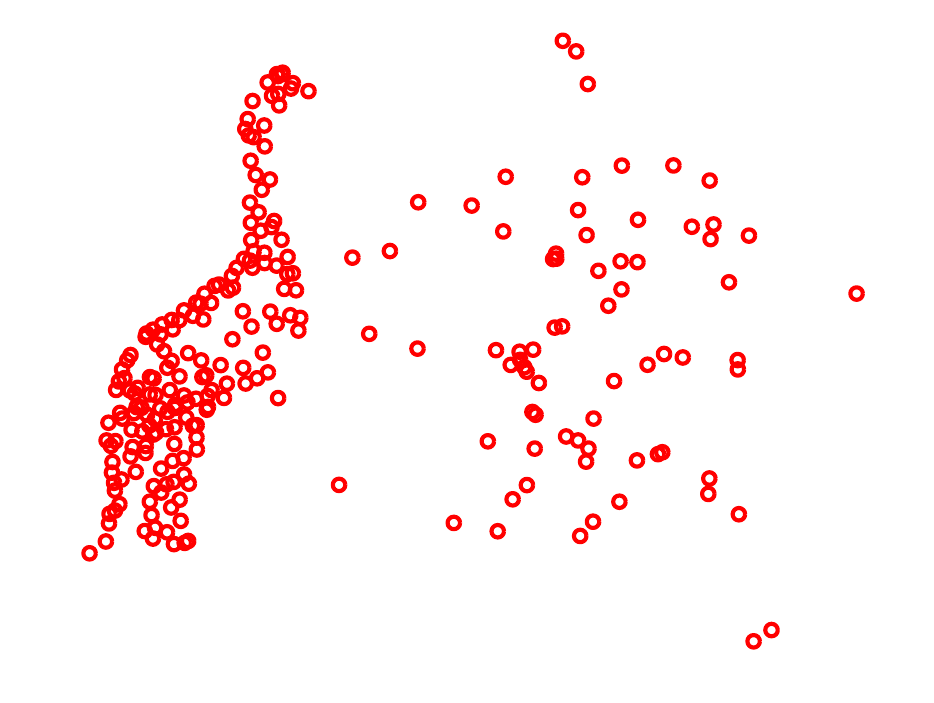}} &
		\subfigure[]{
			\includegraphics[width=\scale\linewidth]{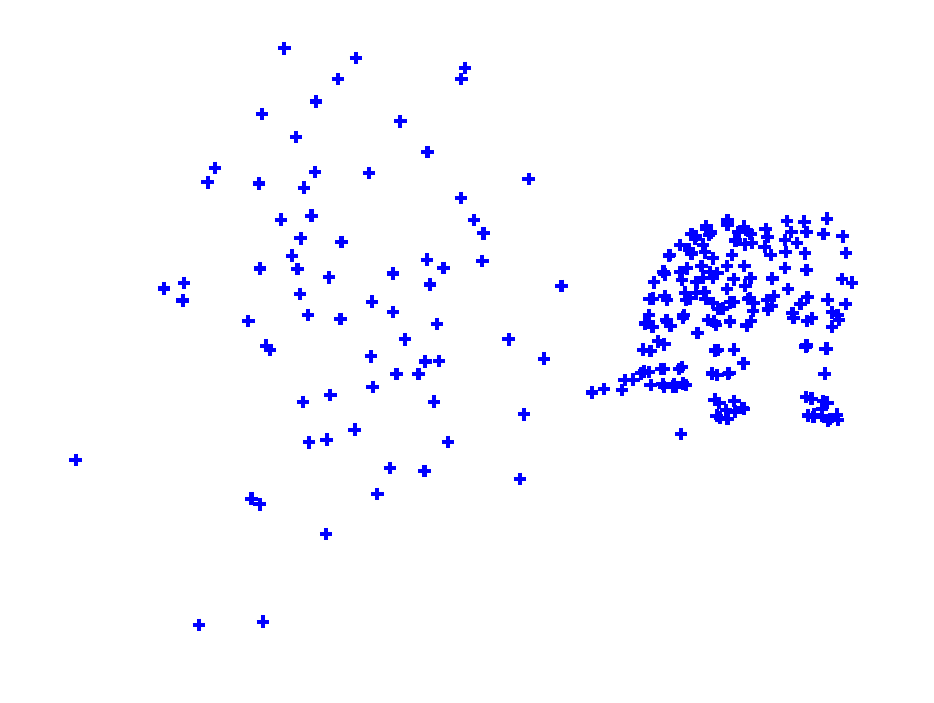}}
	\end{tabular}	
	\caption{
		(a) to (c): Model point sets and  examples of scene point sets in the deformation and noise tests, respectively.
		(d) to (i): Examples of model and scene point sets in the mixed outliers and inliers test ((d), (e)), separate outliers and inliers test ((f), (g)), and occlusion+outlier test ((h), (i)), respectively.
		In all cases, model points are indicated by red circles, while scene points are represented by blue crosses.	
		%
%		(a) to (c): model point sets 
%		and 
%		examples of scene point sets  in the deformation and  noise  tests, % (columns 2 to 3), 
%		respectively.		
%		(d) to (i):
%		examples of model and scene point sets  in the 
%		mixed outliers and inliers test ((d), (e)),
%		separate outliers and inliers test ((f), (g)) and occlusion+outlier test  ((h), (i)), respectively.
%		Here, the model points are marked by red circles and
%		the scene points  by blue crosses.
		\label{rot_3D_test_data_exa}}
%	\end{figure*} 
%
%\begin{figure*}[t]
%%%%%%%%%%%%%%%%%%%%%%%%%%%%%%%%
\centering
\newcommand\scaleGd{0.215}
\begin{tabular}{@{\hspace{-0mm}}c@{\hspace{-1.6mm}}c@{\hspace{-1.6mm}}c@{\hspace{-1.6mm}}c@{\hspace{-1.6mm}} c }
\includegraphics[width=\scaleGd\linewidth]{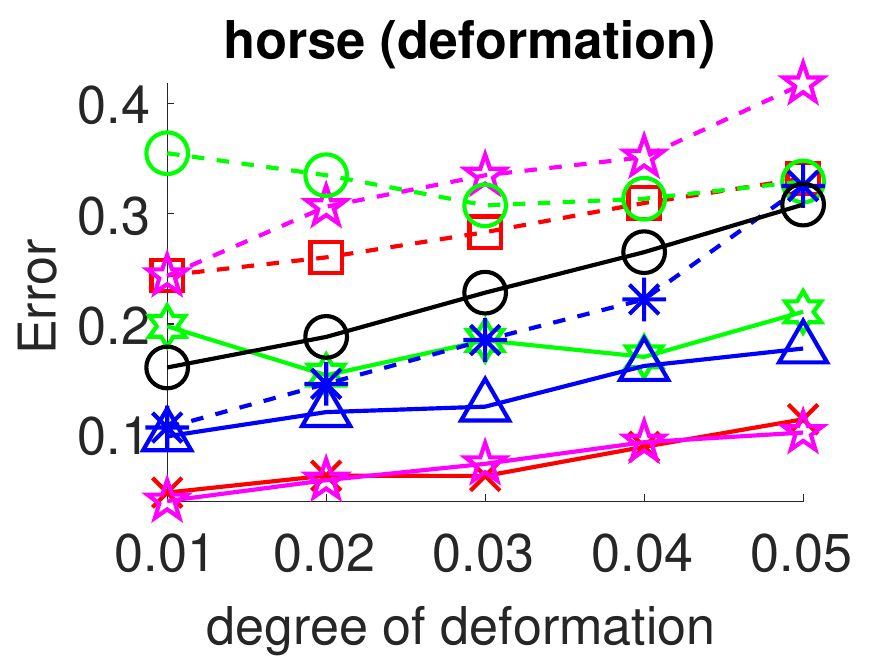}&	
	\includegraphics[width=\scaleGd\linewidth]{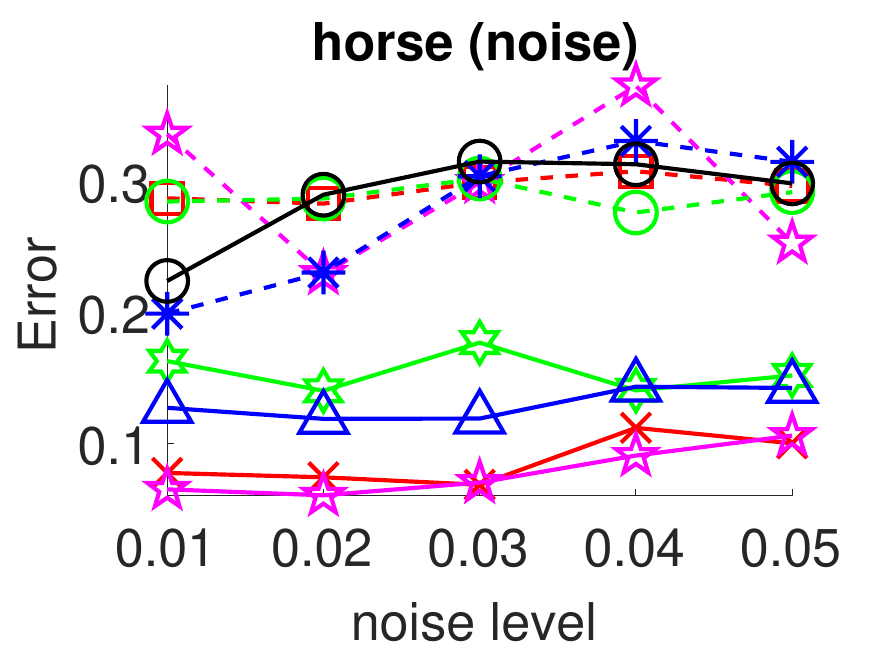}&		
	\includegraphics[width=\scaleGd\linewidth]{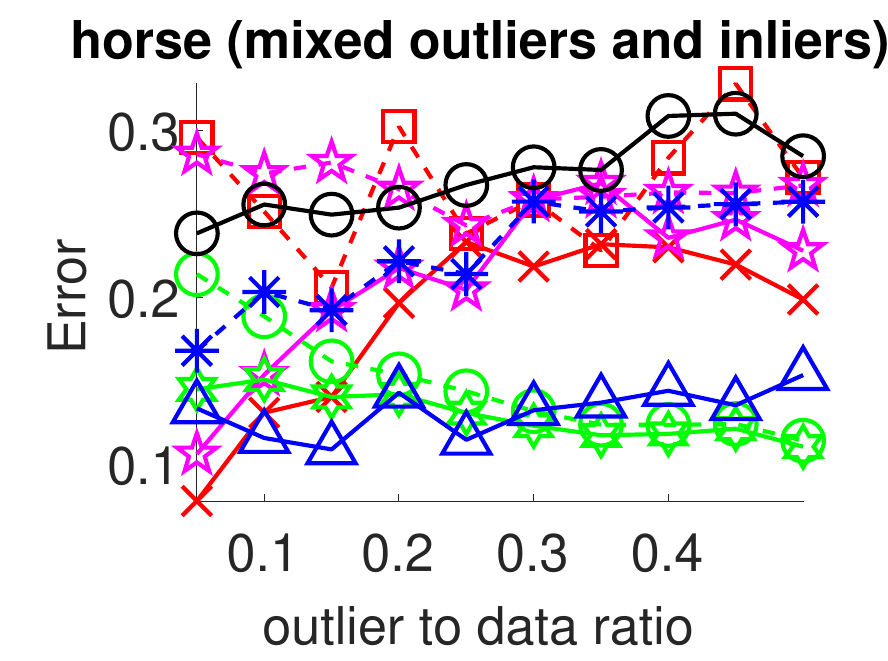}&
	
	\includegraphics[width=\scaleGd\linewidth]{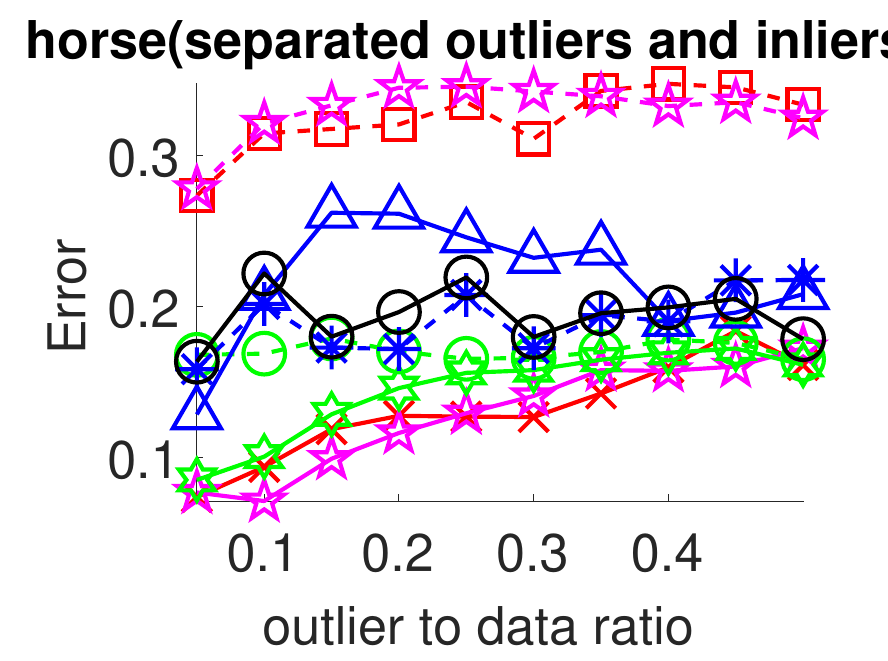}&
	\includegraphics[width=\scaleGd\linewidth]{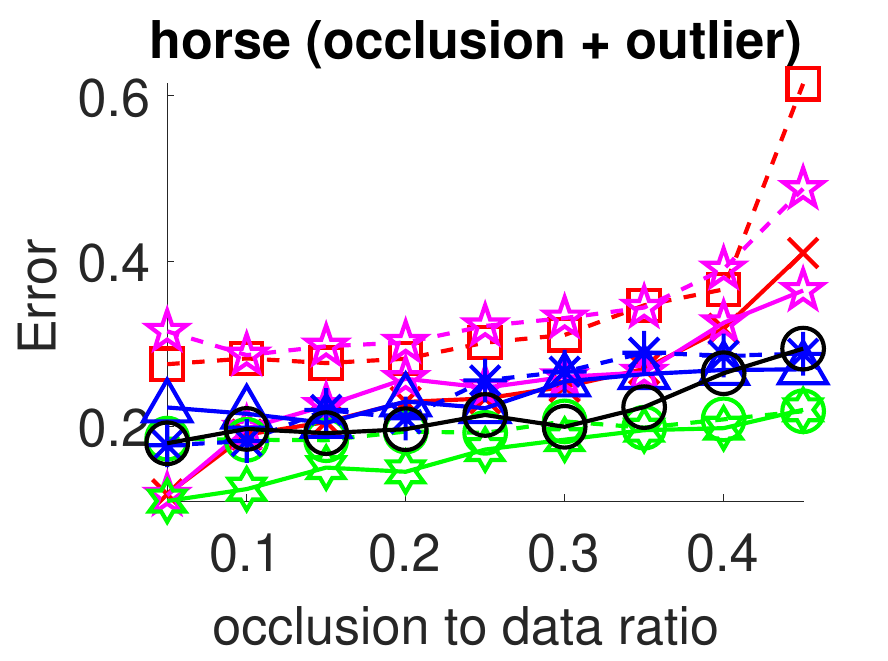}\\
	
	\includegraphics[width=\scaleGd\linewidth]{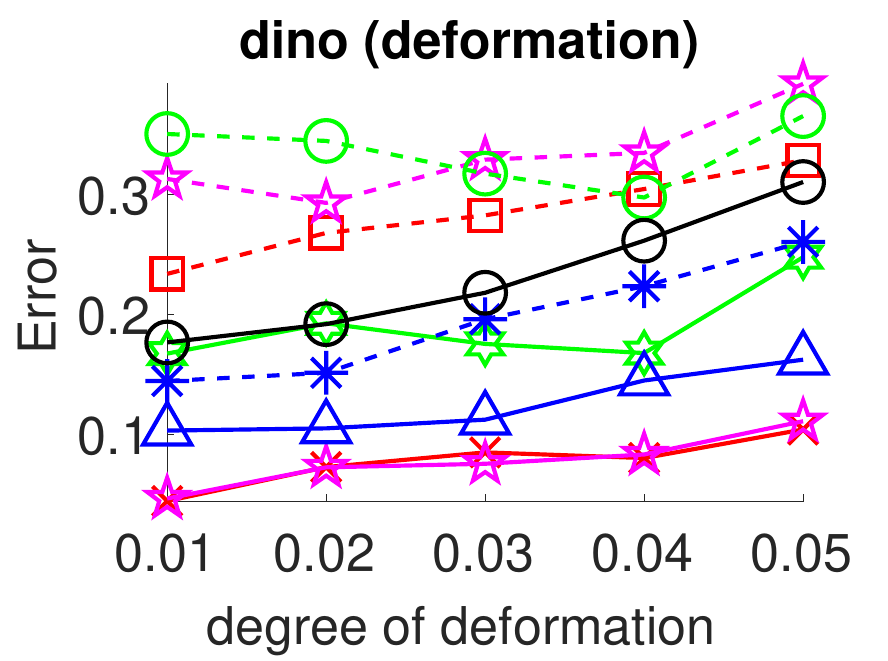}&
	\includegraphics[width=\scaleGd\linewidth]{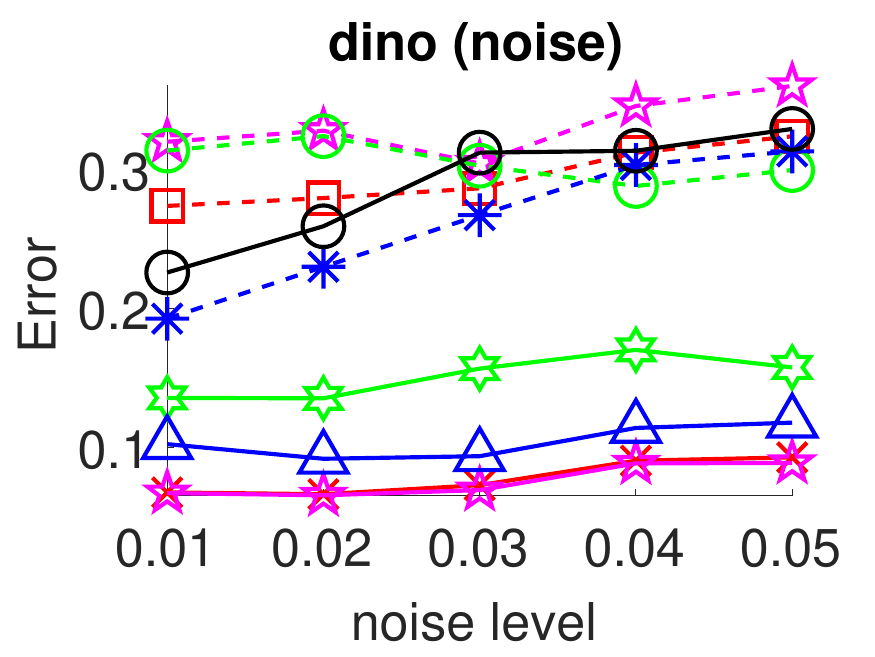}&
	\includegraphics[width=\scaleGd\linewidth]{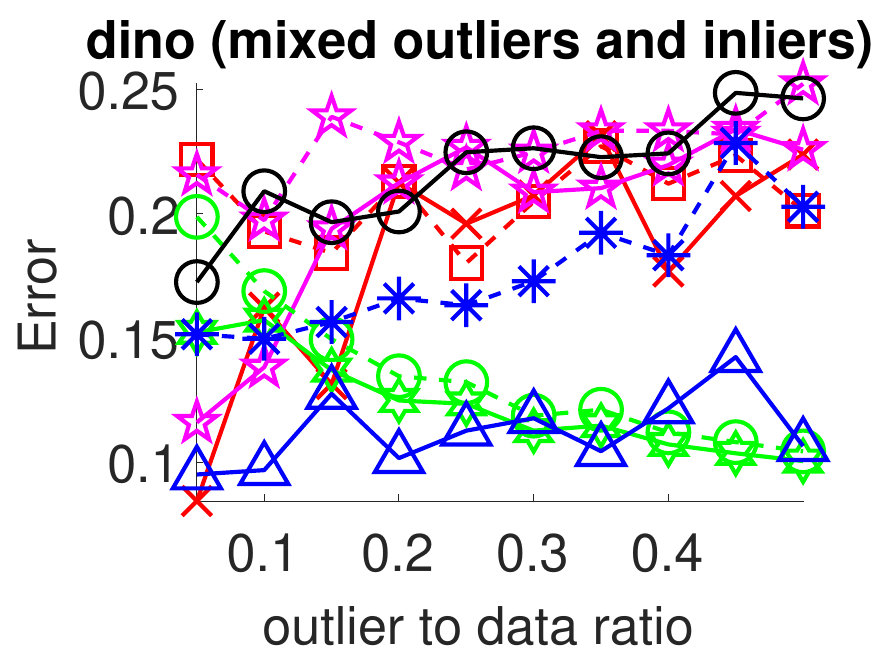}&		
	\includegraphics[width=\scaleGd\linewidth]{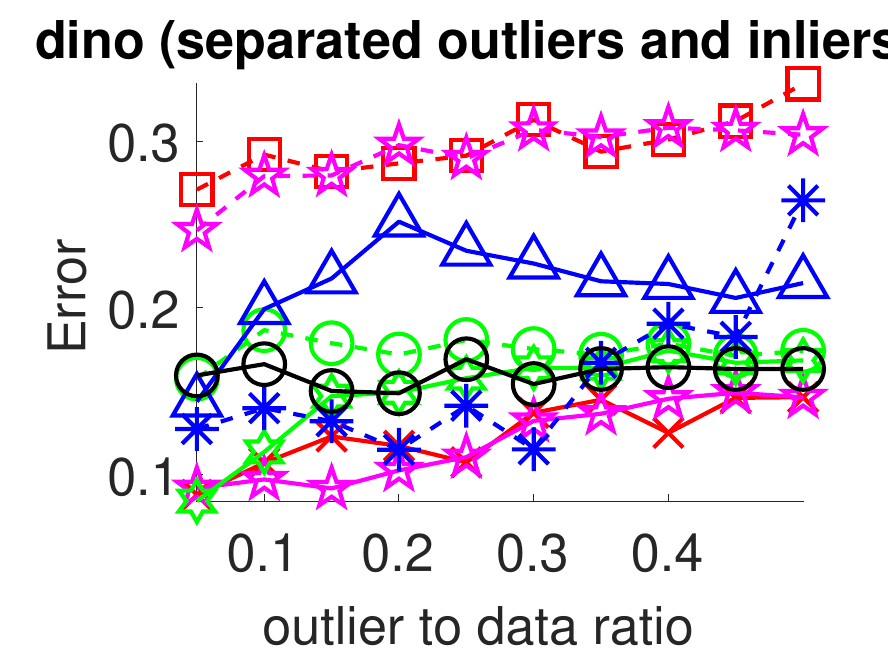}&
	\includegraphics[width=\scaleGd\linewidth]{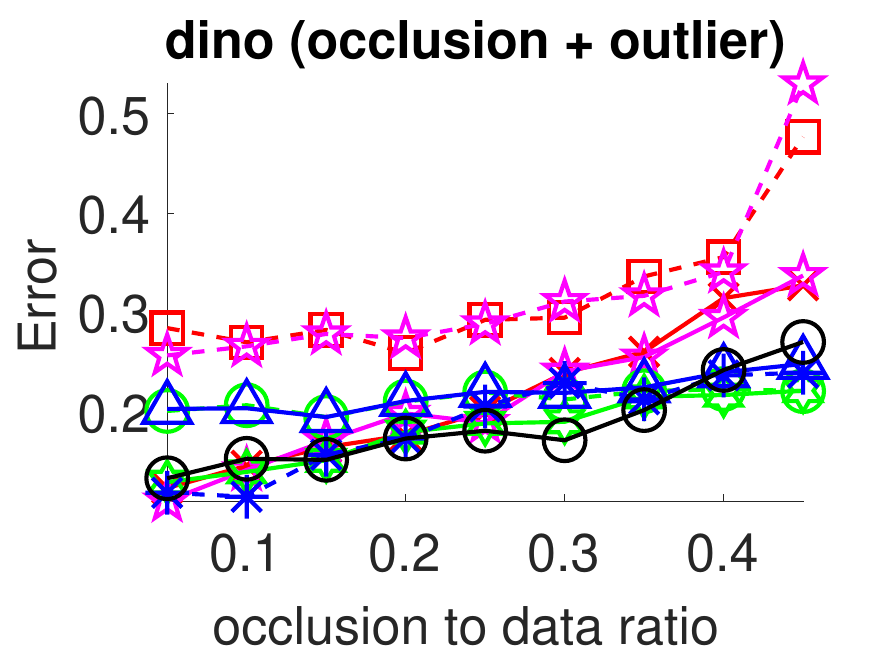} 
\end{tabular}
%	\begin{tabular}{@{\hspace{-0.1mm}}c}
\includegraphics[width=.9\linewidth]{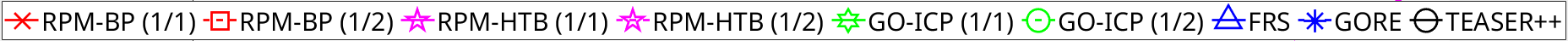}		
%\end{tabular}
\begin{tabular}{@{\hspace{-0mm}}c@{\hspace{-1.6mm}}c@{\hspace{-1.6mm}}c@{\hspace{-1.6mm}}c@{\hspace{-1.6mm}} c }
\includegraphics[width=\scaleGd\linewidth]{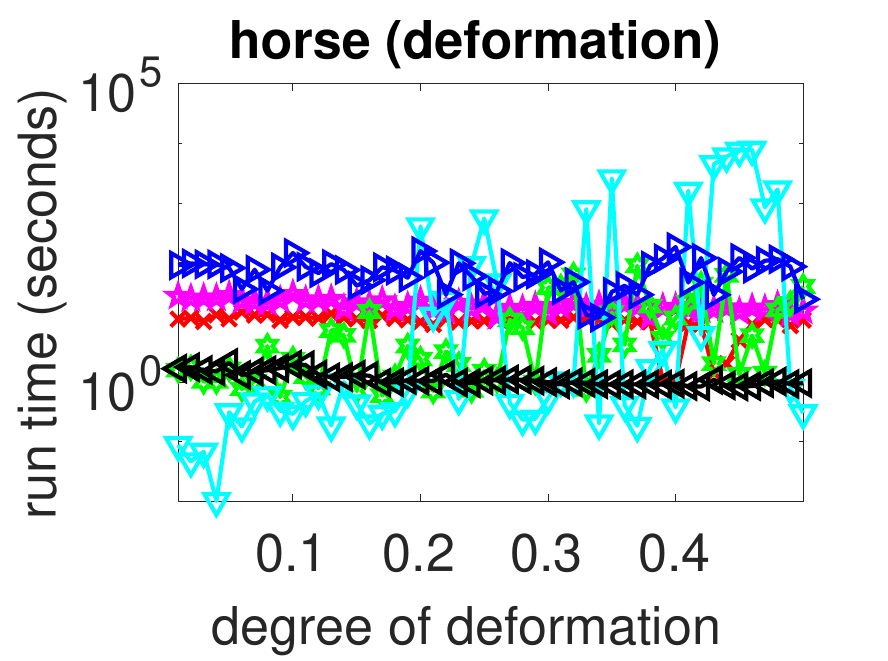}&
\includegraphics[width=\scaleGd\linewidth]{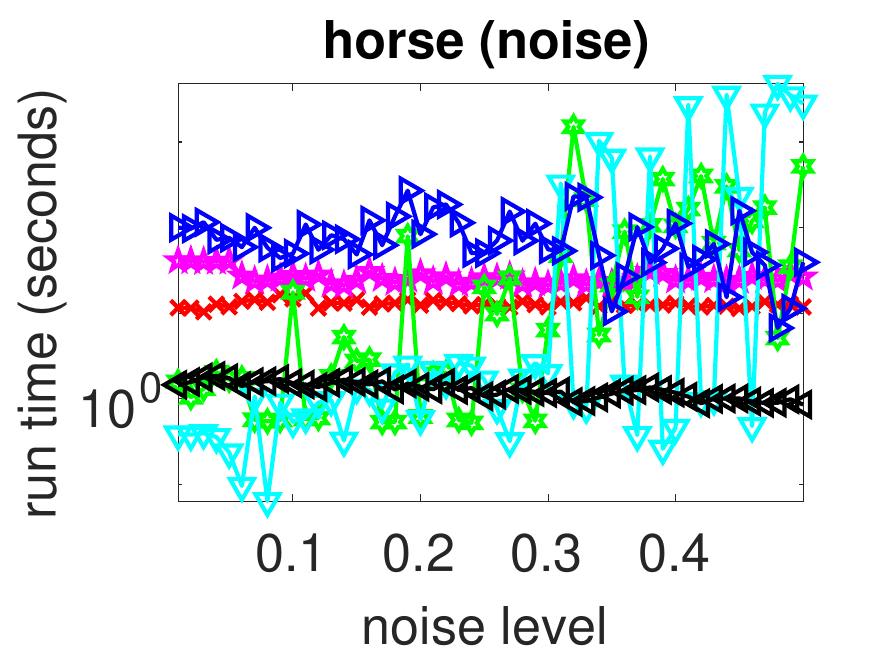}&
\includegraphics[width=\scaleGd\linewidth]{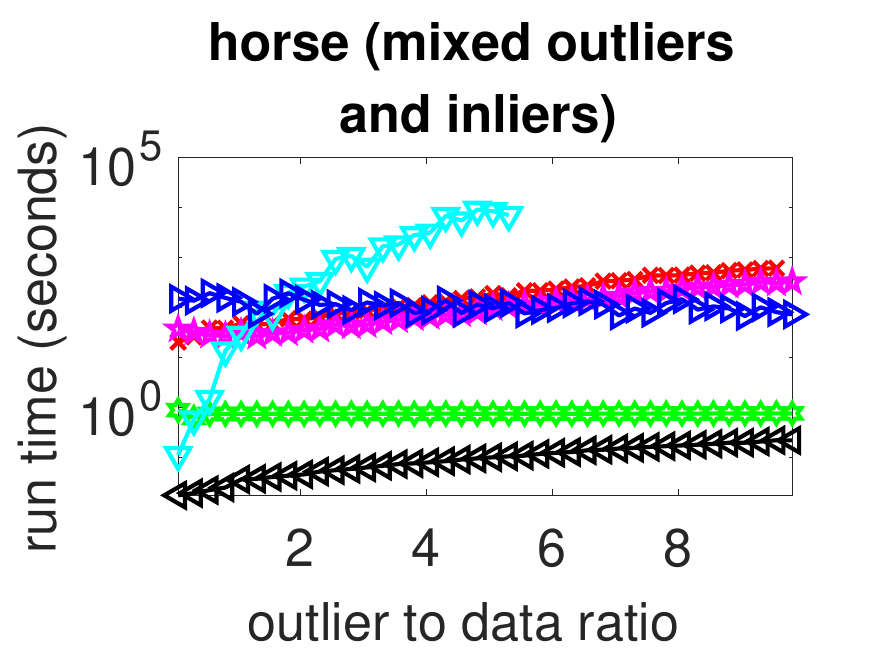}&
\includegraphics[width=\scaleGd\linewidth]{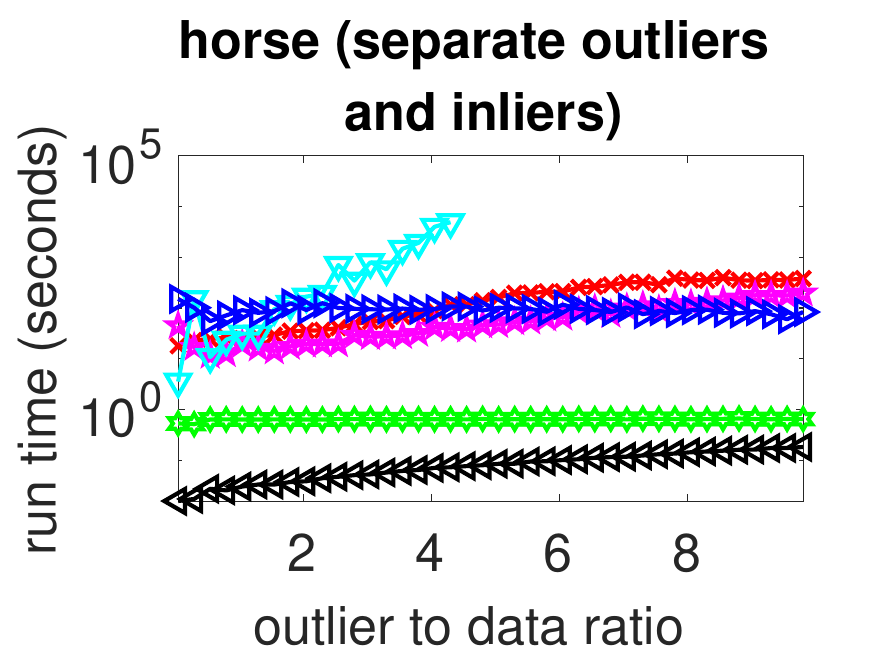}	&
\includegraphics[width=\scaleGd\linewidth]{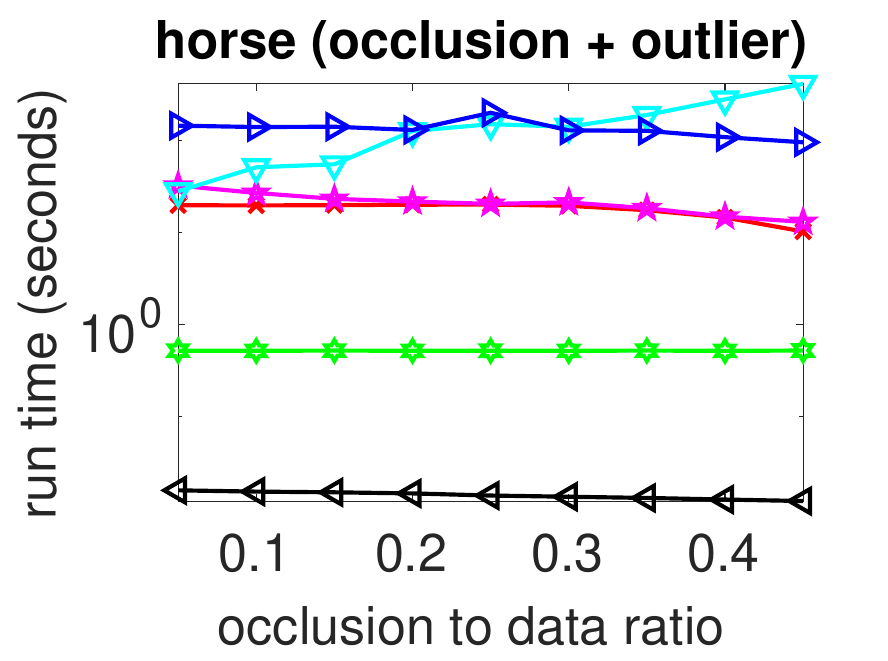}		\end{tabular}
\includegraphics[width=.65\linewidth]{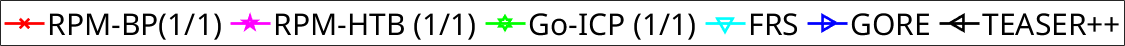}	
\caption{
Average registration errors (top 2 rows) and run times (bottom row) of
various methods
% RPM-BP, RPM-HTB, Go-ICP, FRS, GORE, and TEASER++ 
 under varying $n_p$ values ($1/2$ or $1/1$ of the ground truth value) across 100 random trials for 3D deformation, positional noise, mixed outliers and inliers, separate outliers and inliers, and occlusion+outlier tests.
%
%For the separate outliers and inliers and occlusion+outlier tests, only a portion of FRS's run time results is presented, as FRS experiences significant slowdowns with large problem sizes.				 
\label{3D_rigid_sta}}
%\end{figure*}
%\begin{figure*} [t]
	\centering
	\includegraphics[width=\linewidth]{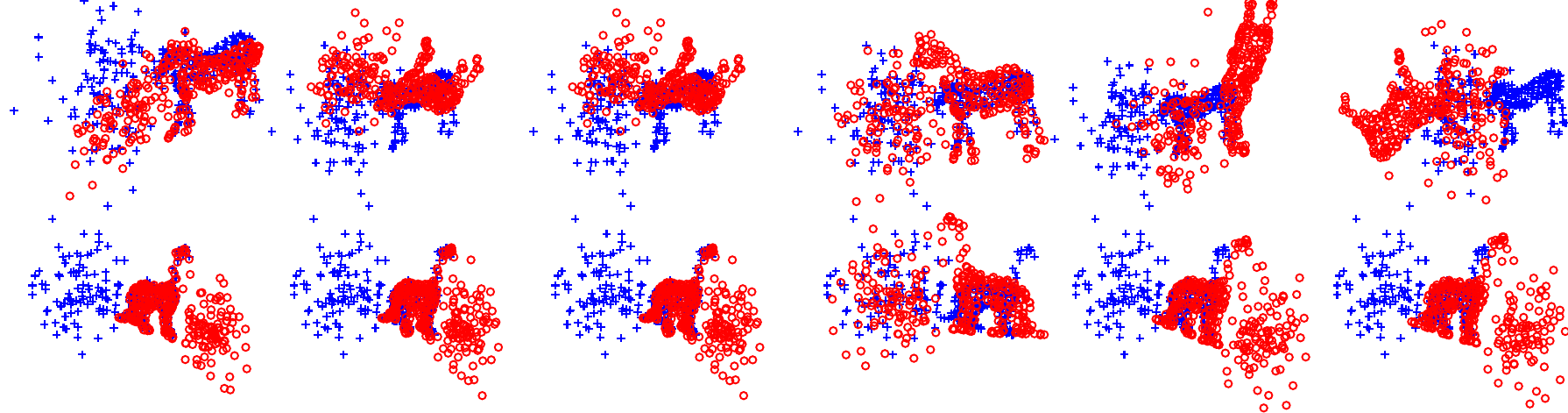}
	\newcommand\sctwo{8}	
	{\scriptsize
\begin{tabular}{@{\hspace{\sctwo mm}}c @{\hspace{\sctwo mm}} c@{\hspace{\sctwo mm}}  c@{\hspace{\sctwo mm}}  c@{\hspace{\sctwo mm}}   c@{\hspace{\sctwo mm}} c}	
	(a) RPM-BP & (b) RPM-HTB & (c) Go-ICP & (d) FRS & (e) GORE & (f) TEASER++
\end{tabular}
}
	\caption{
		Examples of registration results from different methods in the separate outliers and inliers test, where the $n_p$ values of RPM-HTB and Go-ICP are both chosen as  the ground truth.		
	\label{rot_3D_syn_match_exa}}	
\end{figure*}

\subsubsection{3DMatch Dataset}
The 3DMatch benchmark \cite{3Dmatch_dataset} provides scans from 62 scenes across five established RGB-D reconstruction datasets (sun3d, 7-scenes, rgbd-scenes-v2, bundlefusion, and analysis-by-synthesis). For convenience, we used the scan pairs from D3Feat \cite{D3feat_dataset} to evaluate the performance of our method.

Fig. \ref{3DMatch_tests} shows the matching errors for the different methods. The results indicate that RPM-BP and RPM-HTB generally outperform Go-ICP and FRS but are less accurate than GORE and TEASER++. The superior performance of GORE and TEASER++ can be attributed to their use of features, which significantly enhances their ability to match complex structures. In contrast, the other methods rely solely on point position information. Furthermore, the 3DMatch benchmark contains purely rigid transformations, which aligns perfectly with the assumptions of GORE and TEASER++. Between our method and RPM-HTB, RPM-BP demonstrates slightly better performance, especially in the sun3d and rgbd-scenes-v2 tests.

Examples of registration results are presented in Fig. \ref{3DMatch_exa}.

\begin{figure*}[t]
\begin{minipage}{1\textwidth}	
\centering
\newcommand\scale{0.22} 
\begin{tabular}{@{\hspace{-1mm}}c@{\hspace{-3mm}}c@{\hspace{-3mm}}c@{\hspace{-3mm}}c@{\hspace{-3mm}}c}	
\includegraphics[width=\scale\linewidth]{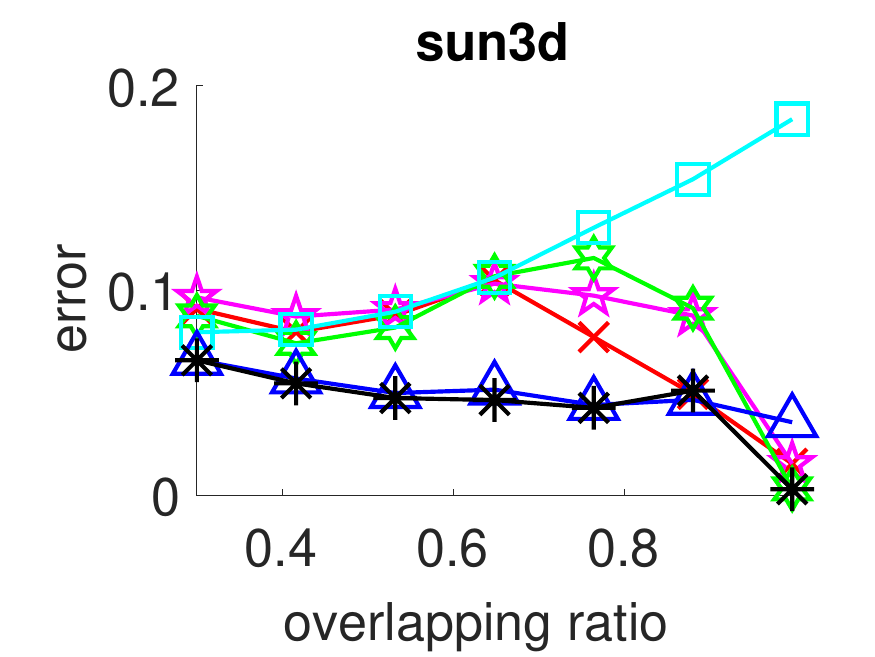}	&
\includegraphics[width=\scale\linewidth]{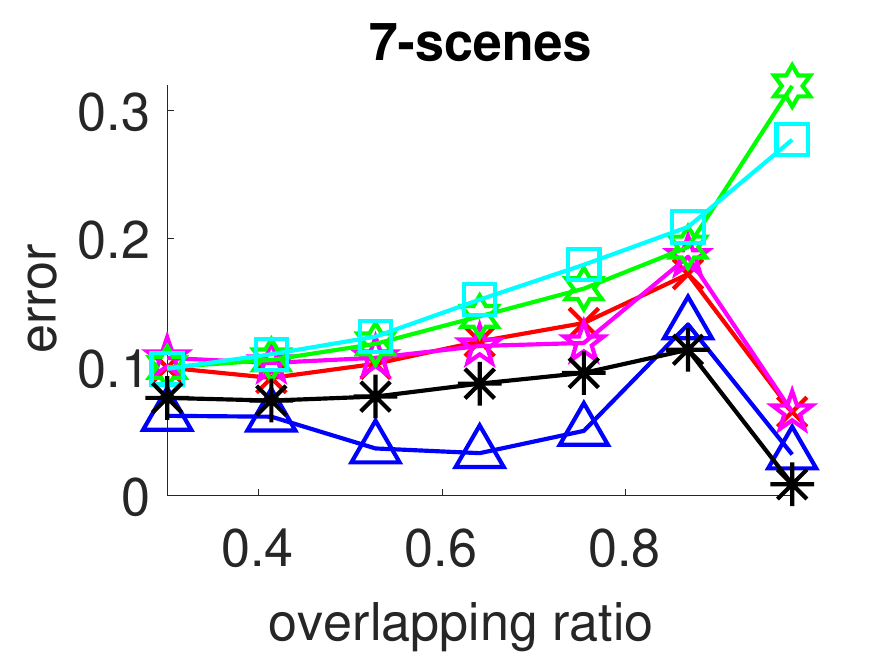} &
\includegraphics[width=\scale\linewidth]{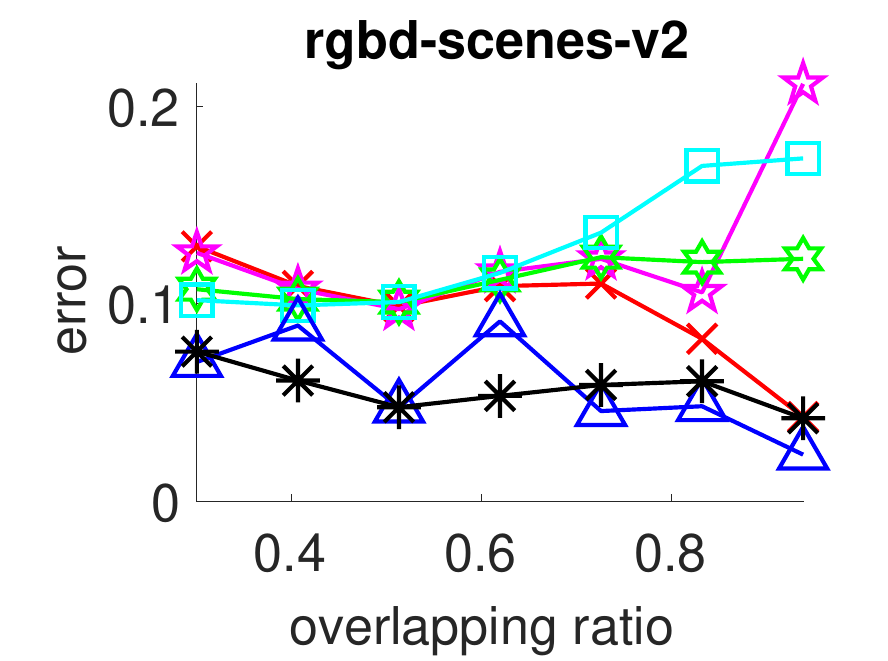} &
%\end{tabular}\\
%\begin{tabular}{c@{} c }
\includegraphics[width=\scale\linewidth]{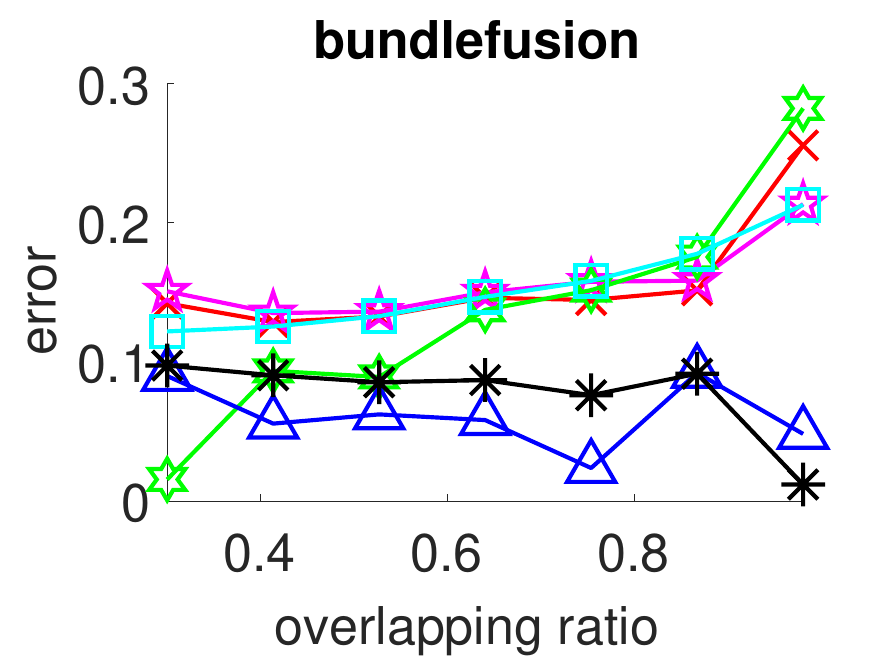}	&
\includegraphics[width=\scale\linewidth]{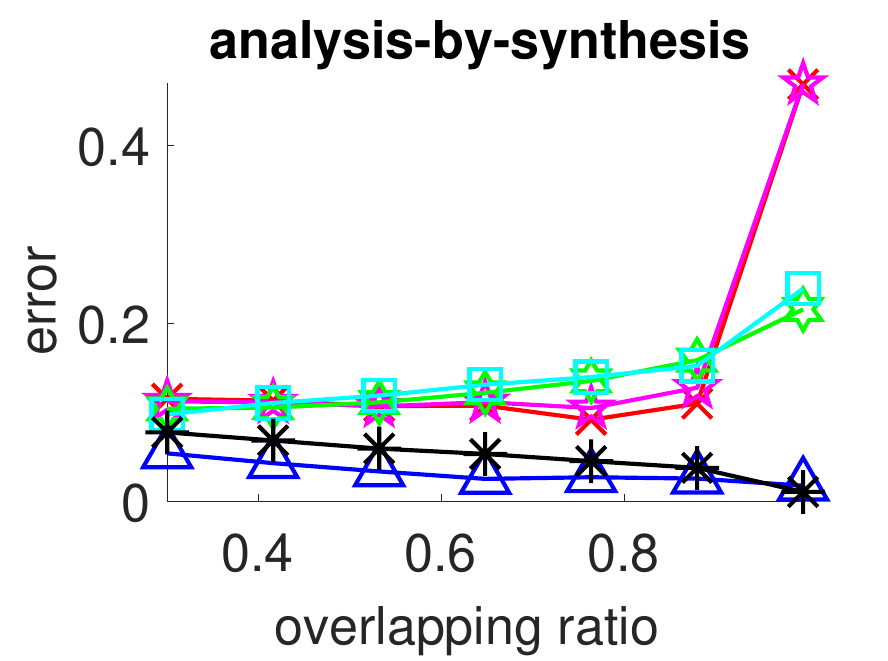}
\end{tabular}	
\includegraphics[width=0.7\linewidth]{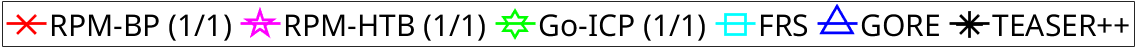}

\caption{
	Average registration errors from various methods across five RGB-D reconstruction datasets.
%Average registration errors by different methods 
%on the five RGB-D reconstruction datasets. 
\label{3DMatch_tests}	}
\end{minipage}	
%\end{figure*}
%\begin{figure*}[h]
\begin{minipage}{1\textwidth}
\centering
\newcommand\scale{1}	
\includegraphics[width=\scale\linewidth]{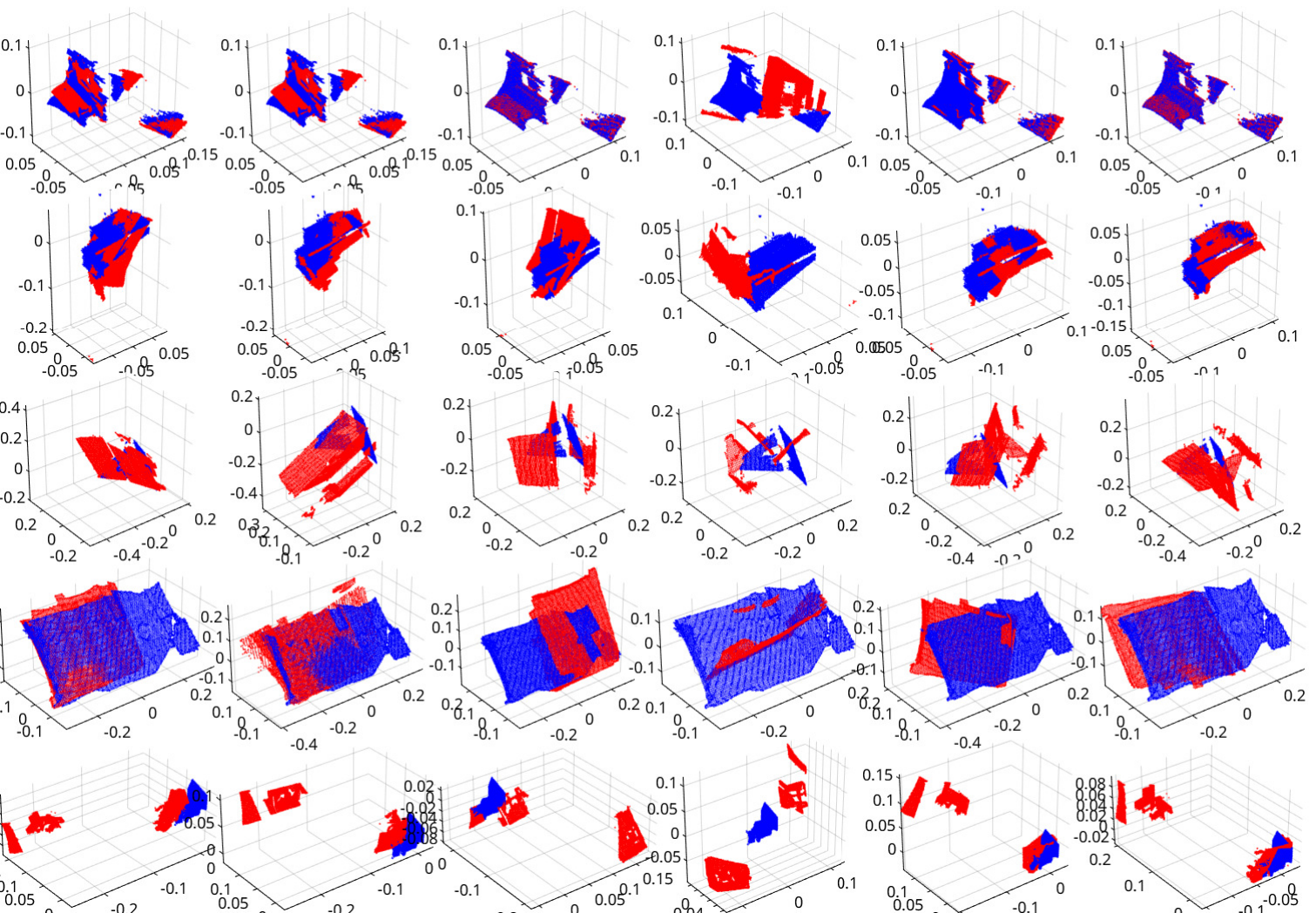}
\newcommand\sctwo{8}
{\scriptsize
\begin{tabular}{@{\hspace{-0mm}}c @{\hspace{\sctwo mm}} c@{\hspace{\sctwo mm}}  c@{\hspace{\sctwo mm}}  c@{\hspace{\sctwo mm}}   c@{\hspace{\sctwo mm}} c}	
(a) RPM-BP & (b) RPM-HTB & (c) Go-ICP & (d) FRS & (e) GORE & (f) TEASER++
\end{tabular}
}
\caption{
	Examples of registration results generated by different methods on the five RGB-D reconstruction datasets, arranged from top to bottom: sun3d, 7-scenes, rgbd-scenes-v2 (repeated), and analysis-by-synthesis.
%Examples of registration results generated by different methods
%on (from top to bottom) the five RGB-D reconstruction dataset: sun3d, 7-scenes, rgbd-scenes-v2, rgbd-scenes-v2 and analysis-by-synthesis.
\label{3DMatch_exa}	}
\end{minipage}	
\end{figure*}

\section{Conclusion \label{sec:conclude}}
In this paper, we presented RPM-BP, a novel Branch-and-Bound (BnB) algorithm for globally optimal point set registration under partial overlap. Our method is designed to be invariant to the underlying transformation, encompassing both 2D similarity/affine and 3D rigid transformations.

The efficiency of RPM-BP stems from several key features:
\begin{enumerate}%[\upshape i\upshape)]
	\item \textbf{Efficient Lower Bound Computation:} The lower bound is efficiently computed by decoupling the minimization problem into a tractable linear assignment problem (for correspondences) and a low-dimensional convex quadratic program (for transformation parameters).
	\item \textbf{Low-Dimensional Branching Space:} The dimensionality of the branching space is confined solely to the transformation parameters, which is inherently low-dimensional (e.g., 4 for 2D similarity, 6 for 3D rigid). This greatly mitigates the exponential complexity typical of BnB algorithms.
\end{enumerate}
These attributes contribute to the method's computational efficiency and its strong scalability with respect to the number of points. Our extensive experiments confirm that RPM-BP demonstrates superior robustness against non-rigid deformations, positional noise, and outliers, particularly in scenarios where outliers and inliers are spatially separated.

Despite its advantages, the proposed method has two primary limitations. First, its use is inherently restricted to transformations with a small number of parameters, as the complexity of the BnB algorithm remains exponential with respect to the branching space dimensionality. Second, the method's overall robustness is  sensitive to the parameter setting for the number of inliers, $n_p$.

For future work, we plan to focus on developing an adaptive scheme for setting the parameter $n_p$ to enhance the method's performance and practical applicability.

%\subsection*{Acknowledgments}
%This work was supported by 
%the Natural Science Foundation of Shanxi Province %, China 
%(202403021221231, 202303021222271),
%%张剑妹
%the Changzhi Key Laboratory of Intelligent Human-Machine Collaborative Operation (2024sy008),
%%马飞：
%the Fundamental Research Program of Changzhi City (JC202401),
%%崔：
%%the Fundamental Research Program of Shanxi Province (202303021222271),
%%郭子汉
%and Shanxi Provincial Education Department (2024L351, Zihan Guo).

%			National Natural Science Foundation of China under Grants  61773002 and  U19A2073, 
			%
%			the Fundamental Research Program of Shanxi Province, China under Grant 202103021223381
%			and
			%
%			Scientific and Technological Innovation Programs of Higher Education Institutions in Shanxi Province,  China under Grant 2022L517.

			%and scientific and technologial innovation programs of higher education instituions in Shanxi (grant number 2019L0911).

			%{\small
				%\bibliographystyle{ieee_fullname}
				%\bibliography{egbib}
				
				%	\bibliographystyle{ieee}
%				\bibliographystyle{ieee_fullname}
				%\bibliography{egbib}
				
				%	\bibliographystyle{ieee}
%				\bibliography{DP_SC_CVPR}
				%}
			
			{\small
				\bibliographystyle{elsarticle-num}%{splncs}%
				\bibliography{DP_SC_CVPR}
			}
			
			%\bibliographystyle{elsarticle-num}
			%\bibliography{../DP_SC_rotate/DP_SC_CVPR}

%			%\begin{appendix}
%			%	content...
%			%	\section{ee}
%			%	\section{ff}
%			%\end{appendix}
			\appendix

\section{Heuristic Justification for the Positive Semidefiniteness of $\mathbf H^0+\mathbf C$ \label{sec:appendix}}

The convergence of our Branch-and-Bound algorithm relies on the lower bound subproblem being convex, which requires the matrix $\mathbf H^0 + \mathbf C$ to be \textbf{Positive Semidefinite (PSD)}. While a formal proof that $\mathbf H^0 + \mathbf C \succeq 0$ is elusive, we offer a heuristic justification, consistent with our empirical findings in Section \ref{subsec:PSD}.

\begin{proposition}
	The matrix $\mathbf H^0+\mathbf C$ is an average of matrices that are heuristically expected to be close to Positive Semidefinite.
\end{proposition}	

\begin{proof}
	We first recall the definition of the Hessian matrix $\mathbf G(\mathbf p)$ with respect to $\boldsymbol\theta$ for a fixed correspondence vector $\mathbf p$:
	$$
	\mathbf G(\mathbf p) = \text{mat} ( \mathbf K \mathbf B_2 \mathbf p ) + \mathbf C.
	$$
	
	\noindent Since the point correspondence vector $\mathbf p$ satisfies $\mathbf p \ge 0$ (i.e., $\mathbf p \in \Omega$), and the original objective function $E(\mathbf p, \boldsymbol\theta)$ represents a weighted sum of squared errors, it is fundamentally a convex quadratic function of $\boldsymbol\theta$ for any feasible $\mathbf p$. This mandates that the Hessian must be Positive Semidefinite for all valid $\mathbf p$:
	$$
	\mathbf G(\mathbf p) \succeq 0 \quad \text{for all } \mathbf p \in \Omega.
	$$
	
	The matrix $\mathbf H^0+\mathbf C$ is defined as the average of two bounding matrices:
	$$
	\mathbf H^0+\mathbf C = \frac{1}{2} \left[ (\underline{\boldsymbol\Gamma} + \mathbf C) + (\overline{\boldsymbol\Gamma} + \mathbf C) \right].
	$$
	where $\underline{\boldsymbol\Gamma} + \mathbf C$ (and similarly $\overline{\boldsymbol\Gamma} + \mathbf C$) is an element-wise lower bound on the Hessian $\mathbf G(\mathbf p)$:
	$$
	(\underline{\boldsymbol\Gamma} + \mathbf C)_{ij}= \left( \min_{\mathbf p\in \Omega} [\text{mat} ( \mathbf K\mathbf B_2\mathbf p )]_{ij} \right) + \mathbf C_{ij}.
	$$
	
	The element-wise minimum $\underline{\boldsymbol\Gamma}_{ij}$ (and maximum $\overline{\boldsymbol\Gamma}_{ij}$) is attained by some $\mathbf p \in \Omega$. Since all $\mathbf G(\mathbf p)$ are PSD matrices, the elements used to construct $\underline{\boldsymbol\Gamma} + \mathbf C$ and $\overline{\boldsymbol\Gamma} + \mathbf C$ are drawn from a collection of PSD matrices. While $\underline{\boldsymbol\Gamma} + \mathbf C$ and $\overline{\boldsymbol\Gamma} + \mathbf C$ may not be strictly PSD themselves (as the minimum/maximum of each entry may be attained by a different $\mathbf p$), they are structurally derived from PSD matrices.
	
	It is therefore reasonable to conclude that both $(\underline{\boldsymbol\Gamma} + \mathbf C)$ and $(\overline{\boldsymbol\Gamma} + \mathbf C)$ are \textbf{close to being positive semidefinite}. Because the sum of two PSD matrices is PSD, their average, $\mathbf H^0+\mathbf C$, is consequently expected to be  \textbf{close to being positive semidefinite}, ensuring the practical convexity of the subproblem.
\end{proof}

		\end{document}